\documentclass[a4paper,reqno]{amsart}
\usepackage[british]{babel}
\usepackage[T1]{fontenc}
\usepackage{amssymb,amsthm,bm}
\usepackage{mathtools, mathrsfs,thmtools}
\usepackage{nicefrac}
\mathtoolsset{mathic}
\allowdisplaybreaks
\usepackage{courier,mathptmx,stmaryrd}
\usepackage{fontawesome5}
\usepackage[dvipsnames]{xcolor}
\usepackage[outline]{contour}
\usepackage{manfnt}
\usepackage{longtable}
\usepackage{csquotes}
\usepackage[backend=biber,style=numeric-comp,sorting=nyt,isbn=false,url=false]{biblatex}
\renewbibmacro*{name:family}[4]{%
  \ifuseprefix
    {\usebibmacro{name:delim}{#3#1}%
     \usebibmacro{name:hook}{#3#1}%
     \mkbibcompletenamefamily{%
       \ifdefvoid{#3}
         {}
         {\mkbibnameprefix{\MakeCapital{#3}}\isdot%
          \ifprefchar{}{\bibnamedelimc}}%
       \mkbibnamefamily{#1}\isdot}}
    {\usebibmacro{name:delim}{#1}%
     \usebibmacro{name:hook}{#1}%
     \mkbibcompletenamefamily{%
       \mkbibnamefamily{#1}\isdot}}}%
\renewbibmacro*{name:family-given}[4]{%
  \ifuseprefix
    {\usebibmacro{name:delim}{#3#1}%
     \usebibmacro{name:hook}{#3#1}%
     \mkbibcompletenamefamilygiven{%
       \ifdefvoid{#3}
         {}
         {\mkbibnameprefix{\MakeCapital{#3}}\isdot
          \ifprefchar{}{\bibnamedelimc}}%
       \mkbibnamefamily{#1}\isdot
       \ifdefvoid{#4}
         {}
         {\bibnamedelimd\mkbibnamesuffix{#4}\isdot}%
       \ifdefvoid{#2}
         {}
         {\revsdnamepunct\bibnamedelimd\mkbibnamegiven{#2}\isdot}}}
    {\usebibmacro{name:delim}{#1}%
     \usebibmacro{name:hook}{#1}%
     \mkbibcompletenamefamilygiven{%
       \mkbibnamefamily{#1}\isdot
       \ifdefvoid{#4}
         {}
         {\bibnamedelimd\mkbibnamesuffix{#4}\isdot}%
       \ifboolexpe{%
         test {\ifdefvoid{#2}}
         and
         test {\ifdefvoid{#3}}}
         {}
         {\revsdnamepunct}%
       \ifdefvoid{#2}
         {}
         {\bibnamedelimd\mkbibnamegiven{#2}\isdot}%
       \ifdefvoid{#3}
         {}
         {\bibnamedelimd\mkbibnameprefix{#3}\isdot}}}}
\usepackage{hyperref,url}
  \hypersetup{  
    bookmarksnumbered
  }
  \hypersetup{  
    breaklinks=false,
    pdfborderstyle={/S/U/W 0},
    citebordercolor=.235 .702 .443,
    urlbordercolor=.255 .412 .882,
    linkbordercolor=.804 .149 .19,
  }
  \hypersetup{ 
    pdfauthor={Kathelijne Coussement, Gert de Cooman, Keano De Vos},
    pdftitle={AGM-like belief change for Accept-Desirability models},
    pdfkeywords={}
  }

\usepackage[capitalise,noabbrev]{cleveref}
\crefrangeformat{enumi}{#3#1#4--#5#2#6}
\crefformat{enumi}{#2#1#3}
\crefformat{section}{Section~#2#1#3}
\crefformat{proposition}{Proposition~#2#1#3}
\usepackage{hypcap}
\usepackage{tikz}
\usepackage{enumitem}
\newcommand*{\gamblechili}{\raisebox{1pt}{\textcolor{ForestGreen}{\scalebox{0.7}{\faPepperHot}}}}
\newcommand*{\quantumchili}{\raisebox{1pt}{\textcolor{red}{\scalebox{0.7}{\faPepperHot}}}}
\newcommand*{\blackchili}{\raisebox{1pt}{\textcolor{black}{\scalebox{0.7}{\faPepperHot}}}}
\renewcommand{\thmcontinues}[1]{continued}
\newtheorem{theorem}{Theorem}
\newtheorem{proposition}[theorem]{Proposition}
\newtheorem{definition}[theorem]{Definition}
\newtheorem{corollary}[theorem]{Corollary}
\newtheorem{lemma}[theorem]{Lemma}
\newenvironment{counterexample}[1]{\par\medskip\noindent\textit{Counterexample for #1}.\enskip}%
{\hfill\blackchili\par\medskip}
\newenvironment{gamblesexample}[1][]{\par\medskip\noindent\textit{Classical probability#1}.\enskip}%
{\hfill\gamblechili\par\medskip}
\newenvironment{quantumexample}[1][]{\par\medskip\noindent\textit{Quantum probability#1}.\enskip}%
{\hfill\quantumchili\par\medskip}
\newenvironment{gamblesexamplecont}{\par\medskip\noindent\textit{Classical probability} (continued).\enskip}%
{\hfill\gamblechili\par\medskip}
\newenvironment{quantumexamplecont}{\par\medskip\noindent\textit{Quantum probability} (continued).\enskip}%
{\hfill\quantumchili\par\medskip}

\definecolor{definecolour}{RGB}{45,125,162} 

\newcommand*{\DummyArgumentSymbol}{\hbox{\scalebox{.9}{\(\bullet\)}}}
\newcommand*{\SmallDummyArgumentSymbol}{\hbox{\scalebox{.7}{\(\bullet\)}}}
\newcommand*{\bolleke}{\DummyArgumentSymbol}
\newcommand*{\smallbolleke}{\SmallDummyArgumentSymbol}
\providecommand\given{}
\newcommand\CondSymbol[1][]{\nonscript\,#1\vert\allowbreak\nonscript\,\mathopen{}}
\newcommand\AltCondSymbol[1][]{\nonscript\,#1\Vert\allowbreak\nonscript\,\mathopen{}}

\newcommand*{\IndifSymbol}{I}
\newcommand*{\AssessmentSymbol}{A}
\newcommand*{\ModelSymbol}{M}
\newcommand*{\BackgroundModelSymbol}{V}
\newcommand*{\PowerSetSymbol}{\mathscr{P}}
\newcommand*{\PrevisionSymbol}{P}

\newcommand*{\OptionsSetSymbol}{U}
\newcommand*{\EventsSetSymbol}{E}
\newcommand*{\GamblesSetSymbol}{\mathscr{G}}
\newcommand*{\ExpansionSymbol}{E}
\newcommand*{\RevisionSymbol}{R}
\newcommand*{\ContractionSymbol}{C}

\DeclareMathOperator{\posi}{posi}
\DeclareMathOperator{\shull}{sh}
\DeclareMathOperator{\linspan}{span}
\DeclareMathOperator{\clsr}{cl}
\DeclareMathOperator{\comp}{co}
\DeclareMathOperator{\range}{rng}

\DeclarePairedDelimiterX\group[1]{(}{)}{\ifblank{#1}{\,\DummyArgumentSymbol\,}{#1}}
\DeclarePairedDelimiterX{\sqgroup}[1]{[}{]}{\ifblank{#1}{\,\DummyArgumentSymbol\,}{#1}}
\DeclarePairedDelimiterX{\set}[1]{\{}{\}}{\renewcommand\given{\SetSymbol}#1}
\DeclarePairedDelimiter{\structure}{\langle}{\rangle}
\DeclarePairedDelimiterXPP{\abs}[1]{\,}{\vert}{\vert}{\,}{\ifblank{#1}{\,\DummyArgumentSymbol\,}{#1}}
\DeclarePairedDelimiterX{\Adelim}[2]{\langle}{\rangle}{\,#1\,;\,#2\,}

\DeclarePairedDelimiterXPP{\optnormof}[1]{\,}{\lVert}{\rVert}{_\opts}{\ifblank{#1}{\,\DummyArgumentSymbol\,}{#1}}

\DeclarePairedDelimiterXPP{\compof}[1]{\comp}{(}{)}{}{\ifblank{#1}{\,\DummyArgumentSymbol\,}{#1}}
\DeclarePairedDelimiterXPP{\shullof}[1]{\shull}{(}{)}{}{\ifblank{#1}{\,\DummyArgumentSymbol\,}{#1}}
\DeclarePairedDelimiterXPP{\posiof}[1]{\posi}{(}{)}{}{\ifblank{#1}{\,\DummyArgumentSymbol\,}{#1}}
\DeclarePairedDelimiterXPP{\linspanof}[1]{\linspan}{(}{)}{}{\ifblank{#1}{\,\DummyArgumentSymbol\,}{#1}}
\DeclarePairedDelimiterXPP{\closureof}[2]{\clsr_{#1}}{(}{)}{}{\ifblank{#2}{\,\DummyArgumentSymbol\,}{#2}}
\DeclarePairedDelimiterXPP{\cohmodclosureof}[1]{\cohmodclosure}{(}{)}{}{\ifblank{#1}{\,\DummyArgumentSymbol\,}{#1}}

\DeclarePairedDelimiterXPP{\rangeof}[1]{\range}{(}{)}{}{\ifblank{#1}{\,\DummyArgumentSymbol\,}{#1}}
\DeclarePairedDelimiterXPP{\powersetof}[1]{\PowerSetSymbol}{(}{)}{}{\ifblank{#1}{\,\DummyArgumentSymbol\,}{#1}}
\DeclarePairedDelimiterXPP{\expandof}[1]{\expand}{(}{)}{}{\renewcommand\given{\CondSymbol}#1}
\DeclarePairedDelimiterXPP{\reviseof}[1]{\revise}{(}{)}{}{\renewcommand\given{\CondSymbol}#1}
\DeclarePairedDelimiterXPP{\contractof}[1]{\contract}{(}{)}{}{\renewcommand\given{\CondSymbol}#1}

\DeclarePairedDelimiterXPP{\Msabove}[1]{\Ms}{(}{)}{}{\ifblank{#1}{\,\DummyArgumentSymbol\,}{#1}}
\DeclarePairedDelimiterXPP{\closedMsabove}[1]{\closedMs}{(}{)}{}{\ifblank{#1}{\,\DummyArgumentSymbol\,}{#1}}
\DeclarePairedDelimiterXPP{\closeddifMsabove}[1]{\closeddifMs}{(}{)}{}{\ifblank{#1}{\,\DummyArgumentSymbol\,}{#1}}
\DeclarePairedDelimiterXPP{\adfMsabove}[1]{\adfMs}{(}{)}{}{\ifblank{#1}{\,\DummyArgumentSymbol\,}{#1}}
\DeclarePairedDelimiterXPP{\difMsabove}[1]{\difMs}{(}{)}{}{\ifblank{#1}{\,\DummyArgumentSymbol\,}{#1}}
\DeclarePairedDelimiterXPP{\maxMsabove}[1]{\maxMs}{(}{)}{}{\ifblank{#1}{\,\DummyArgumentSymbol\,}{#1}}
\DeclarePairedDelimiterXPP{\maxadfMsabove}[1]{\maxadfMs}{(}{)}{}{\ifblank{#1}{\,\DummyArgumentSymbol\,}{#1}}

\newcommand*{\cls}[1]{\clsr_{#1}}
\newcommand*{\union}{\cup}
\newcommand*{\intersection}{\cap}
\newcommand*{\Intersection}{\bigcap}
\newcommand*{\Mcls}{\cls{\Ms}}

\newcommand*{\Mclsof}[2][]{\closureof[#1]{\Ms}{#2}}
\newcommand*{\Mclsbgof}[3][]{\closureof[#1]{\Msabove{#2}}{#3}}

\DeclarePairedDelimiterXPP{\prevof}[1]{\prev}{(}{)}{}{\renewcommand\given{\CondSymbol}\ifblank{#1}{\,\DummyArgumentSymbol\,}{#1}}
\newcommand*{\prev}{\PrevisionSymbol}

\newcommand*{\lprev}{\underline{\prev}}

\newcommand*{\uprev}{\overline{\prev}}

\newcommand*{\reals}{\mathbb{R}}

\newcommand*{\then}{\Rightarrow}
\newcommand*{\ifandonlyif}{\Leftrightarrow}

\newcommand*{\opt}[1][]{u_{#1}}
\newcommand*{\altopt}[1][]{v_{#1}}
\newcommand*{\altopttoo}[1][]{w_{#1}}
\newcommand*{\eventopt}[1][]{e_{#1}}
\newcommand*{\calledoff}[2][]{\eventopt[#1]\ast#2}

\newcommand*{\unitevent}{1_{\eventopts}}
\newcommand*{\nullevent}{0_{\eventopts}}

\newcommand*{\opts}{\mathscr{\OptionsSetSymbol}}
\newcommand*{\eventopts}{\mathscr{\EventsSetSymbol}}

\newcommand*{\regulareventopts}[1][\bgM]{\eventopts^{{#1}}_\mathrm{reg}}

\newcommand*{\someopts}{\OptionsSetSymbol}
\newcommand*{\optunit}{\unit[\opts]}
\newcommand*{\gbl}{f}
\newcommand*{\altgbl}{g}
\newcommand*{\altgbltoo}{h}

\newcommand*{\gambles}{\mathscr{\GamblesSetSymbol}}

\newcommand{\varstate}[1][]{{\mathit{X}_{#1}}}
\newcommand*{\states}[1][]{\mathscr{X}_{#1}}

\newcommand*{\state}[1][]{x_{#1}}
\newcommand*{\altstate}[1][]{y_{#1}}
\newcommand*{\event}[1][]{E_{#1}}
\newcommand*{\indof}[1]{\mathbb{I}_{#1}}
\newcommand*{\indevent}[1][]{\indof{\event[#1]}}

\newcommand*{\filter}[1][]{\mathscr{F}_{#1}}
\newcommand*{\props}{\mathscr{A}}
\newcommand*{\filterbase}[1][]{\mathscr{C}_{#1}}

\newcommand*{\unit}[1][]{\mathbf{1}_{#1}}

\newcommand*{\belmod}[1][]{b_{#1}}
\newcommand*{\belmods}{\mathbf{B}}
\newcommand*{\cohmod}[1][]{c_{#1}}
\newcommand*{\cohmods}{\mathbf{C}}
\newcommand*{\closedmods}{\overline{\cohmods}}
\newcommand*{\belmodtop}{\mathbf{1}_{\belmods}}
\newcommand*{\belmodbottom}{\mathbf{0}_{\belmods}}
\newcommand*{\modleq}{\sqsubseteq}
\newcommand*{\cohmodbottom}{\mathbf{0}_{\cohmods}}
\newcommand*{\cohmodclosure}{\clsr_{\cohmods}}

\newcommand*{\expand}{\mathrm{\ExpansionSymbol}}
\newcommand*{\revise}{\mathrm{\RevisionSymbol}}
\newcommand*{\contract}{\mathrm{\ContractionSymbol}}

\newcommand*{\A}{\AssessmentSymbol}
\newcommand*{\otherA}{B}
\newcommand*{\As}{\mathbf{\AssessmentSymbol}}
\newcommand*{\adfAs}{\adf\As}
\newcommand{\Astop}{\mathbf{1}_{\As}}

\newcommand*{\M}{\ModelSymbol}
\newcommand*{\Ms}{\mathbf{\ModelSymbol}}
\newcommand*{\closedMs}{\smash{\overline{\Ms}}}
\newcommand*{\closeddifMs}{\smash{\overline{\difMs}}}
\newcommand*{\propadfMs}{\smash{\Ms_{\textrm{\tiny\upshape{AD}}}^\textrm{\tiny\upshape{prop}}(\bgM)}}
\newcommand*{\prevadfMs}{\smash{\Ms_{\textrm{\tiny\upshape{AD}}}^\textrm{\tiny\upshape{prev}}(\bgM)}}

\newcommand*{\maxMs}{\hat{\Ms}}
\newcommand*{\adfMs}{\adf\Ms}
\newcommand*{\difMs}{\dif\Ms}
\newcommand*{\maxadfMs}{\adf\maxMs}

\newcommand*{\bgM}{\BackgroundModelSymbol}
\newcommand*{\zeroM}{{\bgM_o}}

\newcommand*{\eventM}[1][\eventopt]{\M_{#1}}

\newcommand*{\adf}[1]{#1_{\textrm{\tiny\upshape{AD}}}\vphantom{#1}}
\newcommand*{\dif}[1]{#1_{\textrm{\tiny\upshape{DI}}}\vphantom{#1}}

\newcommand*{\des}{\blacktriangleright}
\newcommand*{\rej}{\vartriangleleft}
\newcommand*{\acc}{\trianglerighteq}
\newcommand*{\indif}{\equiv}
\newcommand*{\unres}{\smile}

\newcommand*{\posetleq}{\sqsubseteq}

\newcommand*{\meet}{\frown}
\newcommand*{\join}{\smile}
\newcommand*{\weakgt}{>}
\newcommand*{\weakgeq}{\geq}
\newcommand*{\weakleq}{\leq}
\newcommand*{\stronggt}{\gtrdot}

\newcommand*{\indifset}[1][]{\IndifSymbol_{#1}}
\newcommand*{\eventindifset}[1][]{\indifset[{\eventopt[#1]}]}
\newcommand*{\condon}[1]{\CondSymbol{#1}}
\newcommand*{\altcondon}[1]{\AltCondSymbol{#1}}

\newcommand*{\define}[1]{\emph{\textcolor{definecolour}{#1}}\/}
\newcommand*{\nameit}[1]{\totheright{\upshape\textcolor{definecolour}{[#1]}}}
\newcommand*{\totheright}[1]{{\unskip\nobreak\hfil\penalty50\hskip1em\hbox{}\nobreak\hfil#1\parfillskip=0pt\finalhyphendemerits=0\par}}
\newcommand*{\reason}[1]{\textrm{[#1]}}
\newcommand*{\instantiateas}{\rightsquigarrow}

\DeclarePairedDelimiter\bra{\langle}{\rvert}
\DeclarePairedDelimiter\ket{\lvert}{\rangle}
\DeclarePairedDelimiterX\braket[2]{\langle}{\rangle}{#1\delimsize\vert#2}
\DeclarePairedDelimiterX\braketwithop[3]{\langle}{\rangle}{#1\delimsize\vert#2\delimsize\vert#3}
\DeclarePairedDelimiterXPP{\trace}[1]{\tr}{(}{)}{}{{#1}}

\DeclareMathOperator{\spec}{spec}
\DeclareMathOperator{\interior}{Int}

\DeclareMathOperator{\tr}{Tr}

\newcommand{\hilbertspace}{\mathscr{X}}

\newcommand{\Hop}{\measurements}
\newcommand{\subspace}[1][]{\mathscr{W}_{#1}}

\newcommand{\naturals}{\mathbb{N}}
\newcommand{\nonnegreals}{\reals_{\geq0}}
\newcommand{\posreals}{\reals_{>0}}

\newcommand{\fket}[1][]{\smash{\ket{\phi_{#1}}}}

\newcommand{\uket}{\ket{\Psi}}

\newcommand{\operator}[1]{\hat{#1}}
\newcommand{\measurement}[1]{\operator{#1}}
\newcommand{\measurements}{\mathscr{H}}
\newcommand{\projectoron}[1]{\measurement{P}_{#1}}
\newcommand{\projector}[1][]{\projectoron{\subspace[#1]}}
\newcommand{\density}[1][]{\measurement{\rho}_{#1}}
\newcommand{\identity}{\measurement{I}}
\newcommand{\zero}{\measurement{0}}
\newcommand{\spectrum}[1]{\spec(\measurement{#1})}

\title{Conditioning Accept-Desirability models in the context of AGM-like belief change}
\author{Kathelijne Coussement \and Gert de Cooman \and Keano De Vos}
\date{\today}
\address{Ghent University, Foundations Lab, Technologiepark-Zwijnaarde 125, 9052 Zwijnaarde, Belgium}
\usepackage{etoolbox}
\begin{document}
\begin{abstract}
We discuss conditionalisation for Accept-Desirability models in an abstract decision-making framework, where uncertain rewards live in a general linear space, and events are special projection operators on that linear space.
This abstract setting allows us to unify classical and quantum probabilities, and extend them to an imprecise probabilities context.
We introduce a new conditioning rule for our Accept-Desirability models, based on the idea that observing an event introduces new indifferences between options.
We associate a belief revision operator with our conditioning rule, and investigate which of the AGM axioms for belief revision still hold in our more general framework.
We investigate two interesting special cases where all of these axioms are shown to still hold: classical propositional logic and full conditional probabilities.
\end{abstract}

\maketitle

\section{Introduction}\label{sec::introduction}
Quite a few years ago already, \citeauthor{cooman2003a} \cite{cooman2003a} proposed a generalisation of the Alchourron, Gärdenfors and Makinson (AGM) framework \cite{gardenfors1988} for belief change---change of belief state---in propositional logic (and probability theory) to deal with what he called \emph{belief models}, and which are intended as much more general and abstract representations for what may constitute a belief state.

More recently, he and co-authors introduced a theory of accepting and rejecting gambles \cite{quaeghebeur2015:statement}, leading to the quite general and operationalisable \emph{Accept-Desirability framework} for decision-making under uncertainty.\footnote{The Accept-Desirability framework is an important special case of the yet more general Accept-Reject framework \cite{quaeghebeur2015:statement}, which, though interesting, is nevertheless too general for our present concerns and purposes.}
It provides a context for dealing jointly with such notions as acceptability, desirability, and indifference of gambles, as well as the relationships between them.
Also, it gives a general foundation to many, if not most, of the models commonly encountered in the field of imprecise probabilities, such as credal sets and coherent (conditional) lower previsions \cite{levi1980a,williams1975,walley1991,troffaes2013:lp} and sets of acceptable, or desirable (favourable), gambles; in fact, it allows for dealing simultaneously with \emph{acceptability} and \emph{desirability}, unifying the work of \citeauthor{seidenfeld1990} \cite{seidenfeld1990} and \citeauthor{walley2000} \cite{walley2000} on desirable gambles, and the work of \citeauthor{walley1991} \cite{walley1991} on acceptable gambles.
As opposed to \citeauthor{fishburn1986} \cite{fishburn1986}, it doesn't take either the strict or non-strict preference order to be basic, nor derives the other from it.
Instead, both preference orders are equally important, providing the possibility to express preferences in a more nuanced manner.

In a previous paper \cite{coussement25a}, which was intended as a first exploration, we combined ideas from both papers to show how the AGM framework for belief change can be extended to deal with conditioning on events in the so-called Desirability-Indifference framework, which is a special case of the Accept-Desirability framework \cite{quaeghebeur2015:statement}.

In the present paper, building on that work, we significantly extend our findings there to the more general setting of the Accept-Desirability framework.
What do we want to accomplish, and how will we go about it?

First, in \cref{sec::accept:reject}, we generalise the Accept-Desirability framework, which was originally formulated in terms of gambles \cite{quaeghebeur2015:statement}, towards working with more abstract uncertain rewards, called \emph{options}, and which live in a general linear space.
This allows us to represent a subject's preferences and beliefs about in terms of so-called \emph{Accept-Desirability models}, which are comprised of a set of acceptable options and a set of undesirable options.
These models come with an appropriate \emph{conservative inference} mechanism, together with notions of consistency and inferential closure.
A generalised notion of abstract \emph{events}, defined as special projection operators on the option space, and which can be used to call off options, is then briefly introduced in \cref{sec::events}.

We introduce this level of generality and abstraction because the theory we thus arrive at allows us to encompass---but isn't limited to---both classical probabilistic inference, as described by \citeauthor{walley2000} \cite{walley2000}, \citeauthor{finetti1970} \cite{finetti1970}, and \citeauthor{williams1975} \cite{williams1975}, amongst others, and quantum probabilistic inference, as studied  \citeauthor{benavoli2016:quantum_2016} \cite{benavoli2016:quantum_2016} and by \citeauthor{devos2025:imprecise:qm}, \cite{devos2025:imprecise:qm,devos2025:isipta}.
We'll use these two special cases as running examples, to allow the reader to gain some intuition, and find out what the abstract framework amounts to in these more concrete instances.

Using these abstract notions of options and events, we introduce a novel approach to conditioning on events in \cref{sec::conditioning}, which is intended to work for general Accept-Desirability models, and therefore in particular for all their special cases, such as (classical and quantum) probabilities, coherent lower previsions, sets of acceptable and sets of desirable gambles.
The novelty consists in expressing the new information that an event has occurred in terms of an indifference statement: two options become indifferent when they result in the same called-off option.
Our conditioning rule then combines earlier ideas about how to deal with called-off options \cite{finetti19745,walley1991,williams1975} with new ideas on how to treat indifference within the Accept-Desirability framework.

\cref{sec::belief:change} has a succinct summary of \citeauthor{cooman2003a}'s extension of the AGM framework for belief change to his abstract \emph{belief models}.
With this scaffolding in place, we show in \cref{sec::conditioning:as:belief:change} how Accept-Desirability  models fit into the belief model framework, and how our abstract notion of conditioning can be seen as a form of belief change---\emph{belief revision}---in the Accept-Desirability framework, satisfying relevant generalised versions of the AGM axioms.
Indeed, we prove that in our more general context, many, but not necessarily all, of the original AGM axioms for belief revision still hold; we show by means of counterexamples that two of these, expressing that revision ought to coincide with deductive inference when the incoming new information is consistent with the old beliefs, fail to hold in general for our revision rule based on conditioning; we also argue that these failures occur because of a possible loss of precision when conditioning, which is associated with the phenomenon of \emph{dilation}.

In \cref{sec::propositional:models}, we show that inference in \emph{classical propositional logic} can be seen as a special case of the conservative inference mechanism associated with Accept-Desirability models, and we investigate what our revision operator does when restricted to Accept-Desirability models of the propositional type.
This allows us to prove that \emph{all} the AGM axioms hold for our revision operator when we restrict it to this propositional context.

\cref{sec::full:conditional:previsions} looks at full conditional probability models, and shows how these also can be seen as---put in a one-to-one correspondence with---specific types of Accept-Desirability models.
An investigation of how our revision operator acts on such so-called \emph{full conditional} Accept-Desirability models allows us to confirm that for such models too, \emph{all} of the AGM axioms are verified.
It also enables us to show in detail how our conditioning rule coincides with Bayes' Rule for precise probabilities, and exactly what information our revision operator removes from an existing full conditional model when there's no consistency between the new observation of an event and the existing model.

We then briefly summarise our results, and draw some preliminary conclusions from them, in \cref{sec::discussion}.
To improve readability, we've decided to move technical proofs and lemmas to an appendix.

\section{Statement models}\label{sec::accept:reject}
We consider an agent, called You, who's asked to make accept and/or reject statements about options.
\define{Options} are abstract objects that are intended to represent uncertain rewards.
When You \define{accept} an option, You state that You agree to receive an uncertain reward\footnote{What exactly it is that You pay or receive, can be expressed in utiles. We assume that having more utiles is always the desired outcome, and that utility scales linearly.} that depends on something You're uncertain about.
\define{Rejecting} an option is making the statement that accepting it is something You don't agree to.
However, it may be that You don't have enough information to make an accept or a reject statement about a given option; in that case You're allowed to remain \define{unresolved}.

We'll assume that options~\(\opt\) can be added and multiplied with real numbers, with all the usual properties, so that they live in some real linear space~\(\opts\), called the \define{option space}.
The null option~\(0\) represents the Status Quo, and we'll assume that Your decision problem is \emph{non-trivial} in the sense that \(\opts\neq\set{0}\).

We'll stick to this abstract setting for most of this paper, because it allows us to remain as general as possible, and to apply our results in a wide range of contexts.
But, in order to provide some intuition, as well as evidence for the versatility of our abstract setting, we'll also consider two specific instances: classical probabilistic reasoning and quantum probabilistic reasoning.
We'll discuss these as two examples running through much of the paper, building up a more detailed picture as we go along.
There are other reasons for working with abstract options, besides that they provide a framework that allows us to deal with classical and quantum probabilities in one fell swoop: they allow dealing elegantly with representation under indifference, and also include dealing with horse lotteries as a special case; see the discussion in \cite{cooman2021:archimedean:choice} for more details and examples.

\begin{gamblesexample}\label{example:gambles}
The first running example we'll consider, deals with classical probabilistic reasoning.
In this context, we focus on a variable~\(\varstate\) that assumes values in some non-empty set~\(\states\), but whose actual value is unknown to You.

With any bounded map~\(\gbl\colon\states\to\reals\), called \define{gamble}, there corresponds an uncertain reward \(\gbl(\varstate)\), expressed in units of some linear utility scale.
The set of all gambles, denoted by~\(\gambles(\states)\), or simply by~\(\gambles\) if it's clear from the context what the \define{possibility space}~\(\states\) is, constitutes a real linear space under pointwise addition and pointwise scalar multiplication with real numbers, and thus serves as an option space: gambles are the options You can accept or reject in this classical decision problem.
The Status Quo is represented by the null gamble~\(0\), which is the constant gamble that always returns~\(0\).
\end{gamblesexample}

\begin{quantumexample}\label{example:quantum}
The second running example deals with quantum probabilistic reasoning \cite{devos2025:imprecise:qm,devos2025:isipta,benavoli2016:quantum_2016}.
In this context, we consider a quantum system whose unknown state~\(\uket\) lives in a \emph{finite}-dimensional state space~\(\hilbertspace\), which is a(n \(n\)-dimensional) complex Hilbert space.
Options in this context are the Hermitian operators~\(\measurement{A}\) on~\(\hilbertspace\) corresponding to \define{measurements} on the system, where the uncertain outcome of a measurement~\(\measurement{A}\)---one of its eigenvalues---is interpreted as an uncertain reward, expressed in units of some linear utility scale.
The set of all such Hermitian operators~\(\measurement{A}\) constitutes an \(n^2\)-dimensional real linear space~\(\measurements\), and this is Your option space for this quantum decision problem: measurements are the options You can accept or reject.
The Status Quo is represented by the null operator~\(\measurement{0}\), all of whose eigenvalues are zero; it corresponds to a measurement that's sure to give zero as an outcome.
\end{quantumexample}

The set of those options You accept, is denoted by~\(\A_\acc\) and the set of options You reject by~\(\A_\rej\).
Together, these sets form an \define{assessment}~\(\A\coloneqq\Adelim{\A_\acc}{\A_\rej}\) that describes Your behaviour regarding those options You care to make statements about.
We'll collect all possible assessments in the set~\(\As\coloneqq\powersetof\opts\times\powersetof\opts\); in other words, assessments are pairs of sets of options.

The options You're unresolved about---the \define{unresolved} options---are given by
\(\A_{\unres}\coloneqq\compof{\A_\acc\cup\A_\rej}\), where `\(\comp\)' denotes set complement.
We can also identify Your set of \define{indifferent} options~\(\A_\indif\coloneqq\A_\acc\cap-\A_\acc\) and Your set of \define{desirable} options~\(\A_\des\coloneqq\A_\acc\cap-\A_\rej\):\footnote{The minus sign before a set denotes the \define{Minkowski additive inverse}. For example, \(-\A_\acc\coloneqq\set{-\opt\given\opt\in\A_\acc}\). We'll also use the \define{Minkowski sum}, so for instance \(\A_\acc+\A_\rej\coloneqq\set{\opt+\altopt\given\opt\in\A_\acc\text{ and }\altopt\in\A_\rej}\). The Minkowski difference is defined similarly.} You're indifferent about an option when You both want to get it \emph{and} give it away, and You find an option desirable when You want it \emph{but} don't want to give it away.
Options in \(\A_acc\cap\A_rej\) are called \define{confused}; such confusion is clearly to be avoided, and is considered to be a form of inconsistency, or irrationality, in Your assessment.

If we define inclusion for assessments by letting~\(\A\subseteq\otherA\ifandonlyif\group{\A_\acc\subseteq\otherA_\acc\text{ and }\A_\rej\subseteq\otherA_\rej}\) for all~\(\A,\otherA\in\As\), then the structure \(\structure{\As,\subseteq}\) is a complete lattice with meet and join respectively defined by~\(\A\intersection\otherA\coloneqq\Adelim{\A_\acc\intersection\otherA_\acc}{\A_\rej\intersection\otherA_\rej}\) and \(\A\union\otherA\coloneqq\Adelim{\A_\acc\union\otherA_\acc}{\A_\rej\union\otherA_\rej}\), and with top~\(\Astop\coloneqq\Adelim{\opts}{\opts}\) and bottom~\(\Adelim{\emptyset}{\emptyset}\).
If \(\A\subseteq\otherA\), we say that \(\A\) is \define{less resolved} than~\(\otherA\).

\subsection{Statement models}\label{sec::statement:models}
We now focus on a special subset~\(\Ms\subseteq\As\) of assessments, called \define{statement models}.
For an assessment to be called a statement model, four criteria have to be met:
\begin{enumerate}[label={\upshape M\arabic*.},ref={\upshape M\arabic*},series=AR,widest=4,leftmargin=*,itemsep=0pt]
\item\label{axiom:model:background:respected} \(0\in\M_\acc\);\nameit{Indifference to Status Quo}
\item\label{axiom:model:null:not:rejected} \(0\notin\M_\rej\);\nameit{No Confusion}
\item\label{axiom:model:accepts:convex:cone} \(\M_\acc\) is a convex cone;\nameit{Deductive Closure}
\item\label{axiom:model:no:limbo} \(\shullof{\M_\rej}-\M_\acc\subseteq\M_\rej\);\nameit{No Limbo}
\end{enumerate}
\noindent where \(\shullof{B}\) is the set of all (strictly) positive scalar multiples of the options in a set~\(B\subseteq\opts\).
We see that \(\structure{\Ms,\subseteq}\) is a \emph{complete meet-semilattice}: \(\Ms\) is closed under taking non-empty infima.
This means that it comes with a \emph{conservative inference} method, as it allows us to define a \define{closure operator}
\begin{equation*}
\Mcls\colon\As\to\Ms\cup\set{\Astop}\colon\A\mapsto\Intersection\set{\M\in\Ms\given\A\subseteq\M},
\end{equation*}
which maps any assessment to the \emph{least resolved} statement model including it, if there's any statement model that includes it.
This last condition, namely that \(\Mclsof{\A}\in\Ms\), will be satisfied if and only if \(\posiof{\A_\acc}\cap\A_\rej=\emptyset\), in which case we'll call~\(\A\) \define{deductively closable}.
In that case also
\begin{equation}\label{eq::model:closure}
\Mclsof{\A}
=\Adelim{\posiof{\A_\acc}}{\shullof{\A_\rej}\cup\group{\shullof{\A_\rej}-\posiof{\A_\acc}}},
\end{equation}
where \(\posiof{B}\) is the smallest convex cone that includes the set~\(B\subseteq\opts\), or in other words, the set of all (strictly) positive convex combinations of the elements of~\(B\).
\labelcref{axiom:model:background:respected,axiom:model:null:not:rejected,axiom:model:accepts:convex:cone,axiom:model:no:limbo} are the minimal rationality requirements we'll impose on Accept-Reject statements.
For a justification of these requirements, and a thorough discussion of these and all other results and notions mentioned in this \cref{sec::accept:reject}, we refer to \cite{quaeghebeur2015:statement}.

\subsection{About the background}\label{sec::background}
Axiom~\labelcref{axiom:model:background:respected}, or equivalently, \(\zeroM\coloneqq\Adelim{\set{0}}{\emptyset}\subseteq\M\), requires that You should accept---even be indifferent to---the \define{Status Quo}, represented by the null option~\(0\).
\(\zeroM\) is an example of what we'll call a \define{background} (statement) \define{model}: a statement model that You always accept without any real introspection, regardless of any relevant information You might have.
In specific cases, this background model, generically denoted by~\(\bgM\coloneqq\Adelim{\bgM_\acc}{\bgM_\rej}\in\Ms\), may be larger, but we'll always have that \(\zeroM\subseteq\bgM\).
We'll then require that the \define{background}~\(\bgM\) should be \define{respected}, which amounts to replacing \labelcref{axiom:model:background:respected} by
\begin{enumerate}[label={\upshape M1*.},ref={\upshape M1*},leftmargin=*,itemsep=0pt]
\item\label{axiom:stronger:background:respected} \(\bgM\subseteq\M\).\nameit{Background}
\end{enumerate}
The set of all statement models that respect~\(\bgM\) is denoted by~\(\Msabove{\bgM}\coloneqq\set{\M\in\Ms\given\bgM\subseteq\M}\).
It's also closed under arbitrary infima and therefore also comes with its own \emph{conservative inference} method; for the corresponding closure operator we have that \(\Mclsbgof{\bgM}{\A}=\Mclsof{\bgM\union\A}\) for all~\(\A\in\As\).
The augmented assessment~\(\A\union\bgM\) is deductively closable if there's some statement model~\(\M\in\Msabove{\bgM}\) such that \(\A\union\bgM\subseteq\M\), which is equivalent to
\begin{equation}\label{eq::consistency}
\posiof{\A_\acc\cup\bgM_\acc}\cap\group{A_\rej\cup\bgM_\rej}=\emptyset.
\end{equation}
We'll then call the assessment~\(\A\) \define{\(\bgM\)-consistent}.
In that case, we'll also call the statement model~\(\Mclsbgof{\bgM}{\A}=\Mclsof{\A\union\bgM}\) the \define{\(\bgM\)-natural extension} of the assessment~\(\A\).

The background model~\(\bgM=\Mclsbgof{\bgM}{\emptyset}\) is the bottom of the complete meet-semilattice \(\structure{\Msabove{\bgM},\subseteq}\), and is also called the \define{vacuous} statement model.

\subsection{The Accept-Desirability framework}\label{sec::accept:desirability}
In the context of this paper, the statement models constitute too large a class; we'll restrict the discussion to a special subclass of them.
\begin{definition}
A statement model~\(\M\in\Ms\) is called an \define{Accept-Desirability model}, or \define{AD-model} for short, if it satisfies the \emph{Accept-Desirability condition:}
\begin{equation}\label{eq::AD:condition}
\M_\rej\subseteq-\M_\acc.
\tag{AD}
\end{equation}
\end{definition}
\noindent This condition stipulates that You should only reject options that You want to give away.
An equivalent requirement is that~\(\M_\des=-\M_\rej\), so an AD-model~\(\M\) can be completely described by specifying its set of acceptable options~\(\M_\acc\) and its set of desirable options~\(\M_\des\), in the sense that \(\M=\Adelim{\M_\acc}{-\M_\des}\).

We'll denote the set of all AD-models by~\(\adfMs\), and the set of all AD-models that respect a given AD-background model~\(\bgM\) by~\(\adfMsabove{\bgM}\), with~\(\bgM\in\adfMsabove{\zeroM}\).
We call an assessment~\(\A=\Adelim{\A_\acc}{\A_\rej}\) an \define{AD-assessment} if it satisfies the AD-condition~\(\A_\rej\subseteq-\A_\acc\), and we collect the AD-assessments in the set \(\adfAs\coloneqq\set{\A\in\As\given\A_\rej\subseteq-\A_\acc}\).

For AD-models, the rationality criteria \labelcref{axiom:stronger:background:respected,axiom:model:null:not:rejected,axiom:model:accepts:convex:cone,axiom:model:no:limbo} can be rewritten as follows in terms of the sets~\(\M_\acc\) and~\(\M_\des\):

\begin{proposition}\label{prop::ad:equivalent:criteria}
Consider any background~\(\bgM\in\adfMsabove{\zeroM}\), then an AD-assessment~\(M\) is an AD-model that respects the background~\(\bgM\), so \(M\in\adfMsabove{\bgM}\), if and only if
\begin{enumerate}[label={\upshape AD\arabic*.},ref={\upshape AD\arabic*},series=AD,widest=4,leftmargin=*,itemsep=0pt]
\item\label{axiom:AD:background} \(V\subseteq M\);\nameit{Background}
\item\label{axiom:AD:zero:not:desirable} \(0\notin\M_\des\);\nameit{Strictness}
\item\label{axiom:AD:deductive:closedness} \(\M_\des\) and \(\M_\acc\) are convex cones;\nameit{Deductive Closedness}
\item\label{axiom:AD:no:limbo} \(\M_\acc+\M_\des\subseteq\M_\des\).\nameit{Sweetened Deals}
\end{enumerate}
\end{proposition}
\noindent As an immediate consequence, we find that also
\begin{enumerate}[resume*=AD,widest=5]
\item\label{axiom:AD:nonpositive:not:acceptable} \(0\notin\bgM_\des+\M_\acc\).
\end{enumerate}

\subsection{AD-models and the literature}
An AD-model~\(\M=\Adelim{\M_\acc}{-\M_\des}\) has \emph{both} a set of acceptable options~\(M_\acc\) and a set of desirable options~\(\M_\des\).
In the context of gambles and classical probabilistic inference, such sets appear \emph{on their own} in quite a number of papers: sets of acceptable gambles in work by \citeauthor{smith1961} \cite{smith1961}, \citeauthor{williams1975} \cite{williams1975}, \citeauthor{walley1991} \cite{walley1991}, and \citeauthor{troffaes2013:lp} \cite{troffaes2013:lp}; sets of desirable gambles\footnote{In some, but not all, of these works, sets of desirable gambles can actually be usefully interpreted as sets of acceptable gambles without the zero gamble.} in work by \citeauthor{walley2000} \cite{walley2000,walley1991}, \citeauthor{moral2003} \cite{moral2003}, \citeauthor{couso2011} \cite{couso2011}, \citeauthor{cooman2021:archimedean:choice} and colleagues \cite{cooman2010,decooman2015:coherent:predictive:inference,cooman2007d,cooman2021:archimedean:choice}, and \citeauthor{miranda2010} \cite{miranda2010,miranda2022:nonlinear:desirability,zaffalon2017:incomplete:preferences}.
To the best of our knowledge, they were studied as an interacting pair in their full generality for the first time in \cite{quaeghebeur2015:statement}, of which the account above is at the same time a very succinct summary, as well as a generalisation from gambles to more abstract options.

In the context of quantum theory, sets of desirable measurements appear in papers by \citeauthor{benavoli2016:quantum_2016} \cite{benavoli2016:quantum_2016} and \citeauthor{devos2025:imprecise:qm} \cite{devos2025:imprecise:qm,devos2025:isipta}.

\subsection{Further assumptions about the AD-background}
As already hinted at above, when working in the Accept-Desirability framework, we'll typically adopt a background that's also of the AD-form, so
\begin{enumerate}[label={\upshape BG\arabic*.},ref={\upshape BG\arabic*},series=background,widest=1,leftmargin=*,itemsep=0pt]
\item\label{axiom:background:model} \(\bgM\coloneqq\Adelim{\bgM_\acc}{-\bgM_\des}\in\adfMsabove{\zeroM}\).\nameit{Background is AD}
\end{enumerate}
With this background~\(\bgM\), we can associate two important vector orderings:
\begin{equation}\label{eq::option:orderings}
\opt\geq\altopt\ifandonlyif\opt-\altopt\in\bgM_\acc\text{ and }\opt\stronggt\altopt\ifandonlyif\opt-\altopt\in\bgM_\des,
\text{ for all~\(\opt,\altopt\in\opts\)}.
\end{equation}
It follows from \cref{axiom:AD:background,axiom:AD:zero:not:desirable,axiom:AD:deductive:closedness,axiom:AD:no:limbo} that the first ordering~\(\geq\) will be a vector preorder and the second ordering~\(\stronggt\) a strict vector ordering.

Interestingly, the background can be reconstructed from these two orderings:
\begin{equation*}
\opts_{\geq0}
\coloneqq\set{\opt\in\opts\colon\opt\geq0}
=\bgM_\acc\text{ and }\opts_{\stronggt0}
\coloneqq\set{\opt\in\opts\colon\opt\stronggt0}
=\bgM_\des.
\end{equation*}
From the set of background-indifferent options~\(\bgM_\equiv\coloneqq\bgM_\acc\cap-\bgM_\acc\), we derive the equivalence relation~\(\equiv\):
\begin{equation*}
\opt\equiv\altopt
\ifandonlyif\opt-\altopt\in\bgM_\equiv,
\text{ for all~\(\opt,\altopt\in\opts\)}.
\end{equation*}
In the Accept-Desirability framework, we \emph{don't} require that the vector preorder~\(\geq\) should be derivable from the strict ordering~\(\stronggt\) and the equivalence relation~\(\equiv\), in the sense that \(\opt\geq\altopt\ifandonlyif\group{\opt\stronggt\altopt\text{ or }\opt\equiv\altopt}\) for all~\(\opt,\altopt\in\opts\).\footnote{This connection does hold in the more restrictive setting of so-called Desirability-Indifference models \cite{quaeghebeur2015:statement}, whose relation to AGM belief change was explored in an earlier version of this work \cite{coussement25a}.}

It's a matter of direct verification that the set~\(\adfMsabove{\bgM}\) of all AD-models that respect a given AD-background~\(\bgM\) is also closed under non-empty infima, and therefore also allows for \emph{conservative inference}.
In fact, the closure~\(\Mclsbgof{\bgM}{\A}\) of an assessment~\(\A=\Adelim{\A_\acc}{-\A_\des}\) that satisfies the AD-condition~\(\A_\rej\subseteq-\A_\acc\), is still of the AD-type, so \(\Mclsbgof{\bgM}{\A}\in\adfMsabove{\bgM}\).
In this sense, conservative inference for AD-models uses the same inference mechanism---closure operator---as it does for the more general statement models.

To be able to associate lower and upper previsions with AD-models further on, we make a number of further assumptions about the background AD-model~\(\bgM=\Adelim{\bgM_\acc}{-\bgM_\des}\).
First, we'll assume the existence of an \emph{order unit}~\(\optunit\) for the background strict vector ordering~\(\stronggt\), and therefore also for the background partial vector preorder~\(\weakgeq\):
\begin{enumerate}[resume*=background,widest=2]
\item\label{axiom:background:order:unit} there's some option~\(\optunit\in\bgM_\des\) such that \(\group{\forall\opt\in\opts}\group{\exists\lambda\in\posreals}\,\lambda\optunit\stronggt\opt\stronggt-\lambda\optunit\).\nameit{Order Unit}
\end{enumerate}
We'll also call \(\optunit\) the \define{unit option}.
We're furthermore going to assume that this unit option connects the background partial vector preorder~\(\weakgeq\) and the background strict vector ordering~\(\stronggt\) in the following way:
\begin{enumerate}[resume*=background,widest=3]
\item\label{axiom:background:order:unit:connection} if \(\opt\stronggt0\) then there's some \(\alpha\in\posreals\) such that \(\opt\weakgeq\alpha\optunit\), for all~\(\opt\in\opts\).
\end{enumerate}

\begin{gamblesexamplecont}
For the background vector preorder~\(\weakgeq\), we take the \define{weak dominance ordering}, defined by
\begin{equation*}
\gbl\weakgeq\altgbl\ifandonlyif\inf\group{\gbl-\altgbl}\geq0,
\text{ for all~\(\gbl,\altgbl\in\gambles\)},
\end{equation*}
and for the background strict vector ordering~\(\stronggt\), we choose the \define{strong dominance ordering}, defined by
\begin{equation*}
\gbl\stronggt\altgbl\ifandonlyif\inf\group{\gbl-\altgbl}>0,
\text{ for all~\(\gbl,\altgbl\in\gambles\)},
\end{equation*}
where we let \(\inf\gbl\coloneqq\inf\set{\gbl\group\state\given\state\in\states}\).
Hence, for the AD-background~\(\bgM=\Adelim{\bgM_\acc}{-\bgM_\des}\), we have that \(\bgM_\acc=\gambles_{\weakgeq0}=\set{\gbl\in\gambles\given\inf\gbl\geq0}\) and \(\bgM_\des=\gambles_{\stronggt0}=\set{\gbl\in\gambles\given\inf\gbl>0}\).
Note that \(\bgM_\indif=\set{0}\) and that, typically, \(\bgM_\des\cup\bgM_\indif\neq\bgM_\acc\).

For the order unit for the background strict vector ordering~\(\stronggt\), or in other words the \define{unit gamble}, we can take the constant gamble~\(1\), since for any gamble~\(\gbl\in\gambles\), we have that \(\lambda\stronggt\gbl\stronggt-\lambda\) for all real~\(\lambda>\sup\abs{\gbl}\), so \labelcref{axiom:background:order:unit} holds.
Moreover, \(\gbl\stronggt0\) means that \(\inf\gbl>0\), so there's some real~\(\epsilon>0\) such that \(\gbl\weakgeq\epsilon\), making sure that \labelcref{axiom:background:order:unit:connection} also holds.
\end{gamblesexamplecont}

\begin{quantumexamplecont}
We denote by~\(\spectrum{A}\) the set of all (real) eigenvalues--- possible outcomes---of the measurement~\(\measurement{A}\).
For the background vector preorder~\(\geq\), we take the one associated with \define{positive semidefiniteness}, defined by
\begin{equation*}
\measurement{A}\weakgeq\measurement{B}
\ifandonlyif\min\spec\group{\measurement{A}-\measurement{B}}\geq0,
\text{ for all~\(\measurement{A},\measurement{B}\in\Hop\)}.
\end{equation*}
For the background strict vector ordering~\(\stronggt\), we opt for a similar definition, namely the one associated with \define{positive definiteness}:
\begin{equation*}
\measurement{A}\stronggt\measurement{B}
\ifandonlyif\min\spec\group{\measurement{A}-\measurement{B}}>0,
\text{ for all~\(\measurement{A},\measurement{B}\in\Hop\)}.
\end{equation*}
The AD-background~\(\bgM=\Adelim{\bgM_\acc}{-\bgM_\des}\) has \(\bgM_\acc=\Hop_{\weakgeq0}=\set{\measurement{A}\in\Hop\given\min\spectrum{A}\geq0}\) and \(\bgM_\des=\Hop_{\stronggt0}=\set{\measurement{A}\in\Hop\given\min\spectrum{A}>0}\).
Observe that \(\bgM_\indif=\set{\measurement{0}}\) and that, typically, \(\bgM_\des\cup\bgM_\indif\neq\bgM_\acc\).

For the order unit for the background strict vector ordering~\(\stronggt\), or in other words for the \define{unit measurement}, we can take the identity~\(\identity\), since for any measurement~\(\measurement{A}\in\measurements\), we have that \(\lambda\identity\stronggt\measurement{A}\stronggt-\lambda\identity\) for all~\(\lambda>\max\set{\abs{\lambda}\given\lambda\in\spectrum{A}}\).
Moreover, \(\measurement{A}\stronggt0\) means that \(\min\spectrum{A}>0\), so there's some real~\(\epsilon>0\) such that \(\measurement{A}\weakgeq\epsilon\identity\), making sure that \labelcref{axiom:background:order:unit:connection} also holds.
\end{quantumexamplecont}

\subsection{AD-models and coherent (lower and upper) previsions}\label{sec::AD-previsions}
That sets of desirable or acceptable gambles can be used to define what are sometimes called coherent (lower and upper) previsions, has been discussed by several authors, amongst whom \citeauthor{smith1961} \cite{smith1961}, \citeauthor{williams1975} \cite{williams1975}, \citeauthor{walley1991} \cite{walley1991,walley2000}, and \citeauthor{troffaes2013:lp} \cite{troffaes2013:lp}, as well as by \citeauthor{cooman2021:archimedean:choice} \cite{cooman2021:archimedean:choice} in more general contexts where gambles are replaced by the more abstract options.
It will be useful to briefly discuss the connections between AD-models and coherent (lower and upper) previsions here as well.

We can associate, with any AD-model \(\M\in\adfMsabove{\bgM}\), two real-valued\footnote{That these functionals are real-valued was proved by \citeauthor{cooman2021:archimedean:choice} in \cite{cooman2021:archimedean:choice}.  The underlying deeper reason for their being real-valued is that, since \(\optunit\in\interior\group{\bgM_\des}\), as we assumed in \labelcref{axiom:background:order:unit}, the unit option acts as an \emph{order unit} for the convex cone \(\bgM_\des\), and therefore also for the convex cones~\(\M_\des\) and \(\M_\acc\) that include it.} price functionals: the \define{buying price functional} \(\lprev_{\M}\colon\opts\to\reals\) and the \define{selling price functional} \(\uprev_{\M}\colon\opts\to\reals\), defined by\footnote{That the expressions involving~\(\M_\acc\) and~\(\M_\des\) are equivalent, is due to \(\M_\des\subseteq\M_\acc\) and \(\M_\acc+\M_\des\subseteq\M_\des\).}
\begin{multline}\label{eq::definition:lprev:uprev}
\left.
\begin{aligned}
\lprev_{\M}\group\opt
&\coloneqq\sup\set{\alpha\in\reals\given\opt-\alpha\optunit\in\M_\acc}
=\sup\set{\alpha\in\reals\given\opt-\alpha\optunit\in\M_\des}\\
\uprev_{\M}\group\opt
&\coloneqq\inf\set{\alpha\in\reals\given\alpha\optunit-\opt\in\M_\acc}
=\inf\set{\alpha\in\reals\given\alpha\optunit-\opt\in\M_\des}
\end{aligned}
\right\}\\
\text{for all~\(\opt\in\opts\)}.
\end{multline}
Observe that these price functionals are \define{conjugate} in the sense that
\begin{equation*}
\lprev_{\M}\group\opt
=-\uprev_{\M}\group{-\opt}
\text{ for all~\(\opt\in\opts\)}.
\end{equation*}
Moreover, \(\lprev_{\M}\group\opt\leq\uprev_{\M}\group\opt\) for all~\(\opt\in\opts\), \(\lprev_{\M}\) is super-linear and \(\uprev_{\M}\) is sub-linear; for details, other properties and proofs, see \cite{cooman2021:archimedean:choice}.

A real-valued functional~\(\Lambda\) on the option space~\(\opts\) is a \define{\(\bgM\)-coherent lower prevision} if there's some AD-model~\(\M\in\adfMsabove{\bgM}\) such that \(\Lambda=\lprev_{\M}\), and a \define{\(\bgM\)-coherent upper prevision} if there's some AD-model~\(\M\in\adfMsabove{\bgM}\) such that \(\Lambda=\uprev_{\M}\).\footnote{There are more direct ways of characterising coherent (lower and upper previsions); for details, see \cite{cooman2021:archimedean:choice}.}

\(\Lambda\) is a \define{\(\bgM\)-coherent prevision} if it's both a \(\bgM\)-coherent lower prevision and a \(\bgM\)-coherent upper prevision, or in other words if there's some AD-model~\(\M\in\adfMsabove{\bgM}\) such that \(\Lambda=\lprev_{\M}=\uprev_{\M}\eqqcolon\prev_{\M}\).
It's a matter of straightforward verification that a \(\bgM\)-coherent prevision is a real \emph{linear functional}, so \(\prev_{\M}\group{\lambda\opt+\mu\altopt}=\lambda\prev_{\M}\group\opt+\mu\prev_{\M}\group\altopt\) for all~\(\opt,\altopt\in\opts\) and \(\lambda,\mu\in\reals\).
It's moreover \emph{constant additive} in the sense that \(\prev_{\M}\group{\opt+\alpha\optunit}=\prev_{\M}\group\opt+\alpha\) for all~\(\opt\in\opts\) and \(\alpha\in\reals\).
Finally, it's \emph{positive} in the sense that \(\prev_{\M}\group{\opt}\geq0\) for all~\(\opt\in\opts_{\weakgeq0}\).

\begin{gamblesexamplecont}
In a probability context, what we call \(\Adelim{\gambles_{\weakgeq0}}{-\gambles_{\stronggt0}}\)-coherent lower previsions here are the coherent lower previsions in the sense of \citeauthor{williams2007} \cite{williams2007} and \citeauthor{walley1991} \cite{walley1991}, and the \(\Adelim{\gambles_{\weakgeq0}}{-\gambles_{\stronggt0}}\)-coherent previsions are \citeauthor{finetti19745}'s \cite{finetti19745} coherent previsions; they are continuous, positive and normalised linear real functionals on the linear space of all gambles, and can be seen as expectation functionals associated with finitely additive probability measures.
\end{gamblesexamplecont}

\begin{quantumexamplecont}
In a quantum context, what we call \(\Adelim{\measurements_{\weakgeq0}}{-\measurements_{\stronggt0}}\)-coherent lower previsions here are the coherent lower previsions on the space of all measurements introduced by \citeauthor{benavoli2016:quantum_2016} \cite{benavoli2016:quantum_2016}. They were also studied in more detail, and given a different motivation and interpretation by \citeauthor{devos2025:imprecise:qm} \cite{devos2025:imprecise:qm}.
The \(\Adelim{\measurements_{\weakgeq0}}{-\measurements_{\stronggt0}}\)-coherent previsions~\(\prev\) are  continuous, positive and normalised linear real functionals on the linear space of all measurements, and they are in a one-to-one relationship with the density operators~\(\density\) on the Hilbert space~\(\hilbertspace\), in the sense that \(\prev(\measurement{A})=\trace{\density\measurement{A}}\) for all measurements~\(\measurement{A}\).
\end{quantumexamplecont}

\section{Events}\label{sec::events}
Suppose that after You've expressed Your beliefs by means of an AD-model~\(\M\), You obtain new information in the form of the occurrence of an event~\(\eventopt\).
How will Your AD-model change, given this new information?

Before answering this question in the form of a so-called \define{updated AD-model}, we first need to characterise what \emph{events} are.
We'll do so indirectly, by establishing the effect events have on options.
This abstract approach will allow us, here too, to deal with the case of classical and quantum probability simultaneously.

We'll assume the existence of a non-empty collection of \define{events}~\(\eventopts\), as well as a \define{calling-off operation} \(\ast\):
\begin{equation}
\ast\colon\eventopts\times\opts\mapsto\opts\colon(\eventopt,\opt)\mapsto\calledoff{\opt}.
\end{equation}
This calling-off operation should satisfy the following axioms:
\begin{enumerate}[label={\upshape E\arabic*.},ref={\upshape E\arabic*},series=events,widest=3,leftmargin=*,itemsep=0pt]
\item\label{axiom:event:linear} \(\calledoff{(\opt+\lambda\altopt)}=\calledoff{\opt}+\lambda\group{\calledoff{\altopt}}\) for \(\eventopt\in\eventopts\), \(\opt,\altopt\in\opts\) and \(\lambda\in\reals\);\nameit{Linearity}
\item\label{axiom:event:idempotent}\(\calledoff{\opt}=\calledoff{(\calledoff{\opt})}\) for all~\(\eventopt\in\eventopts\) and all~\(\opt\in\opts\);\nameit{Idempotency}
\item\label{axiom:event:monotone} \(\opt\geq0\then\calledoff{\opt}\geq0\) for all~\(\eventopt\in\eventopts\) and all~\(\opt\in\opts\).\nameit{Monotonicity}
\end{enumerate}
By \labelcref{axiom:event:linear,axiom:event:idempotent}, the calling-off operation~\(\calledoff{\bolleke}\) associated with an event~\(\eventopt\in\eventopts\) is a linear projection operator on the linear option space~\(\opts\).
The corresponding range is also called \define{the called-off space} \(\calledoff{\opts}\coloneqq\set{\calledoff{\opt}\colon\opt\in\opts}\), and the kernel is given by
\begin{equation}\label{eq::definition:kernel}
\eventindifset\coloneqq\set{\opt\in\opts\colon\calledoff{\opt}=0}.
\end{equation}

\noindent Furthermore, we require the existence of a \define{unit event}~\(\unitevent\) and a \define{null event}~\(\nullevent\):
\begin{enumerate}[resume*=events,widest=4]
\item\label{axiom:event:unit:event}  \(\unitevent\in\eventopts\) and \(\unitevent\ast\opt=\opt\) for all~\(\opt\in\opts\);\nameit{Unit Event}
\item\label{axiom:event:null:event}  \(\nullevent\in\eventopts\) and \(\nullevent\ast\opt=0,\) for all~\(\opt\in\opts\).\nameit{Null Event}
\end{enumerate}
We also define an \define{event ordering} \(\posetleq\):
\begin{equation}
\eventopt[1]\posetleq\eventopt[2]\ifandonlyif\group{\forall\opt\in\opts}\group{\calledoff[1]{\opt}=\calledoff[1]{(\calledoff[2]{\opt})}=\calledoff[2]{(\calledoff[1]{\opt})}}.
\end{equation}
The meaning of `\(\eventopt[1]\posetleq\eventopt[2]\)' is that `the event~\(\eventopt[1]\) implies the event~\(\eventopt[2]\)': in whatever order the events are invoked, calling off under event~\(\eventopt[1]\) makes calling off under event~\(\eventopt[2]\) redundant.
Since the relation~\(\posetleq\) is obviously reflexive, antisymmetric and transitive, \(\structure{\eventopts,\posetleq}\) is a partially ordered set, and it has top~\(\unitevent\) and bottom~\(\nullevent\).

The following axiom then imposes a link between the weak ordering~\(\weakgeq\) on options as defined in \cref{eq::option:orderings} and the event ordering~\(\posetleq\):
\begin{enumerate}[resume*=events,widest=6]
\item\label{axiom:event:ordering}\(\group{\forall \opt\in\opts}\group{\eventopt[1]\ast\opt>0\then\eventopt[2]\ast\opt\neq0}\then\eventopt[1]\posetleq\eventopt[2]\), for all~\(\eventopt[1],\eventopt[2]\in\eventopts\),\nameit{Event Ordering}
\end{enumerate}
where we let \(\opt>\altopt\ifandonlyif\group{\opt\geq\altopt\text{ and }\opt\neq\altopt}\) for all~\(\opt,\altopt\in\opts\).

The event ordering~\(\posetleq\) can, when invoking \labelcref{axiom:event:ordering}, be related directly to the kernels of the relevant projections.

\begin{proposition}\label{prop::ordering:equivalent:statements}
For all events~\(\eventopt[1],\eventopt[2]\in\eventopts\), the following statements are equivalent:
\begin{enumerate}[label={\upshape(\roman*)},leftmargin=*,widest=iii]
\item\label{prop::event:ordering} \(\eventopt[1]\posetleq\eventopt[2]\);
\item\label{prop::event:kernels} \(\eventindifset[2]\subseteq\eventindifset[1]\);
\item\label{prop::event:nonpositivity} \(\group{\forall\opt\in\opts}\group{\calledoff[1]{\opt}>0\then\calledoff[2]{\opt}\ne0}\).
\end{enumerate}
\end{proposition}
\noindent An interesting consequence of \cref{prop::ordering:equivalent:statements} is that events are in a one-to-one correspondence with the kernels of the calling-off operation associated with them: of course, for any event~\(\eventopt\), the kernel of the corresponding projection~\(\eventopt\ast\bolleke\) is completely determined; but also conversely, for any linear subspace of~\(\opts\) there's at most one event that has it as the kernel of the associated projection.
This follows directly from the equivalence of statements \labelcref{prop::event:ordering} and \labelcref{prop::event:kernels} in \cref{prop::ordering:equivalent:statements} and the fact that \(\posetleq\) is a partial order (and therefore anti-symmetric).
As a result, we find that
\begin{equation}\label{eq::kernels:identify:events}
\eventindifset[1]=\eventindifset[2]
\ifandonlyif\group{\forall\opt\in\opts}\group{\eventopt[1]\ast\opt=0\ifandonlyif\eventopt[2]\ast\opt=0}\ifandonlyif\eventopt[1]=\eventopt[2],
\text{ for all~\(\eventopt[1],\eventopt[2]\in\eventopts\)}.
\end{equation}

There's also a requirement on the called-off versions of the unit option:
\begin{enumerate}[resume*=events,widest=7]
\item\label{axiom:event:unit:option} \(\optunit\weakgeq\eventopt\ast\optunit\) for all~\(\eventopt\in\eventopts\);\nameit{Unit Option}
\end{enumerate}
and finally, we add an axiom that allows us to replace the sum of the kernels of subsequent calling-off operations by the kernel of a single one:
\begin{enumerate}[resume*=events,widest=8]
\item\label{axiom:event:kernel:sum} for all~\(\eventopt[1],\eventopt[2]\in\eventopts\), there's some~\(\eventopt[1]\sqcap\eventopt[2]\in\eventopts\) such that \(\eventindifset[1]+\eventindifset[2]=\indifset[{\eventopt[1]\sqcap\eventopt[2]}]\).\nameit{Kernel Sum}
\end{enumerate}

\begin{gamblesexamplecont}
In the classical probability context, the \emph{events} are the subsets~\(\event\) of the possibility space~\(\states\), and they can (and will) be identified with special \define{indicator gambles}~\(\indevent\) that assume the value~\(1\) on~\(\event\) and \(0\) elsewhere; so \(\eventopts\instantiateas\set{\indevent\given\event\subseteq\states}\).

The \emph{unit event} corresponds to the constant gamble~\(1=\indof{\states}\) and the \emph{null event} to the constant gamble~\(0=\indof{\emptyset}\).

For the \emph{calling-off operation}, we have that \(\indevent\ast\gbl\instantiateas\indevent\gbl\), which is a special linear projection on the gamble space~\(\gambles\).

It's a trivial exercise to show that the assumptions \labelcref{axiom:event:linear,axiom:event:idempotent,axiom:event:monotone,axiom:event:unit:event,axiom:event:null:event,axiom:event:ordering,axiom:event:unit:option,axiom:event:kernel:sum} are satisfied, with the event~\(\indof{\event[1]\cap\event[2]}\) playing the part of~\(\indof{\event[1]}\sqcap\indof{\event[2]}\), so set intersection takes the role of the \(\sqcap\)-operation.

Finally, the event ordering~\(\posetleq\) corresponds to set inclusion: \(\indevent[1]\posetleq\indevent[2]\ifandonlyif\event[1]\subseteq\event[2]\).
\end{gamblesexamplecont}

\begin{quantumexamplecont}
In a quantum mechanical context, the \emph{events} correspond to non-empty finite sets~\(\set{\subspace[1],\dots,\subspace[n]}\) of mutually orthogonal subspaces~\(\subspace[k]\) of the Hilbert space~\(\hilbertspace\).\footnote{The reason why we don't restrict ourselves to single subspaces here, can be traced back to another paper by two of us \cite{devos2025:isipta}, where we argue for the need of such `unions of mutually orthogonal subspaces', as they also correspond to physical processes.}

With any subspace~\(\subspace\), we can associate a special measurement~\(\projector=\projector\projector\), which is the (linear and orthogonal) projection operator onto the subspace~\(\subspace\), with eigenvalue~\(1\) associated with the eigenspace~\(\subspace\) and eigenvalue~\(0\) associated with its orthogonal complement~\(\subspace^\perp\).

For the \emph{calling-off operator} associated with an event~\(\set{\subspace[1],\dots,\subspace[n]}\), where \(\subspace[1],\dots,\subspace[m]\) are mutually orthogonal subspaces, we have that
\begin{equation*}
\set{\subspace[1],\dots,\subspace[n]}\ast\measurement{A}\instantiateas\sum_{k=1}^n\projector[k]\measurement{A}\projector[k]
\text{ for all~\(\measurement{A}\in\Hop\)},
\end{equation*}
which is a linear projection operator on the option space~\(\Hop\).
The \emph{unit event} corresponds to the singleton~\(\set{\hilbertspace}\), where \(\projectoron{\hilbertspace}=\identity\) is the identity measurement; and the \emph{null event} corresponds to the singleton~\(\set{\set{0}}\), where \(\projectoron{\set{0}}=\zero\) is the zero measurement.
The event ordering then corresponds to the inclusion ordering of the subsets, that is,  \(\set{\subspace[1],\dots,\subspace[n]}\posetleq\set{\subspace[1]^{\prime},\dots,\subspace[m]^{\prime}}\) if and only if \(\bigcup_{k=1}^n\subspace[k]\subseteq\bigcup_{\ell=1}^m\subspace[\ell]^{\prime}\).\footnote{There's a one-to-one correspondence between a an event~\(\set{\subspace[1],\dots,\subspace[n]}\) and the union~\(\bigcup_{k=1}^n\subspace[k]\) of the orthogonal subspaces~\(\subspace[k]\) in the set~\(\set{\subspace[1],\dots,\subspace[n]}\).}
Finally, the \(\sqcap\)-operation corresponds to the \(\cap\)-operation on the subsets: \(\set{\subspace[1]^{\prime\prime},\dots,\subspace[p]^{\prime\prime}}=\set{\subspace[1],\dots,\subspace[n]}\sqcap\set{\subspace[1]^{\prime},\dots,\subspace[m]^{\prime}}\)
if and only if \(\bigcup_{r=1}^p\subspace[r]^{\prime\prime}=\big(\bigcup_{k=1}^n\subspace[k]\big)\cap\big(\bigcup_{\ell=1}^m\subspace[\ell]^{\prime}\big)\).

It's a straightforward exercise to show that \labelcref{axiom:event:linear,axiom:event:idempotent,axiom:event:monotone,axiom:event:unit:event,axiom:event:null:event,axiom:event:unit:option} hold.
Explicit proofs for \labelcref{axiom:event:ordering,axiom:event:kernel:sum} are given in the Appendix.
\end{quantumexamplecont}

\section{Conditioning on an event}\label{sec::conditioning}
We'll assume that the effect of the occurrence of an event~\(\eventopt\) is that You can no longer distinguish between any two options whose called-off options under~\(\eventopt\) coincide: they've now become indifferent to You.
This leads to a new linear space of indifferent options, given by the kernel
\begin{equation*}
\eventindifset
=\set{\opt\in\opts\given\calledoff{\opt}=0}
\end{equation*}
of the calling-off operation~\(\calledoff{\bolleke}\).
It therefore seems perfectly reasonable to represent the new knowledge that the event~\(\eventopt\) has occurred by the indifference assessment
\begin{equation}\label{eq::model:for:event}
\eventM
\coloneqq\Adelim{\eventindifset}{\emptyset}\in\adfMsabove{\zeroM}.
\end{equation}

\subsection{Regular events}
For this new indifference assessment~\(\eventM\) to be reasonable at all, it must be consistent with the background~\(\bgM\), which amounts to requiring that
\begin{equation}\label{eq::regularity}
0\not\in\bgM_\des+\eventindifset
\text{, or equivalently, }
\bgM_\des\cap\eventindifset=\emptyset.
\end{equation}
We consider an event~\(\eventopt\) to be \define{consistent} with the background~\(\bgM\) when this condition is satisfied, and we'll then also call the event~\(\eventopt\) \define{\(\bgM\)-regular} or simply \define{regular}; regular events are the only ones we can reasonably envision conditioning on: if the occurrence of an event contradicts those accept/reject statements that You find `obvious', You can never end up with a meaningful AD-model after conditioning.
We'll therefore require henceforth in a conditioning setting that the events under consideration are \(\bgM\)-regular.
We'll denote the set of all \(\bgM\)-regular events by~\(\regulareventopts[\bgM]\).

\begin{proposition}\label{prop::regular:events}
An event~\(\eventopt\) is \(\bgM\)-regular if and only if \(\opt\stronggt0\then\calledoff{\opt}>0\) for all~\(\opt\in\opts\).
\end{proposition}

\begin{gamblesexamplecont}
In the context of classical probability, an event~\(\indevent\) is regular with respect to the background~\(\Adelim{\gambles_{\weakgeq0}}{-\gambles_{\stronggt0}}\) if and only if
\begin{equation*}
\gbl\stronggt0\then\indevent\gbl\neq0,
\text{ for all \(\gbl\in\gambles\)}.
\end{equation*}
It's clear that the empty event~\(\emptyset\) isn't regular, because \(\indof{\emptyset}f=0\) for all~\(\gbl\in\gambles\).
But conversely, all non-empty events~\(\event\neq\emptyset\) are regular: there's then always at least one~\(\state\in\event\), and then \((\indevent\gbl)\group\state=0\ifandonlyif\gbl\group\state=0\).
We conclude that the regular events~\(\indevent\) correspond to the non-empty subsets~\(\event\neq\emptyset\) of~\(\states\).
\end{gamblesexamplecont}

\begin{quantumexamplecont}
In the context of quantum probabilistic inference, an event~\(\set{\subspace[1],\dots,\subspace[n]}\) is regular with respect to the background~\(\Adelim{\Hop_{\weakgeq\zero}}{-\Hop_{\stronggt\zero}}\) if and only if
\begin{equation*}
\measurement{A}\stronggt\zero\then\sum_{k=1}^n\projector[k]\measurement{A}\projector[k]\neq\zero,
\text{ for all \(\measurement{A}\in\Hop\)}.
\end{equation*}
It's clear that the null event~\(\set{\set{0}}\) isn't regular, since \(\projectoron{\set{0}}\measurement{A}\projectoron{\set{0}}=\zero\measurement{A}\zero=\zero\) for all~\(\measurement{A}\in\Hop\).
But, conversely, any other event~\(\set{\subspace[1],\dots,\subspace[n]}\neq\set{\set{0}}\) is regular.
Indeed, consider any positive definite measurement~\(\measurement{A}\stronggt0\), then \(\bra{\psi}\measurement{A}\ket{\psi}\geq0\) for all~\(\ket{\psi}\in\hilbertspace\) and \(\bra{\psi}\measurement{A}\ket{\psi}>0\) for all~\(\ket{\psi}\in\hilbertspace\setminus\set{0}\).
But then, with~\(\measurement{B}\coloneqq\sum_{k=1}^n\projector[k]\measurement{A}\projector[k]\),
\begin{equation*}
\bra{\psi}\measurement{B}\ket{\psi}
=\sum_{k=1}^n\bra{\psi}\projector[k]\measurement{A}\projector[k]\ket{\psi}
=\sum_{k=1}^n\group{\projector[k]\ket{\psi}}^\dagger\measurement{A}\group{\projector[k]\ket{\psi}}\geq0
\text{ for all~\(\ket{\psi}\in\hilbertspace\)},
\end{equation*}
and since it follows from the assumption that there's at least one \(\subspace[k]\neq\set{0}\), we have for all~\(\ket{\psi}\in\subspace[k]\setminus\set{0}\) that \(\projector[k]\ket{\psi}=\ket{\psi}\neq0\), and therefore also that \(\group{\projector[k]\ket{\psi}}^\dagger\measurement{A}\group{\projector[k]\ket{\psi}}=\bra{\psi}\measurement{A}\ket{\psi}>0\), implying that \(\bra{\psi}\measurement{B}\ket{\psi}>0\), so indeed \(\measurement{B}\neq\zero\).

We conclude that all non-null events~\(\set{\subspace[1],\dots,\subspace[n]}\neq\set{\set{0}}\) are regular, so the null event~\(\set{\set{0}}\) is the only non-regular event.
\end{quantumexamplecont}

\subsection{Conditioning AD-models}\label{sec::conditioning:admodels}
If You started out with some AD-model~\(\M\in\adfMsabove{\bgM}\), which options should You still find acceptable after the occurrence of the event~\(\eventopt\)?
We're going to assume, in the spirit of discussions by \citeauthor{finetti1937} \cite{finetti1937,finetti19745}, \citeauthor{williams2007} \cite{williams2007}, \citeauthor{walley1991} \cite{walley1991,walley2000}, and \citeauthor{troffaes2013:lp} \cite{troffaes2013:lp} that You now find acceptable at least all options~\(\opt\) whose called-off versions~\(\eventopt\ast\opt\) were already acceptable before the occurrence of the event, which leads to a new set of acceptable options~\(\M_\acc\altcondon\eventopt\),\footnote{The new set of indifferent options~\(\M_\equiv\altcondon\eventopt=\group{\M_\acc\cap-\M_\acc}\altcondon\eventopt=\group{\M_\acc\altcondon\eventopt}\cap-\group{\M_\acc\altcondon\eventopt}\) can be different from~\(\eventindifset\); we merely require that it should  include~\(\eventindifset\) in other to at least reflect the new indifferences created by the observation that the event~\(\eventopt\) has occurred.} where we let
\begin{equation}\label{eq::updated}
\someopts\altcondon\eventopt
\coloneqq\set{\opt\in\opts\given\calledoff{\opt}\in\someopts},
\text{ for any~\(\someopts\subseteq\opts\).}
\end{equation}
A similar argument for desirability leads to the new set of desirable options~\(\M_\des\altcondon\eventopt\); see also \citeauthor{walley2000}'s seminal paper \cite{walley2000}.
This brings us to the following AD-assessment for the updated sets of acceptable and desirable options:
\begin{equation}\label{eq::naive:updated:model}
\Adelim{\M_\acc\altcondon\eventopt}{-\M_\des\altcondon\eventopt}.
\end{equation}
Even though this assessment satisfies \labelcref{axiom:AD:zero:not:desirable,axiom:AD:deductive:closedness,axiom:AD:no:limbo}, this needn't be the case for \labelcref{axiom:AD:background}.
Indeed, while \labelcref{axiom:event:monotone} ensures that \(\bgM_\acc\subseteq\M_\acc\altcondon\eventopt\), it needn't be that \(\bgM_\des\subseteq\M_\des\altcondon\eventopt\), not even if the event~\(\eventopt\) is regular, so \labelcref{axiom:AD:background} may fail to hold for the desirability side of this assessment.\footnote{This is illustrated by the proof of \cref{prop::regular:events} in the Appendix, where we show that regularity of the event guarantees that the called-off option is \emph{strictly} greater than~\(0\) for the vector preorder~\(\weakgeq\), but there's no proof that it should hold for the strict vector ordering~\(\stronggt\).}

We can attempt to resolve this issue by moving to the so-called \(\bgM\)-natural extension of the expression~\labelcref{eq::naive:updated:model}: extend the assessment~\(\Adelim{\M_\acc\altcondon\eventopt}{-\M_\des\altcondon\eventopt}\) by finding the smallest (least resolved) dominating AD-model that also respects the background~\(\bgM\)---if such a least resolved dominating AD-model exists, or in other words, if the assessment~\(\Adelim{\M_\acc\altcondon\eventopt}{-\M_\des\altcondon\eventopt}\) is \(\bgM\)-consistent.

So, let's first investigate this consistency issue.

\begin{proposition}\label{prop::consistency:of:updated:model}
Consider any~\(\M\in\adfMsabove{\bgM}\) and any~\(\bgM\)-regular event~\(\eventopt\in\regulareventopts\), then the AD-assessment~\(\Adelim{\M_\acc\altcondon\eventopt}{-\M_\des\altcondon\eventopt}\) is \(\bgM\)-consistent if and only if \(0\notin\M_\acc\altcondon\eventopt+\bgM_\des\).
\end{proposition}
\noindent In that case, we'll call the AD-model~\(\M\) \define{conditionable} on the regular event~\(\eventopt\).

So the conditionable models~\(\M\) are the ones for which we can take the natural extension of the assessment~\(\Adelim{\M_\acc\altcondon\eventopt}{-\M_\des\altcondon\eventopt}\), because doing so will lead to an AD-model in~\(\adfMsabove{\bgM}\), which we'll then denote by~\(\M\condon\eventopt\).
Let's now find out what this natural extension~\(\M\condon\eventopt\) looks like:
\begin{multline*}
\M\condon\eventopt
\coloneqq\closureof{\Msabove{\bgM}}{\Adelim{\M_\acc\altcondon\eventopt}{-\M_\des\altcondon\eventopt}}\\
\begin{aligned}
&=\closureof{\Ms}{\Adelim{\M_\acc\altcondon\eventopt}{-\M_\des\altcondon\eventopt}\cup\Adelim{\bgM_\acc}{-\bgM_\des}}\\
&=\Adelim{\posiof{\M_\acc\altcondon\eventopt\cup\bgM_\acc}}{-\group{\shullof{\M_\des\altcondon\eventopt\cup\bgM_\des}\cup\group{\shullof{\M_\des\altcondon\eventopt\cup\bgM_\des}+\posiof{\M_\acc\altcondon\eventopt\cup\bgM_\acc}}}},
\end{aligned}
\end{multline*}
where the last equality follows by applying the expression~\labelcref{eq::model:closure} for the AD-model closure.

\begin{proposition}\label{prop::conditional:model}
Consider any AD-model~\(\M\in\adfMsabove{\bgM}\) and any~\(\bgM\)-regular event~\(\eventopt\in\regulareventopts\).
If the AD-model~\(\M\) is conditionable on~\(\eventopt\), then the \(\bgM\)-natural extension~\(\M\condon\eventopt\) of the AD-assessment~\(\Adelim{\M_\acc\altcondon\eventopt}{-\M_\des\altcondon\eventopt}\) is the AD-model in~\(\adfMsabove{\bgM}\) given by
\begin{equation}\label{eq::revised:statement:model}
\M\condon\eventopt
=\Adelim{\M_\acc\altcondon\eventopt}{-\group{\M_\des\altcondon\eventopt\cup\group{\bgM_\des+\M_\acc\altcondon\eventopt}}}.
\end{equation}
\end{proposition}
\noindent We'll then call \(\M\condon\eventopt\) the \define{conditional(isation)} of the (conditionable) AD-model~\(\M\) on the (regular) event~\(\eventopt\); we see that conditioning only makes sense for (i) regular events and for (ii) AD-models that are conditionable on them.

It's now obvious that conditional models on an event are what we want them to be: AD-models that respect the given background as well as the extra indifferences generated by the information that the event has occurred.

\begin{corollary}\label{cor:AD:revision:operator:model}
If an AD-model~\(\M\in\adfMsabove{\bgM}\) is conditionable on a \(\bgM\)-regular event~\(\eventopt\in\eventopts\), then its conditional model~\(\M\condon\eventopt\) is an AD-model for which both  \(\bgM\subseteq\M\condon\eventopt\) and \(\eventM\subseteq\M\condon\eventopt\).
\end{corollary}

\noindent In general, however, we won't have that \(\M\subseteq\M\condon\eventopt\); this implies that the information that an event has occurred doesn't necessarily lead to an \emph{increase of} knowledge, but rather to a \emph{revision} of it.
It's for this reason that we'll turn to investigating the connection between conditioning and \emph{belief revision} in \cref{sec::conditioning:as:belief:change}.

\subsection{The connection with Bayes' Rule and Lüders' rule}\label{sec:bayes:luders}
Now that we've now come up with a rule for conditioning AD-models, the question that arises at once is: how does our conditioning rule relate to the common rules for conditioning in the literature?
Rules that are similar\footnote{That the similarity isn't always perfect, is due to the fact that our present conditioning rule must take into account the interaction between acceptable and desirable gambles, which recedes from view when sets of acceptable and/or desirable gambles are considered on their own.} to \labelcref{eq::updated} for conditioning sets of desirable and sets of acceptable options, considered on their own and not jointly, exist in the literature, both in the case of gambles \cite{walley2000,troffaes2013:lp,williams1975} and in a quantum mechanical context \cite{devos2025:isipta,benavoli2016:quantum_2016}.

But the conditioning rule~\labelcref{eq::updated} can also be shown to have a direct connection with Bayes' Rule in classical probability theory and Lüders' rule in quantum probability.
To provide a very brief sketch of how this comes about, we consider the special case where the AD-model~\(\M\) is conditionable on some regular event~\(\eventopt\in\eventopts\), and gives rise to a fair price functional \(\prev_{\M}=\lprev_{\M}=\uprev_{\M}\).
The conditional~\(\M\condon\eventopt\) then gives rise to a \define{conditional lower price functional} \(\lprev_{\M}\group{\bolleke\condon\eventopt}\coloneqq\lprev_{\M\condon\eventopt}\), given by
\begin{align*}
\lprev_{\M}\group{\opt\condon\eventopt}
\coloneqq\lprev_{\M\condon\eventopt}\group\opt
&=\sup\set{\alpha\in\reals\given\opt-\alpha\optunit\in\M_\acc\altcondon\eventopt}
&&\reason{Eqs.~\labelcref{eq::definition:lprev:uprev,eq::revised:statement:model}}\\
&=\sup\set{\alpha\in\reals\given\eventopt\ast\group{\opt-\alpha\optunit}\in\M_\acc}
&&\reason{Eq.~\labelcref{eq::updated}}
\end{align*}
for all~\(\opt\in\opts\).
It's a reasonably straightforward exercise\footnote{The proof of this statement is a direct generalisation to options and abstract events of the proof given by \citeauthor{troffaes2013:lp} in \cite[Lemma 13.13]{troffaes2013:lp}.} to show that the real number~\(\lprev_{\M}\group{\opt\condon\eventopt}\) is a solution of the equation~\(\prev_{M}\group{\eventopt\ast\group{\opt-\alpha\optunit}}=0\) in~\(\alpha\), so
\begin{align*}
0
&=\prev_{M}\group{\eventopt\ast\group{\opt-\lprev_{\M}\group{\opt\condon\eventopt}\optunit}}\\
&=\prev_{M}\group{\eventopt\ast\opt}-\lprev_{\M}\group{\opt\condon\eventopt}\prev_{M}\group{\eventopt\ast\optunit},
&&\reason{\(\prev_{\M}\) linear functional}
\end{align*}
and therefore
\begin{equation}\label{eq::conditional:prev:update}
\lprev_{\M}\group{\opt\condon\eventopt}\prev_{M}\group{\eventopt\ast\optunit}=\prev_{M}\group{\eventopt\ast\opt}
\text{ for all~\(\opt\in\opts\)}.
\end{equation}

\begin{gamblesexamplecont}
When we apply \cref{eq::conditional:prev:update} in the classical probability context, we can see what it yields for a gamble~\(\gbl\) and a regular (non-empty) event~\(\event\), we get
\begin{equation*}
\prev_{\M}\group{\indevent\gbl}
=\lprev_{\M}\group{\gbl\condon\event}\prev_{\M}\group{\indevent},
\end{equation*}
which leads us directly to Bayes' Rule.
\end{gamblesexamplecont}

\begin{quantumexamplecont}
When we apply \cref{eq::conditional:prev:update} in the quantum probability context, we can see what it yields for a measurement~\(\measurement{A}\) and a regular event~\(\set{\subspace}\); we get
\begin{equation*}
\prev_{\M}\group{\projector\measurement{A}\projector}
=\lprev_{\M}\group{\measurement{A}\condon{\set{\subspace}}}\prev_{\M}\group{\projector\identity\projector},
\end{equation*}
We've seen before that a coherent prevision~\(\prev_{\M}\) corresponds to a density operator~\(\density\) in the sense that \(\prev_{\M}\group{\measurement{A}}=\trace{\density\measurement{A}}\) for all measurements~\(\measurement{A}\).
This implies that \(\prev_{\M}\group{\projector\identity\projector}=\trace{\density\projector\projector}=\trace{\projector\density\projector}\) and that \(\prev_{\M}\group{\projector\measurement{A}\projector}=\trace{\density\projector\measurement{A}\projector}=\trace{\projector\density\projector\measurement{A}}\), using the cyclic property of the trace.
So as soon as \(\trace{\projector\density\projector}>0\), we find that
\begin{equation*}
\lprev_{\M}\group{\measurement{A}\condon{\set{\subspace}}}
=\frac{\trace{\projector\density\projector\measurement{A}}}{\trace{\projector\density\projector}}
=\trace{\density'\measurement{A}}
\text{ with }
\density'=\frac{\projector\density\projector}{\trace{\projector\density\projector}},
\end{equation*}
which is exactly the expression given by Lüders' rule for conditioning on a subspace \cite{luders1950}.
A similar argument leads to the formula
\begin{equation*}
\lprev_{\M}\group{\measurement{A}\condon{\set{\subspace[1],\dots,\subspace[n]}}}
=\trace{\density'\measurement{A}}
\text{ with }
\density'
=\frac{\sum_{k=1}^n\projector[k]\density\projector[k]}{\sum_{k=1}^n\trace{\projector[k]\density\projector[k]}},
\end{equation*}
for conditioning on the more involved event~\(\set{\subspace[1],\dots,\subspace[n]}\).
\end{quantumexamplecont}

\section{Belief states and belief change}\label{sec::belief:change}
A discussion of what `belief states' are, and rules for changing them, have been explored by \citeauthor{alchourron1985} (AGM, for short) in \cite{alchourron1985}, \citeauthor{gardenfors1988} in \cite{gardenfors1988}, and \citeauthor{rott2001} in \cite{rott2001}.
However, these authors focused mainly on belief states that are sets of propositions (or probability measures), so to apply their ideas to the more abstract setting where the belief states are AD-models, we need to adapt and generalise them.
The basic tools for doing so were developed in an early paper \cite{cooman2003a} by \citeauthor{cooman2003a}, and we'll use the present section to briefly summarise the ideas there that are directly relevant to our discussion.

We'll consider abstract objects, which we'll call \define{belief models} or \define{belief states}, and collect them in a set~\(\belmods\).
We assume that the belief states~\(\belmod\) in \(\belmods\) are partially ordered by some binary relation~\(\modleq\), and that they constitute a \emph{complete lattice} \(\structure{\belmods,\modleq}\).
We denote its top as~\(\belmodtop\), its bottom as~\(\belmodbottom\), its meet as~\(\meet\) and its join as~\(\join\).

We're particularly interested in a special non-empty subset~\(\cohmods\subseteq\belmods\) of belief states, whose elements are called \define{coherent} and are considered to be `more perfect' than the other (\define{incoherent}) belief states in~\(\belmods\setminus\cohmods\).
The inherited partial ordering~\(\modleq\) on~\(\cohmods\) is interpreted as `is more conservative than'.

Crucially, we'll assume that \(\cohmods\) is closed under arbitrary non-empty infima, so for every non-empty subset~\(C\subseteq\cohmods\) we assume that \(\inf C\in\cohmods\).
Additionally, we'll assume that the complete meet-semilattice \(\structure{\cohmods,\modleq}\) has no top---so definitely \(\belmodtop\notin\cohmods\)---implying that there's a smallest (most conservative) coherent belief state~\(\cohmodbottom\coloneqq\inf\cohmods\) but no largest (least conservative) one.
We'll let the incoherent top belief state~\(\belmodtop\) represent contradiction.
Of course, the set~\(\closedmods\coloneqq\cohmods\cup\set{\belmodtop}\) provided with the partial ordering~\(\modleq\) is a complete lattice.
The corresponding triple \(\structure{\belmods,\cohmods,\modleq}\) is then called a \define{belief structure}.

That \(\cohmods\) is closed under arbitrary non-empty infima, leads us to define a closure operator as follows:
\begin{equation*}
\cohmodclosure
\colon\belmods\to\closedmods
\colon\belmod\mapsto\cohmodclosureof{\belmod}\coloneqq\inf\set{\cohmod\in\cohmods\given\belmod\modleq\cohmod},
\end{equation*}
so \(\cohmodclosureof{\belmod}\) is the most conservative coherent belief state that's at least as committal as \(\belmod\)---if any.
The closure operator implements \emph{conservative inference} in the belief structure \(\structure{\belmods,\cohmods,\modleq}\).
Using this closure~\(\cohmodclosure\), we can now define notions of \define{consistency} and \define{closedness}: a belief state~\(\belmod\in\belmods\) is \define{consistent} if \(\cohmodclosureof{\belmod}\neq\belmodtop\), or equivalently, if \(\cohmodclosureof{\belmod}\in\cohmods\), and \define{closed} if \(\cohmodclosureof{\belmod}=\belmod\).
Two belief states \(\belmod[1],\belmod[2]\in\belmods\) are called \define{consistent} if their join \(\belmod[1]\join\belmod[2]\) is.
We see that \(\closedmods\) is the set of all closed belief states, that the coherent belief states in \(\cohmods\) are the ones that are both consistent and closed, and that \(\belmodtop\) represents inconsistency for the closed belief states.

So, how do the AD-models we discussed in the earlier sections fit into this abstract picture?
The AD-models \(\M\in\adfMsabove{\bgM}\) that respect a given AD-background~\(\bgM\in\adfMsabove{\zeroM}\), provided with the `is at most as resolved as' ordering~\(\subseteq\), constitute a special instance of these belief structures, with the correspondences identified in \cref{tab:correspondence}.

\begin{table}[ht]
\centering
\begin{tabular}{rcl}
belief state
&\(\instantiateas\)
&AD-assessment\\
\(\belmods\)
&\(\instantiateas\)
&\(\adfAs\)\\
\(\modleq\)
&\(\instantiateas\)
&\(\subseteq\)\\
\(\belmodbottom\)
&\(\instantiateas\)
&\(\Adelim{\emptyset}{\emptyset}\)\\
\(\belmodtop\)
&\(\instantiateas\)
&\(\Adelim{\opts}{\opts}\)\\
\(\inf\)
&\(\instantiateas\)
&\(\Intersection\)\\
\(\cohmods\)
&\(\instantiateas\)
&\(\adfMsabove{\bgM}\)\\
\(\closedmods\)
&\(\instantiateas\)
&\(\adfMsabove{\bgM}\cup\set{\Adelim{\opts}{\opts}}\)\\
\(\cohmodbottom\)
&\(\instantiateas\)
&background AD-model~\(\bgM\)\\
coherent belief state
&\(\instantiateas\)
& AD-model that respects \(\bgM\)\\
\(\cohmodclosureof{}\)
&\(\instantiateas\)
&\(\Mclsof{\bgM\union\bolleke}\)\\
consistent
&\(\instantiateas\)
& \(\bgM\)-consistent\\
coherent
&\(\instantiateas\)
& AD-model that respects \(\bgM\)\\
closure of a \dots
&\(\instantiateas\)
&\(\bgM\)-natural extension of a \dots\\
consistent belief state&&\(\bgM\)-consistent AD-assessment\\[1ex]
\end{tabular}
\caption{Correspondences that show how in exactly what sense the AD-models constitute a belief structure}
\label{tab:correspondence}
\end{table}
\noindent In other words, \(\structure{\adfAs,\adfMsabove{\bgM},\subseteq}\) is a belief structure in the sense described above.

Generally speaking, when You're in a coherent belief state~\(\cohmod\in\cohmods\), new information the form of some belief state~\(\belmod\in\belmods\)---not necessarily coherent---can cause You to change Your belief state.
\citeauthor{alchourron1985} \cite{alchourron1985} considered three important types of belief change: belief expansion, belief revision and belief contraction.
They proposed a number of axioms for these operations, and studied and discussed their resulting properties.

De Cooman \cite{cooman2003a} has argued that the axioms for belief expansion and belief revision can be elegantly translated to the setting of belief structures---the counterpart of belief contraction is more problematic.
We'll not discuss his abstract axioms here, but propose to postpone listing their more concrete forms until the next section, where we consider the concrete forms of belief expansion and belief revision for Accept-Desirability models caused by the occurrence of events.

\section{Conditioning as belief change}\label{sec::conditioning:as:belief:change}
Let's now investigate how conditioning on events can be seen as a form of belief change in the Accept-Desirability framework.
As explained above in \cref{sec::conditioning}, You start out with an AD-model~\(\M=\Adelim{\M_\acc}{-\M_\des}\in\adfMsabove{\bgM}\) that respects the background AD-model~\(\bgM=\Adelim{\bgM_\acc}{-\bgM_\des}\in\adfMsabove{\zeroM}\) and You then gain new information in the form of the occurrence of an event~\(\eventopt\in\eventopts\).
We've argued that this event~\(\eventopt\) corresponds to an AD-assessment~\(\eventM=\Adelim{\eventindifset}{\emptyset}\in\adfMsabove{\zeroM}\).
The question we're going to answer in this section then is: how should Your AD-model~\(\M\) change, given this new information in the shape of the assessment~\(\eventM\)?\footnote{We've  already mentioned above that, due to \cref{eq::kernels:identify:events}, the sets of indifferent options~\(\eventindifset\)---and therefore also the AD-assessment \(\eventM=\Adelim{\eventindifset}{\emptyset}\)---are in a one-to-one correspondence with the events~\(\eventopt\).}

\subsection{Belief expansion}\label{sec::belief:expansion}
One way to combine Your AD-model~\(\M\) with new information in the form of an AD-assessment~\(\A\) goes via a \define{belief expansion} operator\footnote{The `\(\vert\)' symbol merely acts as a way to separate the first from the second argument here.}
\begin{equation*}
\expand
\colon\adfMsabove{\bgM}\times\adfAs\to\adfAs
\colon(\M,\A)\mapsto\expandof{\M\given\A},
\end{equation*}
which produces a new AD-assessment~\(\expandof{\M\given\A}\), intended to combine Your original beliefs and the new information on an equal footing.
De Cooman \cite{cooman2003a} has argued that the AGM axioms for belief expansion \cite{alchourron1985,gardenfors1988,rott2001} can be translated directly to the abstract setting of belief structures.
His results, together with the correspondences in the previous section, lead us to conclude that in the present context of the belief structure~\(\structure{\adfAs,\adfMsabove{\bgM},\subseteq}\), the resulting expansion operator \(\expand\) is then uniquely determined by the resulting model closure, in the sense that
\begin{equation*}
\expandof{\M\given\A}
=\Mclsbgof{\bgM}{\M\union\A}
=\Mclsof{\M\union\A}
\text{ for all~\(\M\in\adfMsabove{\bgM}\) and~\(\A\in\adfAs\)}.
\end{equation*}
In our conditioning setting, this leads to
\begin{multline*}
\expandof{\M\given\eventopt}
\coloneqq\expandof{\M\given\eventM}
=\Mclsbgof{\bgM}{\M\union\eventM}
=\Mclsof{\M\union\eventM}\\
\text{ for all~\(\M\in\adfMsabove{\bgM}\) and~\(\eventopt\in\eventopts\)},
\end{multline*}
where \(\M\union\eventM=\Adelim{\M_\acc\cup\eventindifset}{-\M_\des}\).
Since, clearly, \(\posiof{\M_\acc\cup\eventindifset}=\M_\acc+\eventindifset\), we see that
\begin{align*}
\posiof{\M_\acc\cup\eventindifset}\cap-\M_\des=\emptyset
&\ifandonlyif0\notin\posiof{\M_\acc\cup\eventindifset}+\M_\des\\
&\ifandonlyif0\notin\M_\acc+\eventindifset+\M_\des\\
&\ifandonlyif0\notin\M_\des+\eventindifset,
&&\reason{\(\M_\acc+\M_\des=\M_\des\) by \labelcref{axiom:AD:no:limbo}}
\end{align*}
and therefore \(\M\union\eventM\) is \(\bgM\)-consistent if and only if \(0\notin\M_\des+\eventindifset\).
In that case, we infer from \cref{eq::model:closure} that
\begin{align}
\Mclsof{\M\union\eventM}
&=\Adelim{\posiof{\M_\acc\cup\eventindifset}}{-\shullof{\M_\des}\cup-\group{\shullof{\M_\des}+\posiof{\M_\acc\cup\eventindifset}}}
\notag\\
&=\Adelim{\M_\acc+\eventindifset}{-\group{\M_\des\cup\group{\M_\des+\M_\acc+\eventindifset}}}
\notag\\
&=\Adelim{\M_\acc+\eventindifset}{-\group{\M_\des\cup\group{\M_\des+\eventindifset}}}
\notag\\
&=\Adelim{\M_\acc+\eventindifset}{-\group{\M_\des+\eventindifset}}.
\label{eq::model:closure:for:expansion}
\end{align}
This leads to the following expression for our expansion operator:
\begin{equation}\label{eq::AD:expansion:operator}
\expandof{\M\given\eventopt}\\
=\begin{cases}
\Adelim{\M_\acc+\eventindifset}{-\group{\M_\des+\eventindifset}}
&\text{if \(0\notin\M_\des+\eventindifset\)}\\
\Adelim{\opts}{\opts}
&\text{otherwise}.
\end{cases}
\end{equation}
Observe that for any event~\(\eventopt\) that isn't \(\bgM\)-regular, it holds by definition that \(0\in\bgM_\des+\eventindifset\), and therefore also that \(0\in\M_\des+\eventindifset\) for all~\(\M\in\adfMsabove{\bgM}\), so the expansion~\(\expandof{\M\given\eventopt}\) \emph{always} yields the contradictory closed belief state~\(\Adelim{\opts}{\opts}\).

\subsection{Belief revision}
Another way of changing Your beliefs when new information comes in, is by a \define{revision operator}
\begin{equation*}
\revise
\colon\adfMsabove{\bgM}\times\adfAs\to\adfAs
\colon(\M,\A)\mapsto\reviseof{\M\given\A},
\end{equation*}
where the new AD-assessment~\(\reviseof{\M\given\A}\) is intended to combine Your original beliefs~\(\M\) and the new information~\(\eventM\) in such a way that, when the new information is inconsistent with Your original beliefs, some of Your original beliefs are replaced by the new information, to make the combination~\(\reviseof{\M\given\A}\) consistent, whenever this is at all possible.

\citeauthor{cooman2003a} \cite{cooman2003a} has argued that the AGM axioms for belief revision \cite{alchourron1985,gardenfors1988} can be transported to the abstract setting of belief structures.
His revision axioms, which correspond one to one with the AGM revision postulates \(\mathrm{K}^*1\)--\(\mathrm{K}^*8\) in \cite{gardenfors1988}, are the following, when instantiated in the present context of AD-models~\(\M\in\adfMsabove{\bgM}\) and new information coming in the form of AD-assessments~\(\eventM\) corresponding to events~\(\eventopt\in\eventopts\):
\begin{enumerate}[label={\upshape BR\arabic*}.,ref={\upshape BR\arabic*},series=BR,widest=8,leftmargin=*]
\item\label{axiom:revision:1} \(\reviseof{\M\given\eventopt}\in\adfMsabove{\bgM}\cup\set{\Adelim{\opts}{\opts}}\) for all~\(\M\in\adfMsabove{\bgM}\) and all~\(\eventopt\in\eventopts\);
\item\label{axiom:revision:2} \(\eventM\subseteq\reviseof{\M\given\eventopt}\) for all~\(\M\in\adfMsabove{\bgM}\) and all~\(\eventopt\in\eventopts\);
\item\label{axiom:revision:3} \(\reviseof{\M\given\eventopt}\subseteq\expandof{\M\given\eventopt}\) for all~\(\M\in\adfMsabove{\bgM}\) and all~\(\eventopt\in\eventopts\);
\item\label{axiom:revision:4} \(\expandof{\M\given\eventopt}\subseteq\reviseof{\M\given\eventopt}\) for all~\(\M\in\adfMsabove{\bgM}\) and all~\(\eventopt\in\eventopts\) such that \(\M\union\eventM\) is \(\bgM\)-consistent;
\item\label{axiom:revision:5} \(\reviseof{\M\given\eventopt}\) is \(\bgM\)-consistent if \(\eventM\) is \(\bgM\)-consistent, for all~\(\M\in\adfMsabove{\bgM}\) and all~\(\eventopt\in\eventopts\);
\item\label{axiom:revision:6} \(\reviseof{\M\given\eventopt}=\reviseof{\M\given\Mclsbgof{\zeroM}{\eventM}}\) for all~\(\M\in\adfMsabove{\bgM}\) and all~\(\eventopt\in\eventopts\);
\item\label{axiom:revision:7} \(\reviseof{\M\given\eventopt[1]\sqcap\eventopt[2]}\subseteq\expandof{\reviseof{\M\given\eventopt[1]}\given\eventopt[2]}\) for all~\(\M\in\adfMsabove{\bgM}\) and all~\(\eventopt[1],\eventopt[2]\in\eventopts\);
\item\label{axiom:revision:8} \(\expandof{\reviseof{\M\given\eventopt[1]}\given\eventopt[2]}\subseteq\reviseof{\M\given\eventopt[1]\sqcap\eventopt[2]}\) for all~\(\M\in\adfMsabove{\bgM}\) and all~\(\eventopt[1],\eventopt[2]\in\eventopts\) such that \(\reviseof{\M\given\eventopt[1]}\cup\M_{\eventopt[2]}\) is \(\bgM\)-consistent.
\end{enumerate}

Let's briefly comment on these axioms.
\labelcref{axiom:revision:1} stipulates that the revised AD-assessment \(\reviseof{\M\given\eventopt}\) should either be an AD-model or the contradictory closed belief state~\(\Adelim{\opts}{\opts}\); that it should fully reflect the consequences of the observation of the event~\(\eventopt\), is expressed by~\labelcref{axiom:revision:2}.
\labelcref{axiom:revision:3} expresses that the result of a revision should always be at most as resolved as that of an expansion, and by \labelcref{axiom:revision:4} should coincide with it whenever the expansion leads to an AD-model, or in other words, whenever Your initial AD-model~\(\M\) and the new information~\(\eventM\) are \(\bgM\)-consistent.
The combination of \labelcref{axiom:revision:5,axiom:revision:2} leads to the requirement that the revised AD-assessment~\(\reviseof{\M\given\eventopt}\) should be \(\bgM\)-consistent if and only if the new information~\(\eventM\) is.

\labelcref{axiom:revision:6} deserves some extra attention, because there's no immediate interpretation for it in our current context.
In the original propositional logic context, \labelcref{axiom:revision:6} states that there should be no difference between adding a new proposition and adding its deductive closure.
In our present context, however, we look at revising an AD-model with new information that an event~\(\eventopt\) has occurred.
Below, we'll introduce a revision operator for this that acts on the old AD-model through the action of the projection operator~\(\eventopt\ast\bolleke\), rather through a new AD-assessment~\(\eventM\).
This means that the right-hand side in the equality in \labelcref{axiom:revision:6} won't make much sense.
However, we can preserve the spirit of the original axiom, by interpreting it in a more general way, as follows: \(\eventopt\) is in a one-to-one correspondence with~\(\eventindifset\) and therefore with~\(\eventM\), which in its turn is in a one-to-one correspondence with~\(\Mclsbgof{\bgM}{\eventM}=\Adelim{\bgM_\acc+\eventindifset}{-\group{\bgM_\des+\eventindifset}}\), at least for regular events~\(\eventopt\).
It's in this perhaps more general sense that revising an AD-model with an event~\(\eventopt\) will be equivalent to revising the background AD-model~\(\bgM\) with~\(\Mclsbgof{\bgM}{\eventM}\).
On this less strict way of looking at it, \labelcref{axiom:revision:6} will, in its original intent, definitely be satisfied for our revision operator, so we'll simply not take it into further consideration in what follows.

\labelcref{axiom:revision:7} and \labelcref{axiom:revision:8} are counterparts to \labelcref{axiom:revision:3} and \labelcref{axiom:revision:4} in the sense that they relate a revision and an expansion operation to each other.
\labelcref{axiom:revision:7} states that revising Your beliefs using the combined new information~\(\eventopt[1]\sqcap\eventopt[2]\) should lead to an AD-assessment that's at most as resolved as first revising Your beliefs using the new information~\(\eventopt[1]\) and then expanding the result using the new information~\(\eventopt[2]\).
Finally, \labelcref{axiom:revision:8} expresses that, whenever the result of the first revision~\(\reviseof{\M\given\eventopt[1]}\) and the new information~\(\eventopt[2]\) are \(\bgM\)-consistent, expanding the result of the first revision using the new information~\(\eventopt[2]\) should lead to an AD-assessment that coincides with the revised beliefs using the combined new information~\(\eventopt[1]\sqcap\eventopt[2]\).

Interestingly, and in contrast with belief expansion, \citeauthor{alchourron1985} \cite{alchourron1985,gardenfors1988} as well as \citeauthor{cooman2003a} in the wider context of belief models \cite{cooman2003a}, have argued that the axioms for belief revision don't single out a unique revision operator; there seem to be many reasonable ways of revising Your original beliefs with new information.
Near the end of \cref{sec::conditioning}, we came to the conclusion that conditioning a (conditionable) AD-model~\(\M\) on a (regular) event~\(\eventopt\) could be seen as a form of belief revision, because generally speaking \(\eventM\subseteq\M\condon\eventopt\) but not necessarily \(\M\subseteq\M\condon\eventopt\): the conditional~\(\M\condon\eventopt\) reflects the new information~\(\eventM\) but needn't reflect the original beliefs~\(\M\) completely.
In the remainder of this section, we'll investigate whether the discussion in \cref{sec::conditioning} can lead us to a reasonable revision operator that is based on the ideas behind conditioning.

Our suggestion for a revision operator is grounded in the following case analysis, based on the interplay between the background AD-model~\(\bgM\), Your initial AD-model~\(\M\) and the new information that the event~\(\eventopt\) has occurred, represented by the (indifference) AD-assessment~\(\eventM=\Adelim{\eventindifset}{\emptyset}\).

First, if the event~\(\eventopt\) isn't \(\bgM\)-regular, this means that the new information that the event~\(\eventopt\) has occurred, represented by the (indifference) AD-assessment~\(\eventM=\Adelim{\eventindifset}{\emptyset}\), isn't consistent with the background~\(\bgM\), so there's no way to come up with a revised AD-model that respects both the background and the new information.
In this case we're left with no other option than to suggest that the revised model should be the contradictory closed belief state~\(\Adelim{\opts}{\opts}\), especially if we want our revision operator to be in line with the requirements~\labelcref{axiom:revision:1,axiom:revision:2}.

Otherwise, if the event~\(\eventopt\) is \(\bgM\)-regular and Your original AD-model~\(\M\) is conditionable on it, we suggest using the conditional~\(\M\condon\eventopt\) as the revised AD-assessment~\(\reviseof{\M\given\eventopt}\).

Finally, we turn to the case which at this point hasn't received any attention yet: what to suggest if the event~\(\eventopt\) is \(\bgM\)-regular but Your original AD-model~\(\M\) isn't conditionable on it?
We'll take the simplest possible approach here, and take \(\reviseof{\M\given\eventopt}\) to be the least resolved AD-model~\(\Mclsbgof{\bgM}{\eventM}=\Mclsof{\bgM\union\eventM}=\Adelim{\bgM_\acc+\eventindifset}{-\group{\bgM_\des+\eventindifset}}\) that reflects both the background~\(\bgM\) and the new information~\(\eventM=\Adelim{\eventindifset}{\emptyset}\), \emph{thus throwing away all preferences that were initially present in the AD-model~\(\M\)}.
Though this is a quite drastic step, it does guarantee that we end up with an AD-model with a clear interpretation: the knowledge that the occurrence of the event gives You contradicts Your original beliefs.
Since You clearly can't trust Your original beliefs, You choose to become as conservative as possible and only retain the new knowledge, which You're now assumed to be certain of.\footnote{This simplest possible approach essentially consists in taking the \emph{full meet revision} of Your original AD-model~\(\M\) with the new information~\(\eventM\) in this case; see also \citeauthor{alchourron1985} \cite{alchourron1985,gardenfors1988} as well as \citeauthor{cooman2003a} in the wider context of belief models \cite{cooman2003a}. Other approaches, which preserve more of the original AD-model, are also possible, but we'll not consider them here.}

Summarising, we propose the following revision operator:
\begin{align}
\reviseof{\M\given\eventopt}
\coloneqq&
\begin{cases}
\Adelim{\M_\acc\altcondon\eventopt}{-\group{\M_\des\altcondon\eventopt\cup\group{\bgM_\des+ \M_\acc\altcondon\eventopt}}}
&\text{ if }0\notin\M_\acc\altcondon\eventopt+\bgM_\des\\
\Mclsof{\bgM\union\eventM}
&\text{ otherwise}
\end{cases}
\label{eq::AD:revision:operator}\\
=&
\begin{cases}
\Adelim{\M_\acc\altcondon\eventopt}{-\group{\M_\des\altcondon\eventopt\cup\group{\bgM_\des+ \M_\acc\altcondon\eventopt}}}
&\text{if }0\notin\M_\acc\altcondon\eventopt+\bgM_\des\\
\Adelim{\bgM_\acc+\eventindifset}{-\group{\bgM_\des+\eventindifset}}
&\text{if }0\in\M_\acc\altcondon\eventopt+\bgM_\des\text{ and }0\notin\bgM_\des+\eventindifset\\
\Adelim{\opts}{\opts}
&\text{if }0\in\bgM_\des+\eventindifset,
\end{cases}
\label{eq::AD:revision:operator:rewritten}
\end{align}
where the second equality follows from \cref{eq::model:closure:for:expansion} for \(\M\instantiateas\bgM\).

\begin{proposition}\label{prop::AD:model:revision:axioms:obeyed}
The belief revision operator \(\reviseof{\M\given\eventopt}\) in \cref{eq::AD:revision:operator} satisfies \labelcref{axiom:revision:1,axiom:revision:2,axiom:revision:3,axiom:revision:5,axiom:revision:7}.
\end{proposition}
\noindent We see that amongst the AGM axioms for belief revision that are still up for verification in the context of AD-models, namely \labelcref{axiom:revision:1,axiom:revision:2,axiom:revision:3,axiom:revision:4,axiom:revision:5,axiom:revision:7,axiom:revision:8}, only \labelcref{axiom:revision:4,axiom:revision:8} are at this point still unconfirmed for the revision operator based on conditioning suggested in \cref{eq::AD:revision:operator:rewritten}.
Interestingly, and perhaps surprisingly, there's no hope for these two revision axioms to be satisfied in general, as the following counterexample makes clear.
We'll show in the next two sections, however, that \labelcref{axiom:revision:4,axiom:revision:8} are satisfied for interesting and specific subclasses of our AD-models: (i) the subclass of propositional AD-models that corresponds to classical propositional inference; and (ii) the subclass of full previsional AD-models, wich corresponds to conditioning in (precise) probability theory.

\begin{counterexample}{\labelcref{axiom:revision:4}}
Consider a classical experiment with four possible outcomes and the AD-model~\(\M=\Adelim{\M_\acc}{-\M_\des}\) with
\begin{equation*}
\M_\acc=\posiof{\gambles_{\geq0}\cup\set{(-1,1,-1,1)}}
\text{ and }
\M_\des=\gambles_{\stronggt0}\cup\posiof{\gambles_{\stronggt0}+\set{(-1,1,-1,1)}}.
\end{equation*}
This assessment is clearly an AD-model: the background~\(\bgM=\Adelim{\gambles_{\geq0}}{-\gambles_{\stronggt0}}\) as defined in the example in \cref{sec::background} is included, \(0\notin \M_\des\), both \(\M_\acc\) and \(\M_\des\) are convex cones and \(\M_\acc+\M_\des\subseteq\M_\des\).
The last statement can be verified as follows: any gamble~\(\gbl\in\M_\acc\) can be written as
\begin{equation*}
\gbl=(a,b,c,d)+\alpha(-1,1,-1,1),
\text{ with \(a,b,c,d\in\nonnegreals\) and \(\alpha\in\nonnegreals\)},
\end{equation*}
while any gamble~\(\altgbl\in\M_\des\) can be written as
\begin{equation*}
\altgbl=(w,x,y,z)+\beta (-1,1,-1,1),
\text{ with \(w,x,y,z\in\posreals\) and \(\beta\in\nonnegreals\)}.
\end{equation*}
Their sum can then be written as:
\begin{equation*}
\gbl+\altgbl=(a+w,b+x,c+y,d+z)+(\alpha+\beta)(-1,1,-1,1).
\end{equation*}
The first term is clearly an element of \(\gambles_{\stronggt0}\) and \(\alpha+\beta\in\nonnegreals\), so \(\gbl+\altgbl\in\M_\des\).

Consider now the event~\(\indof{\event}\coloneqq(1,1,0,0)\).
Clearly, \(\M_\des\cap\indifset[\event]=\emptyset\), so \(\M\) and \(\M_{\event}\) are consistent.
Additionally, we have that any gamble \(\altgbltoo\in\M_\acc\altcondon\event\) can be written
\begin{equation}\label{eq::counterexample:4:vacuous:conditional}
\altgbltoo=(a,b,\alpha,\beta),
\text{ with \(a,b\in\nonnegreals\) and \(\alpha,\beta\in\reals\)},
\end{equation}
so we see that \(0\notin\M_\acc\altcondon{\event}+\gambles_{\stronggt0}\), and therefore \(\M\) is conditionable on~\(\event\), and the revision operator becomes
\begin{equation*}
\reviseof{\M\given\event}
=\Adelim{\M_\acc\altcondon{\event}}{-(\M_\des\altcondon{\event}\cup(\M_\acc\altcondon{\event}+\gambles_{\stronggt0}))}.
\end{equation*}
In order for \labelcref{axiom:revision:4} to be obeyed, we should then have that \(\expandof{\M\given\event}\subseteq\reviseof{\M\given\event}\), which implies in particular that
\begin{equation*}
\M_\acc+\indifset[\event]\subseteq\M_\acc\altcondon\event.
\end{equation*}
If we now look at the gamble~\(\altgbltoo\coloneqq (-1,1,-1,1)\), we see that \(\altgbltoo\in\M_\acc+\indifset[\event]\), but we infer from \cref{eq::counterexample:4:vacuous:conditional} that \(\altgbltoo\notin\M_\acc\altcondon{\event}\).

This counterexample is special in the sense that the lower probability~\(\lprev_{\M}\group{\indof{\event}}\) of the conditioning event~\(\event\), associated with the AD-model~\(\M\), is 0.
Indeed, recalling the definition from \cref{eq::definition:lprev:uprev}, it's easy to check that \(\lprev_{\M}\group{\indof{\event}}=\sup\set{\alpha\in\reals\colon\indof{\event}-\alpha\in\M_\acc}=0\).
Moreover, we infer from \cref{eq::counterexample:4:vacuous:conditional} that \(\altgbltoo\in\M_\acc\altcondon{\event}\ifandonlyif\inf\group{\altgbltoo\vert\event}\geq0\), so after conditioning on \(\event\), the model becomes vacuous; conditioning leads to a loss of resolvedness (or precision), even in this case where there is consistency between the original model and conditioning event.
And this is exactly the reason why \labelcref{axiom:revision:4} fails: the fact that conditioning may lead to a loss of resolvedness (or precision) in an imprecise probabilities context.

Interestingly, that the same thing happens when conditioning on the complement of the event \(\compof{\event}\), follows from the symmetry of the setup.
This means that we're in a situation where \emph{dilation} occurs \cite{seidenfeld1993:dilation,levi1980a,walley1991,cooman2004b,augustin2013:itip}: there's a loss of resolvedness/precision when conditioning on either event in the partition \(\set{\event,\compof{\event}}\).
\end{counterexample}

Since \labelcref{axiom:revision:4} is a special case of \labelcref{axiom:revision:8} (where the first event~\(\eventopt[1]\) coincides with the unit event), the counterexample mentioned above is also a counterexample for \labelcref{axiom:revision:8}, albeit a trivial one.
For the sake of completeness, we also add a less trivial counterexample for \labelcref{axiom:revision:8}.

\begin{counterexample}{\labelcref{axiom:revision:8}}
Consider again the classical experiment from the previous counterexample, with four possible outcomes and the AD-model~\(\M=\Adelim{\M_\acc}{-\M_\des}\) now given by
\begin{equation*}
\M_\acc
=\posiof{\gambles_{\geq0}\cup\set{(-1,1,1,0)}}
\text{ and }
\M_\des
=\gambles_{\stronggt0} \cup \posiof{\gambles_{\stronggt0} +\set{(-1,1,1,0)}}.
\end{equation*}
A similar argument to the one in the previous counterexample shows that this assessment is an AD-model.
Also consider the events \(\indof{\event[1]}\coloneqq(1,1,1,0)\) and \(\indof{\event[2]}\coloneqq(1,0,1,0)\), so \(\indof{\event[1]\cap\event[2]}=(1,0,1,0)=\indof{\event[2]}\).

For the derivation, we'll rely on \cref{eq:BR7:revision,eq:BR7:expansion} in the proof of \cref{prop::AD:model:revision:axioms:obeyed} in the Appendix.
We're first going to show that \(\M\) is conditionable on~\(\event[1]\) and on~\(\event[1]\cap\event[2]=\event[2]\), or in other words that \(0\notin\M_\acc\altcondon{\event[1]}+\gambles_{\stronggt0}\) and \(0\notin\M_\acc\altcondon{\event[2]}+\gambles_{\stronggt0}\).

Any element of \(\M_\acc\altcondon{\event[1]}\) can be written as
\begin{equation}\label{eq:BR7:counterexample:M_acc:altcondon:event1}
(a-\alpha,b+\alpha,c+\alpha,\beta)
\text{ with \(a,b,c,\alpha\in\nonnegreals\) and \(\beta\in\reals\)},
\end{equation}
making it clear that \(0\notin\M_\acc\altcondon{\event[1]}+\gambles_{\stronggt0}\).
Similarly, any element of \(\M_\acc\altcondon{\event[2]}=\M_\acc\altcondon{(\event[1]\cap\event[2])}\) can be written as
\begin{equation}\label{eq:BR7:counterexample:M_acc:altcondon:event2}
(a,\beta,c,\gamma)
\text{ with \(a,c\in\nonnegreals\) and \(\beta,\gamma\in\reals\)},
\end{equation}
making it clear that \(0\notin\M_\acc\altcondon{\event[2]}+\gambles_{\stronggt0}\).

Next, we're going to show that \(\M_\des\altcondon{\event[1]}=\emptyset\), and therefore that \(0\notin\M_\des\altcondon{\event[1]}+\gambles_{\stronggt0}\).
For any gamble~\(\gbl=(\alpha,\beta,\gamma,\delta)\in\gambles\) to be an element of \(\M_\des\altcondon{\event[1]}\), we need to have that
\begin{equation*}
\indof{\event[1]}\gbl
=(\alpha,\beta,\gamma,0)
=(a,b,c,d)+\eta(-1,1,1,0)
\text{ for some \(a,b,c,d\in\posreals\) and \(\eta\in\nonnegreals\)},
\end{equation*}
which is, indeed, clearly impossible.

And lastly, we're going to show that \(\group{\M_\acc\altcondon{\event[1]}+\gambles_{\stronggt0}}\cap\indifset[{\event[2]}]=\emptyset\).
Indeed, any element of \(\M_\acc\altcondon{\event[1]}+\gambles_{\stronggt0}\) can be written as
\begin{equation*}
(a-\alpha,b+\alpha,c+\alpha,\beta)+(x,y,z,w)
\text{ with \(a,b,c,\alpha\in\nonnegreals\) and \(\beta\in\reals\) and \(x,y,z,w\in\posreals\)},
\end{equation*}
Calling off this gamble with~\(\indof{\event[2]}=(1,0,1,0)\) results in the gamble
\begin{equation*}
(a-\alpha+x,0,c+\alpha+z,0)
\text{ with \(a,c,\alpha\in\nonnegreals\) and \(x,z\in\posreals\)},
\end{equation*}
which can clearly never be equal to the zero gamble.

If we now look at \cref{eq:BR7:revision,eq:BR7:expansion} in the proof of \cref{prop::AD:model:revision:axioms:obeyed} in the Appendix, we see that
\begin{equation*}
\reviseof{\M\given\event[1]\cap\event[2]}_\acc
=\M_\acc\altcondon{(\event[1]\cap\event[2])}
=\M_\acc\altcondon{\event[2]}
\text{ and }
\expandof{\reviseof{\M\given\event[1]}\given\event[2]}_\acc
=\M_\acc\altcondon{\event[1]}+\indifset[{\event[2]}].
\end{equation*}
Consider now the gamble~\(\gbl\coloneqq (-1,1,1,0)+(0, \beta, 0, \gamma)\in\M_\acc\altcondon{\event[1]}+I_{\event[2]}\), with~\(\beta,\gamma\in\reals\); see also \cref{eq:BR7:counterexample:M_acc:altcondon:event1}.
We see that \(\indof{\event[1]\intersection\event[2]}f=(-1,0,1,0)\), so for this gamble~\(\gbl\) to be an element of \(\M_\acc\altcondon{(\event[1]\intersection\event[2])}=\M_\acc\altcondon{\event[2]}\), we'd need, according to \cref{eq:BR7:counterexample:M_acc:altcondon:event2}, that
\begin{equation*}
(-1,0,1,0)=(a,\beta,c,\gamma)
\text{ with \(a,c\in\nonnegreals\) and \(\beta,\gamma\in\reals\)},
\end{equation*}
which is, indeed, clearly impossible.
\end{counterexample}

\section{Special case: propositional AD-models}\label{sec::propositional:models}
The AGM framework for belief revision was developed in a propositional context, where Your belief state is a set of propositions that You deem to be true.
Interestingly, it turns out that this propositional context, and the conservative inference mechanism of logical deduction that's associated with it, can be seen as a  special case of---can be embedded in---our more general framework of AD-models; this was studied in some detail by \citeauthor{vancamp2013} \cite{vancamp2013}.
This implies that we can also investigate what our abstract theory of conditioning and the related type of belief revision, studied in their full generality \cref{sec::conditioning,sec::conditioning:as:belief:change}, amount to in the more concrete propositional context studied by \citeauthor{alchourron1985} \cite{alchourron1985,gardenfors1988,rott2001}.

\subsection{Propositional AD-models}
To make things as easy as possible, we'll look at a \emph{simplified} propositional context, where You can make propositional statements~\(p(\varstate)\) about some variable~\(\varstate\) that may assume values in some set~\(\states\).
Any such propositional statement, or \define{proposition}, \(p(\varstate)\) can then be identified with a subset~\(A_p\) of the possibility space~\(\states\), containing those elements~\(\state\) of~\(\states\) for which the proposition~\(p(x)\) is true.
The set of all propositions about~\(\varstate\) can therefore also be identified with the power set~\(\powersetof{\states}\), the set of all subsets of~\(\states\).

A \define{filter base} is a non-empty subset~\(\filterbase\) of \(\powersetof{\states}\) that's closed under finite intersections: if \(A\in\filterbase\) and \(B\in\filterbase\), then \(A\cap B\in\filterbase\).
If, additionally, \(\emptyset\not\in\filterbase\), then we call \(\filterbase\) a \define{proper filter base}.
A \define{filter}, then, is a non-empty subset~\(\filter\) of \(\powersetof{\states}\) that's: (i) closed under finite intersections and (ii) increasing: if \(A\in\filter \) and \(A\subseteq B\), then \(B\in\filter\).
A filter~\(\filter\) is called \define{proper} if \(\emptyset\not\in\filter\), or equivalently, if~\(\filter\neq\powersetof\states\).
In our propositional context, these proper filters correspond to the deductively closed sets of propositions, and the conservative inference mechanism of deductive logic amounts to finding the smallest proper filter that includes a given set of propositions (subsets), if it exists, in which case the given set of propositions~\(\props\) is called \define{logically consistent}.

It's common knowledge that the set of propositions~\(\props\) is logically consistent if and only if it satisfies the so-called \define{finite intersection property}, meaning that \(\emptyset\not\in\filterbase[\props]\), where
\begin{equation*}
\filterbase[\props]
\coloneqq\set[\bigg]{\bigcap_{k=1}^nA_k\given n\in\naturals,A_k\in\props}
\end{equation*}
is then a proper filter base.
The smallest proper filter that includes~\(\props\), is then given by
\begin{equation*}
\filter[\props]
\coloneqq\set{A\subseteq\states\given\group{\exists B\in\filterbase[\props]}\,B\subseteq A}.
\end{equation*}

To see how this all fits into our general AD-framework, we follow the gist of \citeauthor{vancamp2013}'s approach \cite{vancamp2013}.
Suppose You have an assessment~\(\props\) of propositions that You deem to be true.
The trick they followed consists in associating an AD-assessment of \emph{gambles} with this propositional assessment, then using the conservative inference mechanism associated with AD-models for gambles to this AD-assessment, and finally to identify the result of this AD-inference with a set of propositions, for which we can then show that it yields the same result as deductive logical inference for the set of propositions~\(\props\) would: the smallest proper filter~\(\filter[\props]\) that includes~\(\props\), if it exists, so if the set of propositions~\(\props\) is logically consistent.
This is what we now set out to do.

Since we'll be working with the set of all gambles~\(\gambles\) on the possibility space~\(\states\) as an option space, we find ourselves in the context of our running example for classical probability, and take the AD-model~\(\bgM\coloneqq\Adelim{\gambles_{\weakgeq0}}{-\gambles_{\stronggt0}}\) as the background.

With the set~\(\props\) of those propositions that You state to be true, we associate the AD-assessment~\(\A_\props\coloneqq\Adelim{\indifset[\props]}{\emptyset}\), with~\(\indifset[\props]\) given by
\begin{equation}\label{eq::propositional:assessment}
\indifset[\props]
\coloneqq\set{\indof{B}-1\given B\in\props}
\end{equation}
where \(\indof{B}\) is the so-called \define{indicator (gamble)} of the set~\(B\), which takes the value \(1\) on~\(B\) and \(0\) elsewhere.
In other words, we translate Your assessment that \(B\) is true into Your accepting a bet on~\(B\) at rate~\(1\).
The fact that You're accepting a bet on~\(B\) at rate~\(1\) implies that You'll accept both the gambles~\(\indof{B}-1\) and \(1-\indof{B}\geq0\), so \emph{You're indifferent between the uncertain reward~\(\indof{B}(\varstate)\) and the certain reward~\(1\)}.
In this way, we express that You're practically certain that~\(B\) is true.

Since this assessment~\(\A_\props\) needn't yet satisfy \labelcref{axiom:AD:background,axiom:AD:deductive:closedness,axiom:AD:no:limbo,axiom:AD:zero:not:desirable}, we can look for the smallest dominating AD-model in~\(\adfMsabove{\bgM}\), if it exists, by looking at its closure~\(\closureof{\Msabove{\bgM}}{\A_\props}\).
A similar argument was also presented by \citeauthor{vancamp2013} in \cite{vancamp2013}, so we refer to their work for the inspiration for the proof of the following proposition.

We introduce the following notations, for any set of propositions~\(\filter\subseteq\powersetof\states\):
\begin{equation*}
\gambles_{\weakgeq0}^{\filter}
\coloneqq\set{\gbl\in\gambles\given\group{\exists B\in\filter}
\,\inf\group{\gbl\vert B}\geq0}
\text{ and }
\gambles_{\stronggt0}^{\filter}
\coloneqq\set{\gbl\in\gambles\given\group{\exists B\in\filter}
\,\inf\group{\gbl\vert B}>0},
\end{equation*}
where we let \(\inf\group{\gbl\vert B}\coloneqq\inf\set{f(x)\given x\in B}\) for all~\(\gbl\in\gambles\) and all non-empty~\(B\subseteq\states\).

\begin{proposition}\label{prop::form:propositional:model}
The AD-assessment~\(\A_\props\) is \(\bgM\)-consistent if and only if \(\emptyset\notin\filterbase[\props]\), or in other words, if and only if the set of propositions~\(\props\) is logically consistent.
In that case the \(\bgM\)-natural extension of~\(\A_\props\) is given by
\begin{equation*}
\closureof{\adfMsabove{\bgM}}{\A_\props}
=\Adelim{\gambles_{\weakgeq0}^{\filter[\!\!\props]}}{-\gambles_{\stronggt0}^{\filter[\!\!\props]}}.
\end{equation*}
\end{proposition}
\noindent In other words, when the propositional assessment~\(\props\) is logically consistent, You should find acceptable at least all those gambles that are non-negative on some set in~\(\filter[\!\!\props]\), so those gambles that according to Your propositional assessment are guaranteed to give You a non-negative payoff; and You should find desirable at least all those gambles that are positive and bounded away from zero on some such set, so those gambles that according to Your propositional assessment are guaranteed to give You a uniformly strictly positive payoff.

This result allows us to single out a special class of AD-models:

\begin{definition}
An AD-model~\(\M\in\adfMsabove{\bgM}\) is called \define{propositional} if there's some proper filter~\(\filter\) such that \(\M=\M_{\filter}\), where we let \(\M_{\filter}\coloneqq\Adelim{\gambles_{\weakgeq0}^{\filter}}{-\gambles_{\stronggt0}^{\filter}}\).
We denote the set of all propositional AD-models by~\(\propadfMs\), so
\begin{equation*}
\propadfMs
\coloneqq\set{\M_{\filter}\given\filter\text{ is a proper filter}}
\subseteq\adfMsabove{\bgM}.
\end{equation*}
\end{definition}
\noindent In fact, as we're about to find out, there's a one-to-one correspondence between the propositional AD-models and the proper filters.
Indeed, let's associate, with any given AD-model~\(\M\in\adfMsabove{\bgM}\), the set \(\filter[\M]\subseteq\powersetof\states\) given by
\begin{equation*}
\filter[\M]
\coloneqq\set{B\subseteq\states\given\indof{B}-1\in\M_\acc}
=\set{B\subseteq\states\given\indof{B}-1\in\M_\indif},
\end{equation*}
so \(\filter[\M]\) is the set of all propositions~\(B\) for which Your AD-model~\(\M\) implies that You're indifferent between the uncertain reward~\(\indof{B}(\varstate)\) and the fixed reward~\(1\), so for which You're practically certain that the proposition~\(B\) is true.
As it turns out, \(\filter[\M]\) is always a proper filter, and we have that
\begin{equation}\label{eq::prop:filter:correspondence}
\group{\forall\M\in\adfMsabove{\bgM}}
\group{\M\in\propadfMs\ifandonlyif\M=\M_{\filter[\M]}}.
\end{equation}
The one-to-one-correspondence between the proper filters and the propositional AD-models is therefore given by the map \(\filter\mapsto\M_{\filter}\) and its inverse \(\M\mapsto\filter[\M]\).
Together with \cref{prop::form:propositional:model}, this shows that \emph{the inference mechanism associated with classical propositional logic can be seen as a special case of that associated with AD-models} via the \emph{embedding}~\(\filter\mapsto\M_{\filter}\).

Having come to this conclusion, it now makes sense to investigate what the belief expansion and revision operators for AD-models, as defined in \cref{eq::AD:expansion:operator,eq::AD:revision:operator:rewritten}, reduce to when applied only to \emph{propositional} AD-models.
This is what we now turn to.

\subsection{Belief expansion for propositional AD-models}
We begin by looking at the expansion operator \(\expandof{\bolleke\given\event}=\closureof{\Msabove{\bgM}}{\bolleke\cup\M_{\event}}\), which as we've seen reduces to \(\expandof{\M\given\event}=\Adelim{\M_\acc+\indifset[\event]}{-\group{\M_\des+\indifset[\event]}}\) when the AD-model~\(\M_{\event}\coloneqq\Adelim{\indifset[\event]}{\emptyset}\) representing the occurrence of the event~\(\event\) is consistent with the old AD-model~\(\M\in\adfMsabove{\bgM}\)---or in other words when \(\M_\des\cap\indifset[\event]=\emptyset\) or equivalently \(0\notin\M_\des+\indifset[\event]\)---and otherwise yields the contradictory belief state~\(\Adelim{\gambles}{\gambles}\).

Let's first have a closer look at the consistency condition.

\begin{proposition}\label{prop::propositional:expansion:consistency}
Consider any event~\(\event\subseteq\states\) and any propositional model~\(\M_{\filter}\in\propadfMs\), where \(\filter\) is any proper filter.
Then there's consistency between the AD-model~\(\M_{\filter}\) and the event~\(\event\)---or the AD-model~\(\M_{\event}\)---if and only if the augmented set of propositions~\(\filter\cup\set\event\) is logically consistent, so satisfies the finite intersection property.
\end{proposition}
\noindent Observe, by the way, that augmented propositional assessment~\(\filter\cup\set\event\) gives rise to
\begin{equation*}
\filterbase[\!\!\filter\cup\set\event]
=\set[\bigg]{\bigcap_{k=1}^nA_k\given n\in\naturals,A_k\in\filter\cup\set\event}
=\filter\cup\set{\event\cap B\given B\in\filter},
\end{equation*}
and therefore we find for the finite intersection property that
\begin{equation}\label{eq::augmented:finite:intersection:property}
\emptyset\notin\filterbase[\!\!\filter\cup\set\event]
\ifandonlyif
\emptyset\notin\set{\event\cap B\given B\in\filter},
\end{equation}
and in that case the smallest proper filter that includes \(\filter\cup\set\event\) is given by
\begin{align}
\filter[\!\!\filter\cup\set\event]
&=\set{B\subseteq\states\given\group{\exists D\in\filterbase[\!\!\filter\cup\set\event]}\, D\subseteq B}
\notag\\
&=\filter\cup\set{B\subseteq\states\given\group{\exists D\in\filter}\,D\cap\event\subseteq B}.
\label{eq::augmented:filter}
\end{align}

\begin{proposition}\label{prop::propositional:expansion:operator}
Let \(\event\in\powersetof\states\setminus\set{\emptyset}\) be any regular event, and let \(\filter\) be any proper filter, so \(\M_{\filter}\) is a propositional AD-model.
Then for the expansion operator as defined in \cref{eq::AD:expansion:operator}, we find that
\begin{equation}\label{eq::propositional:expansion:operator:definition}
\expandof{\M_{\filter}\given\event}
=\begin{cases}
\M_{\filter[\!\!\filter\cup\set\event]}
&\text{if \(\emptyset\notin\filterbase[\!\!\filter\cup\set\event]\)}\\
\Adelim{\gambles}{\gambles}
&\text{otherwise.}
\end{cases}
\end{equation}
\end{proposition}
\noindent This tells us that when there's consistency, the expansion operator turns any propositional AD-model into a new propositional AD-model, which is exactly the one that corresponds to adding the event~\(\event\) to the deductively closed set of propositions~\(\filter\) to find the smallest proper filter~\(\filter[\!\!\filter\cup\set\event]\) that includes \(\filter\cup\set\event\).
In other words, \emph{when restricted to propositional AD-models, the expansion operator simply embeds the action of the expansion operator in classical propositional logic into the structure of the more general AD-models}.

\subsection{Belief revision for propositional AD-models}
To conclude this discussion of the connection between AD-models and propositional logic, we now turn to the revision operator associated with conditioning, whose expression is given by \cref{eq::AD:revision:operator:rewritten}.
What happens when we restrict this revision operator to the propositional AD-models of the type~\(\M_{\filter}=\Adelim{\gambles_{\weakgeq0}^{\filter}}{-\gambles_{\stronggt0}^{\filter}}\), where \(\filter\) is any proper filter?

We've already established that the \emph{consistency condition}~\(0\notin\gambles_{\stronggt0}+\indifset[\event]\)---or in other words the \emph{regularity} of the event~\(\event\)--- is equivalent to~\(\event\) being non-empty, so the only remaining requirement to be checked in the context of \cref{eq::AD:revision:operator:rewritten} is the \emph{conditionability condition}~\(0\notin\gambles_{\weakgeq0}^{\filter}\altcondon\event+\gambles_{\stronggt0}\).

\begin{proposition}\label{prop::propositional:revision:conditionability}
For any proper filter~\(\filter\) and any regular event~\(\event\in\powersetof\states\setminus\set{\emptyset}\), the conditionability condition~\(0\notin\gambles_{\weakgeq0}^{\filter}\altcondon\event+\gambles_{\stronggt0}\) is equivalent to the logical consistency of the set of propositions~\(\filter\cup\set\event\), or in other words, to \(\emptyset\not\in\filterbase[\!\!\filter\cup\set\event]\).
\end{proposition}

We then have the following result, in \cref{prop::propositional:revision:operator}, which shows that when there's logical consistency (or equivalently by \cref{prop::propositional:revision:conditionability}, conditionability), the revision operator associated with conditioning turns any propositional AD-model into a new propositional AD-model, which is exactly the one that corresponds to adding the event~\(\event\) to the deductively closed set of propositions~\(\filter\) to find the smallest proper filter~\(\filter[\!\!\filter\cup\set\event]\) that includes \(\filter\cup\set\event\).
And when there's no logical consistency, the proposed revision operator ends up completely replacing the original propositional AD-model~\(\M_{\filter}\) with the new propositional AD-model\footnote{The second equality follows directly from \cref{prop::propositional:expansion:operator} with \(\filter\instantiateas\set{\states}\) and therefore \(\filter[\!\!\filter\cup\set\event]\instantiateas\filter[\set{\event}]\).}
\begin{equation*}
\closureof{\adfMsabove{\bgM}}{\M_{\event}}
=\Adelim{\gambles_{\weakgeq0}+\indifset[\event]}{-\group{\gambles_{\stronggt0}+\indifset[\event]}}
=\Adelim{\gambles_{\weakgeq0}^{\filter[\!\!\set\event]}}{-\gambles_{\stronggt0}^{\filter[\!\!\set\event]}}
=\M_{\filter[\!\!\set\event]},
\end{equation*}
corresponding to the single propositional statement that \(\event\) has occurred.\footnote{This is a generalisation of the full meet revision in classical propositional logic \cite{gardenfors1988,alchourron1985,cooman2003a}.}
In other words, \emph{when restricted to propositional AD-models, the revision operator associated with conditioning as defined in \cref{eq::AD:revision:operator:rewritten} simply embeds the action of the full meet revision operator in classical propositional logic into the structure of the more general AD-models}.\footnote{In the case of logical inconsistency, other proposals can be made for the behaviour of the revision operator, besides the one we're proposing here for reasons of simplicity. A more detailed exploration of such alternatives is beyond the scope of this paper.}

\begin{proposition}\label{prop::propositional:revision:operator}
Let \(\event\in\powersetof\states\setminus\set{\emptyset}\) be any regular event, and let \(\filter\) be any proper filter, so \(\M_{\filter}\) is a propositional AD-model.
Then for the revision operator as defined in \cref{eq::AD:revision:operator:rewritten}, we find that
\begin{equation}\label{eq::propositional:revision:operator:definition}
\reviseof{\M_{\filter}\given\event}
=\begin{cases}
\M_{\filter[\!\!\filter\cup\set\event]}
&\text{if \(\emptyset\notin\filterbase[\!\!\filter\cup\set\event]\)}\\
\M_{\filter[\!\!\set\event]}
&\text{otherwise}.
\end{cases}
\end{equation}
\end{proposition}

The combination of \cref{prop::propositional:expansion:operator,prop::propositional:revision:operator} proves that \labelcref{axiom:revision:4} \emph{does} hold when we restrict our conditioning-based revision operator to propositional AD-models.
That \labelcref{axiom:revision:8} holds as well in this more restrictive setting can be given a similar argumentation: successive conditioning simply amounts to successively adding more and more propositions---or events---to the set~\(\props\), and demanding that the finite intersection property should hold is equivalent to requiring logical consistency.

In summary, then, when we restrict the revision operator based on our newly proposed conditioning rule for AD-models to propositional AD-models, we see that it satisfies our version of the AGM axioms~\labelcref{axiom:revision:1,axiom:revision:2,axiom:revision:3,axiom:revision:4,axiom:revision:5,axiom:revision:7,axiom:revision:8}, and in fact corresponds to what is known as a full meet revision.

\begin{proposition}\label{prop::postulates:satisfied:propositional}
When we restrict the action of the revision operator as defined in \cref{eq::propositional:revision:operator:definition} to propositional AD-models, it satisfies \labelcref{axiom:revision:1,axiom:revision:2,axiom:revision:3,axiom:revision:4,axiom:revision:5,axiom:revision:7,axiom:revision:8}.
\end{proposition}

\section{Special case: coherent full conditional previsions}\label{sec::full:conditional:previsions}
In his book on the subject \cite{gardenfors1988}, \citeauthor{gardenfors1988} not only discusses the AGM framework for belief change in the context of classical propositional logic, but also extends the discussion to probabilistic belief models, as captured by conditional probabilities.
Indeed, he's inclined there to see the belief change of a probabilistic belief state related to the observation of a new event, or in other words, conditioning a probability measure on the event, as a \emph{belief expansion} resulting from a new propositional statement.

We're going to follow a different route here, and show that, in the context of AD-models, conditioning probabilistic models should rather be seen as a form of \emph{belief revision}.
Indeed, earlier work by \citeauthor{quaeghebeur2015:statement} \cite{quaeghebeur2015:statement} has shown that AD-models are also able to capture probabilistic models, besides the propositional models discussed in the previous section; see also the discussion in \cref{sec::AD-previsions}.
It therefore behoves us to investigate what our newly proposed conditioning rule for AD-models amounts to when applied to such probabilistic models: first, whether it coincides with Bayesian conditioning, and second, whether the belief revision operator associated with it then satisfies our version of the AGM-axioms.

\subsection{Coherent full conditional previsions}
Let's first identify the specific type of AD-models we'll be interested in, in the present context: we'll go beyond the scope of \cref{sec::AD-previsions}, and consider \emph{coherent full conditional previsions}; for more information about these objects and their properties, we also refer to the much more extensive discussion in \citeauthor{troffaes2013:lp}'s book on coherent lower previsions \cite{troffaes2013:lp}, as well as to \citeauthor{williams1975}'s seminal paper \cite{williams1975}.

As in the previous section, we'll be working with AD-models consisting of \emph{gambles} on some uncertain variable~\(\varstate\) that assumes values in some set of possible values~\(\states\).
We take the AD-model~\(\bgM\coloneqq\Adelim{\gambles_{\weakgeq0}}{-\gambles_{\stronggt0}}\) as a background.

We're going to assume that You've specified Your beliefs about the uncertain variable~\(\varstate\) by means of a \define{full conditional prevision}~\(\prevof{\bolleke\given\bolleke}\colon\gambles\times\powersetof{\states}\setminus\set{\emptyset}\to\reals\), which maps any gamble~\(\gbl\in\gambles\) and regular event~\(B\in\powersetof{\states}\setminus\set{\emptyset}\) to Your prevision~\(\prevof{\gbl\given B}\) of~\(\gbl\) conditional on the event~\(B\).
We'll assume that this full conditional prevision is \define{coherent}, which means that it satisfies the following conditions \cite{troffaes2013:lp}:
\begin{enumerate}[label={\upshape CC\arabic*.},ref={\upshape CC\arabic*},series=CC,widest=4,leftmargin=*,itemsep=0pt]
\item\label{axiom:coherence:conditional:bounds} \(\inf\group{\gbl\vert B}\leq\prevof{\gbl\given B}\leq\sup\group{\gbl\vert B}\) for all~\(\gbl\in\gambles\) and all~\(B\in\powersetof{\states}\setminus\set{\emptyset}\);\nameit{Bounds}
\item\label{axiom:coherence:conditional:linearity} \(\prevof{\lambda\gbl+\mu\altgbl\given B}=\lambda\prevof{\gbl\given B}+\mu\prevof{\altgbl\given B}\) for all~\(\gbl,\altgbl\in\gambles\), all~\(\lambda,\mu\in\reals\) and all regular events~\(B\in\powersetof{\states}\setminus\set{\emptyset}\);\nameit{Linearity}
\item\label{axiom:coherence:conditional:bayes} \(\prevof{\gbl\indof{C}\given B}=\prevof{\gbl\given B\cap C}\prevof{\indof{C}\given B}\) for all~\(\gbl\in\gambles\) and all events~\(B,C\in\powersetof{\states}\) such that \(B\cap C\neq\emptyset\) is regular.\nameit{Bayes' Rule}
\end{enumerate}
It can be proved \cite[Theorem 13.29]{troffaes2013:lp} that for any full conditional prevision, coherence is equivalent to the condition of \define{avoiding sure loss}\emph{:}
\begin{multline}\label{eq::avoiding:sure:loss}
\sup\group[\bigg]{\sum_{k=1}^n\lambda_k\indof{B_k}\group{\gbl_k-\prevof{\gbl_k\given B_k}+\epsilon_k}\Big\vert\bigcup_{k=1}^nB_k}\geq0\\
\text{ for all \(n\in\naturals\), \(\lambda_k,\epsilon_k\in\posreals\), \(\gbl_k\in\gambles\) and \(B_k\in\powersetof{\states}\setminus\set{\emptyset}\).}
\tag{ASL}
\end{multline}

This full conditional prevision~\(\prevof{\bolleke\given\bolleke}\) can be used to construct the following  AD-assessment of gambles~\(\A_{\prevof{\smallbolleke\given\smallbolleke}}\coloneqq\Adelim{\mathscr{L}_{\prevof{\smallbolleke\given\smallbolleke}}}{\emptyset}\), where
\begin{equation}\label{eq::full:cond:prev:assessment}
\mathscr{L}_{\prevof{\smallbolleke\given\smallbolleke}}
\coloneqq\set{\indof{B}\group{\gbl-\prevof{\gbl\given B}+\epsilon}\given\gbl\in\gambles,B\in\powersetof{\states}\setminus\set{\emptyset},\epsilon\in\posreals}.
\end{equation}
This expresses that You accept to buy any gamble~\(\gbl\) for any price~\(\prevof{\gbl\given B}-\epsilon\) strictly lower than Your conditional prevision~\(\prevof{\gbl\given B}\), on the condition that the buying transaction is called off unless the event~\(B\) occurs.
And since \(\gbl-\prevof{\gbl\given B}+\epsilon=\prevof{-\gbl\given B}-\group{-\gbl}+\epsilon\), it also means that You accept to sell any gamble~\(\altgbl\coloneqq-\gbl\) for any price~\(\prevof{\altgbl\given B}+\epsilon\) strictly higher than Your conditional prevision~\(\prevof{\altgbl\given B}\), on the condition that the selling transaction is called off unless the event~\(B\) occurs.

Given this AD-assessment~\(\A_{\prevof{\smallbolleke\given\smallbolleke}}\), we're now going to find out whether we can associate an AD-model in~\(\adfMsabove{\bgM}\) with it.
For this, we first need to check its \(\bgM\)-consistency.

\begin{proposition}\label{prop::previsional:consistency}
For any coherent full conditional prevision~\(\prevof{\bolleke\given\bolleke}\), the corresponding AD-assessment~\(\A_{\prevof{\smallbolleke\given\smallbolleke}}\) is \(\bgM\)-consistent.
\end{proposition}

This allows us to take the \(\bgM\)-natural extension of the AD-assessment~\(\A_{\prevof{\smallbolleke\given\smallbolleke}}\) to find AD-model that can be associated with a coherent full conditional prevision~\(\prevof{\bolleke\given\bolleke}\)

\begin{proposition}\label{prop::previsional:natural:extension}
The \(\bgM\)-natural extension~\(\M_{\prevof{\smallbolleke\given\smallbolleke}}\coloneqq\closureof{\Msabove{\bgM}}{\A_{\prevof{\smallbolleke\given\smallbolleke}}}\) of the \(\bgM\)-consistent AD-assessment~\(\A_{\prevof{\smallbolleke\given\smallbolleke}}\) is the AD-model in~\(\adfMsabove{\bgM}\) given by
\begin{equation}\label{eq::model:conditional:full:prevision}
\M_{\prevof{\smallbolleke\given\smallbolleke}}
=\Adelim{\gambles_{\weakgeq0}\cup\group{\posiof{\mathscr{L}_{\prevof{\smallbolleke\given\smallbolleke}}}+\gambles_{\weakgeq0}}}{-\group{\gambles_{\stronggt0}\cup\group{\posiof{\mathscr{L}_{\prevof{\smallbolleke\given\smallbolleke}}}+\gambles_{\stronggt0}}}}.
\end{equation}
\end{proposition}

This allows us to single out a special class of AD-models:

\begin{definition}
An AD-model~\(\M\in\adfMsabove{\bgM}\) is called a \define{full previsional AD-model} if there's some coherent full conditional prevision~\(\prevof{\bolleke\given\bolleke}\) such that \(\M=M_{\prevof{\bolleke\given\bolleke}}\).
We denote the set of all full previsional AD-models by~\(\prevadfMs\), so
\begin{equation*}
\prevadfMs
\coloneqq\set{\M\in\adfMsabove{\bgM}\given\M=\M_{\prevof{\bolleke\given\bolleke}}\text{ for some coherent }\prevof{\bolleke\given\bolleke}}
\subseteq\adfMsabove{\bgM}.
\end{equation*}
\end{definition}
\noindent In fact, as we're about to prove, there's a one-to-one correspondence between the full previsional AD-models and the coherent full conditional previsions.
Indeed, let's associate, with any given AD-model~\(\M\in\adfMsabove{\bgM}\), the map \(\lprev_{\M}\group
{\bolleke\condon\bolleke}\colon\gambles\times\powersetof{\states}\setminus\set{\emptyset}\to\reals\cup\set{-\infty,\infty}\) defined by
\begin{equation*}
\lprev_{\M}\group{\gbl\vert B}
\coloneqq\sup\set{\alpha\in\reals\given\indof{B}\group{\gbl-\alpha}\in\M_\acc},
\text{ for all }\gbl\in\gambles\text{ and }B\in\powersetof{\states}\setminus\set{\emptyset},
\end{equation*}
and the conjugate map \(\uprev_{\M}\group{\bolleke\condon\bolleke}\colon\gambles\times\powersetof{\states}\setminus\set{\emptyset}\to\reals\cup\set{-\infty,\infty}\) defined by
\begin{multline*}
\uprev_{\M}\group{\gbl\vert B}
\coloneqq\inf\set{\beta\in\reals\given\indof{B}\group{\beta-\gbl}\in\M_\acc}
=-\lprev_{\M}\group{-\gbl\vert B},\\
\text{ for all }\gbl\in\gambles\text{ and }B\in\powersetof{\states}\setminus\set{\emptyset}.
\end{multline*}
Then
\begin{equation}\label{eq::prevision:model:correspondence}
\lprev_{\M_{\prevof{\smallbolleke\given\smallbolleke}}}\group{\bolleke\condon\bolleke}=\prevof{\bolleke\given\bolleke}
\text{ for any coherent full conditional prevision~\(\prevof{\bolleke\given\bolleke}\)},
\end{equation}
and also, for all AD-models \(\M\in\adfMsabove{\bgM}\),
\begin{equation}\label{eq::model:prevision:correspondence}
\M\in\prevadfMs\ifandonlyif\group{\M=\M_{\lprev_{\M}\group{\smallbolleke\condon\smallbolleke}} \text{ and } \lprev_{\M}\group{\bolleke\condon\bolleke} \text{ is a coherent full conditional prevision}}.
\end{equation}
The one-to-one-correspondence between the coherent full conditional previsions and the full previsional AD-models is therefore given by the map \(\prevof{\bolleke\given\bolleke}\mapsto M_{\prevof{\smallbolleke\given\smallbolleke}}\) and its inverse~\(\M\mapsto\lprev_{\M}\group{\bolleke\condon\bolleke}\), restricted to the set of full previsional AD-models.

\subsection{Belief revision for full previsional AD-models}
For didactic reasons, and contrary to what we've been doing so far, we'll first concentrate on \emph{belief revision} for full previsional AD-models, and only afterwards turn to \emph{belief expansion}.

We'll investigate what the concrete revision operator~\(\reviseof{\bolleke\given\event}\), associated with conditioning on a regular event~\(\event\) and introduced in \cref{eq::AD:revision:operator:rewritten}, does when it acts on a full previsional AD-model \(\M_{\prevof{\smallbolleke\given\smallbolleke}}\in\prevadfMs\).

Before going there, it'll be useful to introduce some new notation that will help us interpret and better understand what's going on.
We'll need to consider a new conditional prevision~\(\prev_{\event}\group{\bolleke\condon\bolleke}\colon\gambles\times\powersetof{\event}\setminus\set{\emptyset}\to\reals\) as a simple restriction of the full conditional prevision~\(\prevof{\bolleke\given\bolleke}\colon\gambles\times\powersetof{\states}\setminus\set{\emptyset}\to\reals\), as follows:
\begin{equation*}
\prev_{\event}\group{\gbl\condon C}
\coloneqq\prevof{\gbl\condon C}
\text{ for all \(\gbl\in\gambles\) and \(C\in\powersetof{\event}\setminus\set\emptyset\)}.
\end{equation*}
This means that \(\prev_{\event}\group{\bolleke\condon\bolleke}\) is \emph{no longer a \emph{full} conditional prevision}, in the sense that it no longer allows conditioning on all non-empty subsets of~\(\states\), but only on the non-empty subsets of~\(\event\); \emph{it loses all the information associated with conditioning on events that aren't subsets of~\(\event\).}
With this \emph{restricted} conditional prevision~\(\prev_{\event}\group{\bolleke\condon\bolleke}\), we can now, as before construct an AD-assessment of gambles~\(\A_{\prev_{\event}\group{\smallbolleke\condon\smallbolleke}}\coloneqq\Adelim{\mathscr{L}_{\prev_{\event}\group{\smallbolleke\condon\smallbolleke}}}{\emptyset}\), where
\begin{equation*}
\mathscr{L}_{\prev_{\event}\group{\smallbolleke\condon\smallbolleke}}
\coloneqq\set{\indof{C}\group{\gbl-\prev_{\event}\group{\gbl\condon C}+\epsilon}\given\gbl\in\gambles,C\in\powersetof{\event}\setminus\set{\emptyset},\epsilon\in\posreals}.
\end{equation*}
We also recall the definitions of the sets \(\gambles^{\set\event}_{\weakgeq0}\) and \(\gambles^{\set\event}_{\stronggt0}\) from \cref{sec::propositional:models}:
\begin{equation*}
\gambles^{\set\event}_{\weakgeq0}
=\set{\gbl\in\gambles\colon\inf\group{\gbl\vert\event}\geq0}
\text{ and }
\gambles^{\set\event}_{\stronggt0}
=\set{\gbl\in\gambles\colon\inf\group{\gbl\vert\event}>0}.
\end{equation*}

Recall from its definition in \cref{eq::AD:revision:operator:rewritten} that the workings of the revision operator~\(\reviseof{\bolleke\given\event}\) depend on whether the AD-model~\(\M\) that it acts on, is conditionable on the event~\(\event\) whose observation prompts the revision.
Previsional AD-models are special, in that they're always conditionable on all regular events.

\begin{proposition}\label{prop::previsional:revision:conditionability}
Any full previsional AD-model~\(\M_{\prevof{\smallbolleke\given\smallbolleke}}\in\prevadfMs\) is conditionable on all regular events~\(\event\in\powersetof{\states}\setminus\set{\emptyset}\).
\end{proposition}

It's now but a small step to find out what the revision operator~\(\reviseof{\bolleke\given\event}\) does when it acts on a full previsional AD-model \(\M_{\prevof{\smallbolleke\given\smallbolleke}}\in\prevadfMs\).

\begin{proposition}\label{prop::previsional:revision}
Consider any regular event~\(\event\in\powersetof{\states}\setminus\set{\emptyset}\) and any full previsional AD-model~\(\M_{\prevof{\smallbolleke\given\smallbolleke}}\in\prevadfMs\), then
\begin{multline}\label{eq::previsional:revision}
\reviseof{\M_{\prevof{\smallbolleke\given\smallbolleke}}\given\event}\\
=\Adelim[\big]{\gambles^{\set\event}_{\weakgeq0}\cup\group{\posiof{\mathscr{L}_{\prev_{\event}\group{\smallbolleke\condon\smallbolleke}}}+\gambles^{\set\event}_{\weakgeq0}}}{-\group[\big]{\gambles^{\set\event}_{\stronggt0}\cup\group{\posiof{\mathscr{L}_{\prev_{\event}\group{\smallbolleke\condon\smallbolleke}}}+\gambles^{\set\event}_{\stronggt0}}}}.
\end{multline}
\end{proposition}

\noindent A few interesting conclusions can be drawn from this result, and from a comparison of the form of \(\reviseof{\M_{\prevof{\smallbolleke\given\smallbolleke}}\given\event}\) in \cref{eq::previsional:revision} with that of \(\M_{\prevof{\smallbolleke\given\smallbolleke}}\) in \cref{eq::model:conditional:full:prevision}.

First, the revised AD-model~\(\reviseof{\M_{\prevof{\smallbolleke\given\smallbolleke}}\given\event}\) can still be interpreted as a full previsional AD-model, but now for the restricted possibility space~\(\event\), rather than for the entire possibility space~\(\states\).
In this specific, and perhaps somewhat loose sense, the revision operator based on our conditioning rule is \emph{internal} in the full previsional AD-models.

Second, \cref{eq::previsional:revision} makes clear in exactly what way the revision of~\(\M_{\prevof{\smallbolleke\given\smallbolleke}}\) with the information~\(\M_{\event}=\Adelim{\indifset[\event]}{\emptyset}\) based on occurrence of the event~\(\event\) works.
It \emph{removes} from the original AD-model~\(\M_{\prevof{\smallbolleke\given\smallbolleke}}\) all statements related to conditioning on those regular events that aren't subsets of the observed event~\(\event\), and \emph{adds} the results of making the new indifference statements in~\(\indifset[\event]\), by replacing the old background~\(\bgM=\Adelim{\gambles_{\weakgeq0}}{-\gambles_{\stronggt0}}\) with the new, \define{augmented}, background~\(\closureof{\adfMsabove{\bgM}}{\M_{\event}}=\Adelim{\gambles_{\weakgeq0}+\indifset[\event]}{-\group{\gambles_{\stronggt0}+\indifset[\event]}}=\Adelim{\gambles^{\set\event}_{\weakgeq0}}{-\gambles^{\set\event}_{\stronggt0}}\).

\subsection{Belief expansion for full previsional AD-models}
Finally, we investigate what the expansion operator \(\expandof{\bolleke\given\event}\) for a regular event~\(\event\in\powersetof{\states}\setminus\set{\emptyset}\), as defined in \cref{eq::AD:expansion:operator}, reduces to when applied to full previsional AD-models; so what becomes of
\begin{equation*}
\expandof{\M\given\event}
\coloneqq\begin{cases}
\Adelim{\M_\acc+\indifset[\event]}{-\group{\M_\des+\indifset[\event]}}
&\text{if \(\M_\des\cap\indifset[\event]=\emptyset\)}\\
\Adelim{\gambles}{\gambles}
&\text{otherwise}
\end{cases}
\end{equation*}
when \(\M=\M_{\prevof{\smallbolleke\given\smallbolleke}}\in\prevadfMs\).
The consistency condition~\(\M_\des\cap\indifset[\event]\) between~\(\M_{\event}=\Adelim{\indifset[\event]}{\emptyset}\) and~\(\M=\M_{\prevof{\smallbolleke\given\smallbolleke}}\), where we recall from \cref{eq::model:conditional:full:prevision} that \(\M_\des=\gambles_{\stronggt0}\cup\group{\posiof{\mathscr{L}_{\prevof{\smallbolleke\given\smallbolleke}}}+\gambles_{\stronggt0}}\), clearly reduces to \(\group{\posiof{\mathscr{L}_{\prevof{\smallbolleke\given\smallbolleke}}}+\gambles_{\stronggt0}}\cap\indifset[\event]=\emptyset\), which can be rewritten as
\begin{multline}\label{eq::precisional:consistency:condition:for:expansion}
\sup\group[\bigg]{\sum_{k=1}^n\lambda_k\indof{B_k}\group{\gbl_k-\prevof{\gbl_k\given B_k}+\epsilon_k}\Big\vert\event}\geq0,\\
\text{ for all \(n\in\naturals\), \(\lambda_k,\epsilon_k\in\posreals\), \(\gbl_k\in\gambles\) and \(B_k\in\powersetof{\states}\setminus\set{\emptyset}\).}
\end{multline}
\noindent This consistency condition seems complex, but it can be expressed much more directly, and indeed quite elegantly, in terms of the real numbers~\(\prevof{\indof{\event}\given\bolleke}\), so the various conditional probabilities of the conditioning event~\(\event\).

\begin{proposition}\label{prop::conditional:expansion:consistency}
Consider any regular event~\(\event\in\powersetof{\states}\setminus\set{\emptyset}\) and any full previsional AD-model~\(\M_{\prevof{\smallbolleke\given\smallbolleke}}\in\prevadfMs\).
Then the AD-assessment~\(\M_{\prevof{\smallbolleke\given\smallbolleke}}\union\M_{\event}\) is \(\bgM\)-consistent---so the condition in \cref{eq::precisional:consistency:condition:for:expansion} holds---if and only the event~\(\event\) has probability~\(1\) conditional on all regular events~\(B\) compatible with it, meaning that \(\prevof{\indof{\event}\vert B}=1\) if \(B\cap\event\neq\emptyset\).
\end{proposition}
\noindent Observe, by the way, that coherence already implies that \(\prevof{\indof{\event}\vert B}=0\) when \(B\cap\event=\emptyset\) [use \labelcref{axiom:coherence:conditional:bounds}].

Whenever there's consistency between the full previsional AD-model~\(\M_{\prevof{\smallbolleke\given\smallbolleke}}\) and the event~\(\event\), belief expansion yields an AD-model~\(\expandof{\M_{\prevof{\smallbolleke\given\smallbolleke}}\given\event}\) that can also be characterised and interpreted fairly easily.

\begin{proposition}\label{prop::conditional:expansion}
Consider any regular event~\(\event\in\powersetof{\states}\setminus\set\emptyset\) and any full previsional AD-model~\(\M_{\prevof{\smallbolleke\given\smallbolleke}}\in\prevadfMs\) such that the AD-assessment~\(\M_{\prevof{\smallbolleke\given\smallbolleke}}\union\M_{\event}\) is \(\bgM\)-consistent.
Then
\begin{equation*}
\expandof{\M_{\prevof{\smallbolleke\given\smallbolleke}}\given\event}
=\Adelim[\big]{\gambles^{\set\event}_{\weakgeq0}\cup\group[\big]{\posiof{\mathscr{L}_{\prev_{\event}\group{\smallbolleke\condon\smallbolleke}}}+\gambles^{\set\event}_{\weakgeq0}}}{-\group[\big]{\gambles^{\set\event}_{\stronggt0}\cup\group[\big]{\posiof{\mathscr{L}_{\prev_{\event}\group{\smallbolleke\condon\smallbolleke}}}+\gambles^{\set\event}_{\stronggt0}}}}.
\end{equation*}
\end{proposition}

This, finally, allows us to show that when we restrict the revision operator based on our newly proposed conditioning rule for AD-models to full conditional AD-models, it satisfies our version of the AGM axioms~\labelcref{axiom:revision:1,axiom:revision:2,axiom:revision:3,axiom:revision:4,axiom:revision:5,axiom:revision:7,axiom:revision:8}.

\begin{proposition}\label{prop::postulates:satisfied:previsional}
When we restrict the action of the revision operator as defined in \cref{eq::propositional:revision:operator:definition} to full conditional AD-models, it satisfies \labelcref{axiom:revision:1,axiom:revision:2,axiom:revision:3,axiom:revision:4,axiom:revision:5,axiom:revision:7,axiom:revision:8}.
\end{proposition}

\section{Conclusion}\label{sec::discussion}
We believe that a number of interesting conclusions can be drawn from the many arguments in this paper.
First, we see that it's possible to introduce a notion of conditioning on abstract events that works both in classical probabilistic and quantum probabilistic inference, and reduces to Bayes' rule and Lüders's Rule there.
It relies crucially on considering a calling-off operation as a projection on the option space, as well as on the idea that observing an event makes any two options indifferent when they coincide after calling them off.
Our conditioning method can also be applied in various sub-contexts of AD-models, such as sets of desirable and or acceptable gambles, classical propositional logic, and full conditional probabilities, where we've shown it to lead back to the already established conditioning rules.

Second, when we use our conditioning method to come up with a revision operator for AD-models in the context of AGM belief change, we see that the revision satisfies all AGM postulates in those special cases that the postulates were originally suggested for: classical propositional logic and (precise) probability theory.
But we also see that when we leave these ambits, two of the postulates, namely \labelcref{axiom:revision:4,axiom:revision:8}, can no longer be guaranteed to hold in general; these are the rules that essentially state in some way that revision and expansion ought to coincide when there's consistency.

There are two ways out of this conundrum, we believe: we might give up on those two postulates as always compelling, or we might come to the conclusion that our conditioning rule doesn't lead to a reasonable revision method.
We've indicated in the discussion of our counterexamples for \labelcref{axiom:revision:4,axiom:revision:8} that the failure they point to, is associated with a possible loss of resolvedness/precision when conditioning imprecise probability models, and related to the phenomenon of \emph{dilation}, well-known in a context of conditioning imprecise probability models.
It's this possible loss of resolvedness/precision after conditioning, even when the original model and the new information are consistent---so even when expansion leaves the original model unchanged---that results in the failure of \labelcref{axiom:revision:4,axiom:revision:8}.
Those who find this (well-documented) increase of imprecision unreasonable, will not be inclined to begin casting doubt on the relevant AGM axioms; others who are less inclined to view this increase in imprecision after conditioning as problematic, may view its manifestation in this context as a reason to question their general validity.

\section*{Acknowledgements}
We have benefited from discussions with Arthur Van Camp, Erik Quaeghebeur, Jasper De Bock and Catrin Campbell-Moore.

\section*{Author contributions}
This paper is the result of a joint effort in working on the ideas, writing and revising.
The ideas about AD-models for abstract options, abstract events and conditioning AD-models on events are borrowed from a book Gert is currently writing.
The application of these ideas to come up with a specific revision operator for AD-models and the subsequent study of its properties, drew on significant contributions from all three authors, and resulted from nearly continuous discussions between us.
Based on these discussions, Kathelijne wrote the first version of the paper, significantly extending the work on DI-models she did in the context of her master's thesis \cite{coussement2025}, and which were first published in a conference paper \cite{coussement25a}; she also worked out the counterexamples and much of the material in \cref{sec::propositional:models,sec::full:conditional:previsions}.
Gert then worked on a significant revision and extension of the material, and contributed a number of the proofs, and Keano provided crucial feedback and numerous suggestions for improvements, filled out many of the quantum mechanical details, and provided important insights in quite a few of the proofs.

\printbibliography

\appendix
\section{Proofs and counterexamples}
In this Appendix, we have gathered proofs of the statements in the main text, as well as the technical lemmas necessary in those proofs.

\subsection{For \cref{sec::accept:reject}}

\begin{proof}[Proof of \cref{prop::ad:equivalent:criteria}]
It's a matter of direct verification to check that any assessment~\(\M=\Adelim{\M_\acc}{-\M_\des}\) that satisfies \labelcref{axiom:AD:background,axiom:AD:zero:not:desirable,axiom:AD:deductive:closedness,axiom:AD:no:limbo}, also satisfies \labelcref{axiom:stronger:background:respected,axiom:model:null:not:rejected,axiom:model:accepts:convex:cone,axiom:model:no:limbo}.

We therefore concentrate on the proof of the converse implication, and assume that \(\M\) satisfies \labelcref{axiom:stronger:background:respected,axiom:model:null:not:rejected,axiom:model:accepts:convex:cone,axiom:model:no:limbo}.
Clearly, \labelcref{axiom:AD:background} is identical to \labelcref{axiom:stronger:background:respected}, and \labelcref{axiom:AD:zero:not:desirable} is identical to \labelcref{axiom:model:null:not:rejected}.
To see that \(\M\) satisfies \labelcref{axiom:AD:deductive:closedness}, simply observe that
\begin{align*}
\shullof{\M_\des}
&\subseteq\shullof{\M_\des}+\M_\acc
&&\reason{\(0\in\M_\acc\), by \labelcref{axiom:stronger:background:respected}}\\
&\subseteq\M_\des.
&&\reason{\labelcref{axiom:model:no:limbo} and \labelcref{eq::AD:condition}}
\end{align*}
At the same time also
\begin{align*}
\shullof{\M_\des}+\M_\des
&\subseteq\shullof{\M_\des}+\M_\acc
&&\reason{\(\M_\des\subseteq\M_\acc\)}\\
&\subseteq\M_\des.
&&\reason{\labelcref{axiom:model:no:limbo} and \labelcref{eq::AD:condition}}
\end{align*}
This implies that \(\M_\des\) is a convex cone, and so is \(\M_\acc\) by \labelcref{axiom:model:accepts:convex:cone}.
Finally,
\begin{align*}
\M_\des+\M_\acc
&\subseteq\shullof{\M_\des}+\M_\acc\\
&\subseteq\M_\des,
&&\reason{\labelcref{axiom:model:no:limbo} and \labelcref{eq::AD:condition}}
\end{align*}
so \(\M\) satisfies \labelcref{axiom:AD:no:limbo} as well.
\labelcref{axiom:AD:nonpositive:not:acceptable} immediately follows from \labelcref{axiom:AD:background,axiom:AD:zero:not:desirable,axiom:AD:no:limbo}:
\begin{align*}
\M_\acc+\bgM_\des
&\subseteq\M_\acc+\M_\des
&&\reason{\labelcref{axiom:AD:background}}\\
&\subseteq\M_\des\not\ni0.
&&\reason{\labelcref{axiom:AD:no:limbo,axiom:AD:zero:not:desirable}}
\qedhere
\end{align*}
\end{proof}

\subsection{For \cref{sec::events}}

\begin{proof}[Proof of \cref{prop::ordering:equivalent:statements}]
To prove that \labelcref{prop::event:ordering}\(\then\)\labelcref{prop::event:kernels}, assume that \(\eventopt[1]\posetleq\eventopt[2]\), so \(\calledoff[1]{(\calledoff[2]{\opt})}=\calledoff[2]{(\calledoff[1]{\opt})}=\eventopt[1]\ast\opt\) for all~\(\opt\in\opts\).
For any~\(\opt\in\eventindifset[2]\), we have that \(\calledoff[2]{\opt}=0\), so \(\calledoff[1]{(\calledoff[2]{\opt})}=\calledoff[1]{0}=0\) by~\labelcref{axiom:event:linear}, and therefore also \(\eventopt[1]\ast\opt=0\).
Hence, indeed, \(\eventindifset[2]\subseteq\eventindifset[1]\).
That \labelcref{prop::event:kernels}\(\then\)\labelcref{prop::event:nonpositivity} is trivial.
To conclude, we gather from \labelcref{axiom:event:ordering} that \labelcref{prop::event:nonpositivity}\(\then\)\labelcref{prop::event:ordering}.
\end{proof}

\begin{proof}[Proof that \labelcref{axiom:event:ordering} holds in a quantum context]
Consider any two sets of mutually orthogonal subspaces \(\set{\subspace[1],\dots,\subspace[n]}\) and \(\set{\subspace[1]^{\prime},\dots,\subspace[m]^{\prime}}\).
Assume that
\begin{equation}\label{eq::ordering:assumption}
\group{\forall\measurement{A}\in\measurements}
\group[\bigg]{\sum_{\ell=1}^{m}\projector[\ell]^{\prime}\measurement{A}\projector[\ell]^{\prime}=\zero\then\sum_{k=1}^{n}\projector[k]\measurement{A}\projector[k]\not\weakgt\zero}.
\end{equation}
Then we must prove that
\begin{equation*}
\sum_{k=1}^{n}\projector[k]\measurement{A}\projector[k]
=\sum_{k=1}^{n}\projector[k]\group[\bigg]{\sum_{\ell=1}^{m}\projector[\ell]^{\prime}\measurement{A}\projector[\ell]^{\prime}}\projector[k]
=\sum_{\ell=1}^{m}\projector[\ell]^{\prime}\group[\bigg]{\sum_{k=1}^{n}\projector[k]\measurement{A}\projector[k]}\projector[\ell]^{\prime},
\text{ for all~\(\measurement{A}\in\measurements\).}
\end{equation*}
Consider any~\(\fket\in\states\) and let \(\subspace\coloneqq\linspanof{\set{\fket}}\), the linear space spanned by~\(\fket\).
Observe that, by \labelcref{axiom:event:monotone},
\begin{equation}\label{eq::ordering:projection:monotone}
\sum_{k=1}^{n}\projector[k]\projector\projector[k]\weakgeq\zero
\text{ and }
\sum_{\ell=1}^{m}\projector[\ell]^{\prime}\projector\projector[\ell]^{\prime}\weakgeq\zero,
\end{equation}
and additionally,
\begin{equation}\label{eq:ordering:projection:complement}
\group[\bigg]{\sum_{k=1}^{n}\projector[k]\projector\projector[k]=\zero\ifandonlyif\fket\notin\bigcup_{k=1}^{n}\subspace[k]}\text{ and }
\group[\bigg]{\sum_{\ell=1}^{m}\projector[\ell]^{\prime}\projector\projector[\ell]^{\prime}=\zero\ifandonlyif\fket\notin\bigcup_{\ell=1}^{m}\subspace[\ell]^{\prime}}
\end{equation}
Now assume that \(\fket\notin\bigcup_{\ell=1}^{m}\subspace[\ell]^{\prime}\), then we infer from the assumption in \Cref{eq::ordering:assumption} and \Cref{eq::ordering:projection:monotone} that \(\sum_{k=1}^{n}\projector[k]\projector\projector[k]=\zero\) and therefore that \(\fket\notin\bigcup_{k=1}^{n}\subspace[k]\).
This implies that \(\bigcup_{k=1}^{n}\subspace[k]\subseteq\bigcup_{\ell=1}^{m}\subspace[\ell]^{\prime}\),
which concludes the proof.
\end{proof}

\begin{proof}[Proof that \labelcref{axiom:event:kernel:sum} holds in a quantum probability context]
We first look at the case where the two events are singletons of subspaces, which then easily generalises to events containing multiple orthogonal subspaces, since the intersection of such events is element-wise.

Let \(\projectoron{k}\coloneqq\projector[k]\) and \(\projectoron{1,2}\coloneqq\projectoron{\subspace[1]\cap\subspace[2]}\), then it follows trivially from \(\projectoron{1,2}=\projectoron{k}\projectoron{1,2}=\projectoron{1,2}\projectoron{k}\) that \(\smash{\indifset[{\subspace[1]}]}+\smash{\indifset[{\subspace[2]}]}\subseteq\smash{\indifset[{\subspace[1]\cap\subspace[2]}]}\).
So it remains to prove that \(\smash{\indifset[{\subspace[1]\cap\subspace[2]}]}\subseteq\smash{\indifset[{\subspace[1]}]}+\smash{\indifset[{\subspace[2]}]}\).
Consider any~\(\measurement{A}\) such that \(\projectoron{1,2}\measurement{A}\projectoron{1,2}=\zero\), then we must find \(\measurement{B}\) and \(\measurement{C}\) such that \(\projectoron{1}\measurement{B}\projectoron{1}=\zero\), \(\projectoron{2}\measurement{C}\projectoron{2}=\zero\) and \(\measurement{A}=\measurement{B}+\measurement{C}\).

By von Neumann's alternating orthogonal projection result \cite[Theorem 13.7]{neumann1950}, \(\projectoron{1,2}=\lim_{n\to\infty}\group{\projectoron{1}\projectoron{2}}^n=\lim_{n\to\infty}\group{\projectoron{2}\projectoron{1}}^n\).
Let
\begin{equation*}
\measurement{B}_k
\coloneqq\group{\projectoron{2}\projectoron{1}}^k\projectoron{2}\measurement{A}\projectoron{2}\group{\projectoron{1}\projectoron{2}}^k
-\group{\projectoron{1}\projectoron{2}}^{k+1}\measurement{A}\group{\projectoron{2}\projectoron{1}}^{k+1},
\end{equation*}
then \(\projectoron{1}\measurement{B}_k\projectoron{1}=\zero\).
Since \(\measurement{B}'_n\coloneqq\sum_{k=0}^n\measurement{B}_k\) converges to some~\(\measurement{B}\) in \(\measurements\) [essentially because \(\projectoron{1}\projectoron{2}\) and \(\projectoron{2}\projectoron{1}\) can be written as orthogonal sums of an identity on~\(\subspace[1]\cap\subspace[2]\) and a strict contraction on its orthogonal complement; see \cite[Thm.~3.1]{kuo1989:factorisation}], we find that \(\projectoron{1}\measurement{B}\projectoron{1}=\zero\).
We are done if we can prove that \(\projectoron{2}\group{\measurement{A}-\measurement{B}}\projectoron{2}=\zero\).
Check that
\begin{equation*}
\projectoron{2}\group{\measurement{A}-\measurement{B}'_n}\projectoron{2}
=\group{\projectoron{2}\projectoron{1}}^{n+1}\projectoron{2}\measurement{A}\projectoron{2}\group{\projectoron{1}\projectoron{2}}^{n+1}
\to\projectoron{1,2}\projectoron{2}\measurement{A}\projectoron{2}\projectoron{1,2},
\end{equation*}
and therefore, indeed, \(\projectoron{2}\group{\measurement{A}-\measurement{B}}\projectoron{2}=\projectoron{1,2}\projectoron{2}\measurement{A}\projectoron{2}\projectoron{1,2}=\projectoron{2}\projectoron{1,2}\measurement{A}\projectoron{1,2}\projectoron{2}=\zero\).
\end{proof}

\subsection{For \cref{sec::conditioning}}

\begin{proof}[Proof of the consistency condition in \cref{eq::regularity}]
Starting from the consistency condition~\labelcref{eq::consistency}, we have the following chain of equivalences:
\begin{multline*}
\posiof{\bgM_\acc\cup\eventindifset}\cap-\group{\bgM_\des\cup\emptyset}=\emptyset\\
\begin{aligned}
&\ifandonlyif\group{\bgM_\acc\cup\eventindifset\cup\group{\bgM_\acc+\eventindifset}}\cap-\bgM_\des=\emptyset\\
&\ifandonlyif\group{\bgM_\acc+\eventindifset}\cap-\bgM_\des=\emptyset
&&\reason{\(0\in\bgM_\acc\) and \(0\in\eventindifset\)}\\
&\ifandonlyif0\notin\bgM_\des+\bgM_\acc+\eventindifset\\
&\ifandonlyif0\notin\bgM_\des+\eventindifset
&&\reason{\(\bgM_\des+\bgM_\acc=\bgM_\des\) by \labelcref{axiom:AD:no:limbo}}\\
&\ifandonlyif\bgM_\des\cap\eventindifset=\emptyset,
&&\reason{\(\eventindifset=-\eventindifset\)}
\end{aligned}
\end{multline*}
which completes the proof.
\end{proof}

\begin{proof}[Proof of \cref{prop::regular:events}]
For the direct implication, assume that \(\eventopt\) is regular, so \(\bgM_\des\cap\eventindifset=\emptyset\).
Consider any~\(\opt\in\opts\) such that \(\opt\stronggt0\), then also \(\opt\weakgeq0\) and therefore also \(\eventopt\ast\opt\weakgeq0\), by \labelcref{axiom:event:monotone}.
But, because of the regularity, we also know that \(\calledoff{\opt}\neq0\), and therefore \(\calledoff{\opt}>0\).
Conversely, assume that \(\opt\stronggt0\then\calledoff{\opt}>0\) for all~\(\opt\in\opts\), then the condition~\(\bgM_\des\cap\eventindifset=\emptyset\) for regularity is trivially satisfied.
\end{proof}

\begin{proof}[Proof that the assessment~\labelcref{eq::naive:updated:model} satisfies~\labelcref{axiom:AD:zero:not:desirable,axiom:AD:deductive:closedness,axiom:AD:no:limbo}]
For \labelcref{axiom:AD:zero:not:desirable}, assume that \(0\in\M_\des\altcondon\eventopt\), then \(0=\eventopt\ast0\in\M_\des\), a contradiction.
For \labelcref{axiom:AD:deductive:closedness}, let's show that \(\M_\acc\altcondon\eventopt\) is a convex cone; the proof for \(\M_\des\altcondon\eventopt\) is completely similar.
So consider any~\(\opt,\altopt\in\M_\acc\altcondon\eventopt\) and any non-negative~\(\lambda,\mu\in\reals\) with~\(\lambda+\mu>0\).
Then \(\eventopt\ast\opt\in\M_\acc\) and \(\eventopt\ast\altopt\in\M_\acc\), and therefore also
\begin{equation*}
\eventopt\ast\group{\lambda\opt+\mu\altopt}
=\lambda\eventopt\ast\opt+\mu\eventopt\ast\altopt
\in\M_\acc,
\end{equation*}
which tells us that, indeed, \(\lambda\opt+\mu\altopt\in\M_\acc\altcondon\eventopt\).
Finally, that \labelcref{axiom:AD:no:limbo} is satisfied follows at once from \cref{lem:compatibility:conditional:sets} with~\(\someopts_1\instantiateas\M_\des\) and \(\someopts_2\instantiateas\M_\acc\).
\end{proof}

\begin{lemma}\label{lem:compatibility:conditional:sets}
If two sets of options~\(\someopts_1,\someopts_2\subseteq\opts\) satisfy \(\someopts_1+\someopts_2\subseteq\someopts_1\), then also
\begin{equation*}
\someopts_1\altcondon\eventopt+\someopts_2\altcondon\eventopt
\subseteq\someopts_1\altcondon\eventopt,
\end{equation*}
with equality if \(0\in\someopts_2\).
\end{lemma}

\begin{proof}[Proof of \cref{lem:compatibility:conditional:sets}]
Consider any~\(\opt\in\someopts_1\altcondon\eventopt\) and~\(\altopt\in\someopts_2\altcondon\eventopt\), so \(\calledoff{\opt}\in\someopts_1\) and \(\calledoff{\altopt}\in\someopts_2\).
Clearly, then, \(\calledoff{(\opt+\altopt)}=\calledoff{\opt}+\calledoff{\altopt}\in\someopts_1+\someopts_2\subseteq\someopts_1\), so \(\opt+\altopt\in\someopts_1\altcondon\eventopt\).
If, moreover, \(0\in\someopts_2\), then \(\calledoff{0}=0\in\someopts_2\), so \(0\in\someopts_2\altcondon\eventopt\), and then also \(\someopts_1\altcondon\eventopt\subseteq\someopts_1\altcondon\eventopt+\someopts_2\altcondon\eventopt\).
\end{proof}

\begin{proof}[Proof of \cref{prop::consistency:of:updated:model}]
Starting from the consistency condition~\labelcref{eq::consistency}, we find the following chain of equivalences:
\begin{multline*}
\posiof{\M_\acc\altcondon\eventopt\cup\bgM_\acc}\cap-\group{\M_\des\altcondon\eventopt\cup\bgM_\des}=\emptyset\\
\begin{aligned}
&\ifandonlyif0\notin\posiof{\M_\acc\altcondon\eventopt\cup\bgM_\acc}+\group{\M_\des\altcondon\eventopt\cup\bgM_\des}\\
&\ifandonlyif0\notin\group{\M_\acc\altcondon\eventopt\cup\bgM_\acc\cup\group{\M_\acc\altcondon\eventopt+\bgM_\acc}}+\group{\M_\des\altcondon\eventopt\cup\bgM_\des}\\
&\ifandonlyif0\notin\group{\M_\acc\altcondon\eventopt+\bgM_\acc}+\group{\M_\des\altcondon\eventopt\cup\bgM_\des}
&&\reason{\(0\in\M_\acc\altcondon\eventopt\) and \(0\in\bgM_\acc\)}\\
&\ifandonlyif0\notin\group{\M_\acc\altcondon\eventopt+\bgM_\acc+\M_\des\altcondon\eventopt}\cup\group{\M_\acc\altcondon\eventopt+\bgM_\acc+\bgM_\des}\\
&\ifandonlyif0\notin\group{\M_\acc\altcondon\eventopt+\M_\des\altcondon\eventopt+\bgM_\acc}\cup\group{\M_\acc\altcondon\eventopt+\bgM_\des}
&&\reason{\(\bgM_\acc+\bgM_\des=\bgM_\des\) by \labelcref{axiom:AD:no:limbo}}\\
&\ifandonlyif0\notin\group{\M_\acc\altcondon\eventopt+\M_\des\altcondon\eventopt}\cup\group{\M_\acc\altcondon\eventopt+\bgM_\des}
&&\reason{Lem.~\labelcref{lem:adding:background:to:conditional:model}}\\
&\ifandonlyif0\notin\M_\des\altcondon\eventopt\cup\group{\M_\acc\altcondon\eventopt+\bgM_\des}\\
&\ifandonlyif0\notin\M_\acc\altcondon\eventopt+\bgM_\des,
\end{aligned}
\end{multline*}
where the penultimate equivalence follows from \cref{lem:compatibility:conditional:sets} with~\(\someopts_1\instantiateas\M_\des\altcondon\eventopt\) and \(\someopts_2\instantiateas\M_\acc\altcondon\eventopt\) and \(0\in\M_\acc\); the last equivalence follows from the fact that \(0\notin\M_\des\) [by \labelcref{axiom:AD:zero:not:desirable}] and therefore also \(0\notin\M_\des\altcondon\eventopt\) [by \labelcref{axiom:event:linear}].
\end{proof}

\begin{lemma}\label{lem:adding:background:to:conditional:model}
Consider any~\(\M\in\adfMsabove{\bgM}\) and any event~\(\eventopt\in\eventopts\), then \(\M_\acc\altcondon\eventopt+\bgM_\acc=\M_\acc\altcondon\eventopt\).
\end{lemma}

\begin{proof}[Proof of \cref{lem:adding:background:to:conditional:model}]
Since \(0\in\bgM_\acc\) [by \labelcref{axiom:AD:background}], it's clear that \(\M_\acc\altcondon\eventopt\subseteq\M_\acc\altcondon\eventopt+\bgM_\acc\), so we concentrate on the converse inclusion.
Let \(\opt\in\M_\acc\altcondon\eventopt+\bgM_\acc\), so there are \(\altopt\in\M_\acc\altcondon\eventopt\) and \(\altopttoo\in\bgM_\acc\) such that \(\opt=\altopt+\altopttoo\).
But then \(\eventopt\ast\opt=\eventopt\ast\group{\altopt+\altopttoo}=\eventopt\ast\altopt+\eventopt\ast\altopttoo\).
Since \(\eventopt\ast\altopt\in\M_\acc\) and \(\eventopt\ast\altopttoo\in\bgM_\acc\) [by \labelcref{axiom:event:monotone}], we see that \(\eventopt\ast\opt\in\M_\acc+\bgM_\acc\subseteq\M_\acc\), where the inclusion follows from the fact that \(\bgM_\acc\subseteq\M_\acc\) [by \labelcref{axiom:AD:background}] and the convex cone character of~\(\M_\acc\) [use \labelcref{axiom:AD:deductive:closedness}].
Hence, indeed, \(\opt\in\M_\acc\altcondon\eventopt\).
\end{proof}

\begin{proof}[Proof of \cref{prop::conditional:model}]
Note that both \(\M_\acc\altcondon\eventopt\) and \(\bgM_\acc\) are convex cones [for \(\M_\acc\altcondon\eventopt\), use \labelcref{axiom:AD:deductive:closedness,axiom:event:linear}; for \(\bgM_\acc\), use \labelcref{axiom:AD:deductive:closedness}], so
\begin{align}
\posiof{\M_\acc\altcondon\eventopt\cup\bgM_\acc}
&=\posiof{\M_\acc\altcondon\eventopt}\cup\posiof{\bgM_\acc}\cup\group{\posiof{\M_\acc\altcondon\eventopt}+\posiof{\bgM_\acc}}
\notag\\
&=\M_\acc\altcondon\eventopt\cup\bgM_\acc\cup\group{\M_\acc\altcondon\eventopt+\bgM_\acc}
\notag\\
&=\M_\acc\altcondon\eventopt+\bgM_\acc
\notag\\
&=\M_\acc\altcondon\eventopt,
\label{eq::conditional:model:helper:1}
\end{align}
where the penultimate equality follows from the fact that \(0\in\M_\acc\altcondon\eventopt\) [by \labelcref{axiom:AD:background,axiom:event:linear}] and \(0\in\bgM_\acc\) [by \labelcref{axiom:AD:background}]; and where the last equality follows from \cref{lem:adding:background:to:conditional:model}.
Since \(\M_\des\altcondon\eventopt\) and \(\bgM_\des\) are convex cones as well [for \(\M_\des\altcondon\eventopt\), use \labelcref{axiom:AD:deductive:closedness,axiom:event:linear}; for \(\bgM_\des\), use \labelcref{axiom:AD:deductive:closedness}], we also find that
\begin{equation}\label{eq::conditional:model:helper:2}
\shullof{\M_\des\altcondon\eventopt\cup\bgM_\des}
=\shullof{\M_\des\altcondon\eventopt}\cup\shullof{\bgM_\des}
=\M_\des\altcondon\eventopt\cup\bgM_\des.
\end{equation}
Combining \cref{eq::conditional:model:helper:1,eq::conditional:model:helper:2}, we find that
\begin{multline*}
\shullof{\M_\des\altcondon\eventopt\cup\bgM_\des}\cup\group{\shullof{\M_\des\altcondon\eventopt\cup\bgM_\des}+\posiof{\M_\acc\altcondon\eventopt\cup\bgM_\acc}}\\
\begin{aligned}
&=\M_\des\altcondon\eventopt\cup\bgM_\des\cup\group{\group{\M_\des\altcondon\eventopt\cup\bgM_\des}+\M_\acc\altcondon\eventopt}\\
&=\M_\des\altcondon\eventopt\cup\bgM_\des\cup\group{\M_\des\altcondon\eventopt+\M_\acc\altcondon\eventopt}\cup\group{\bgM_\des+\M_\acc\altcondon\eventopt}\\
&=\M_\des\altcondon\eventopt\cup\bgM_\des\cup\M_\des\altcondon\eventopt\cup\group{\bgM_\des+\M_\acc\altcondon\eventopt}
&&\reason{Lem.~\labelcref{lem:compatibility:conditional:sets}}\\
&=\M_\des\altcondon\eventopt\cup\group{\bgM_\des+\M_\acc\altcondon\eventopt},
&&\reason{\(0\in\M_\acc\altcondon\eventopt\)}
\end{aligned}
\end{multline*}
which completes the proof.
\end{proof}

\subsection{For \cref{sec::conditioning:as:belief:change}}

\begin{proof}[Proof of \cref{prop::AD:model:revision:axioms:obeyed}]
For \labelcref{axiom:revision:1}, we need to prove that \(\reviseof{\M\given\eventopt}\in\adfMsabove{\bgM}\cup\set{\Adelim{\opts}{\opts}}\).
When the event isn't \(\bgM\)-regular, this is trivially the case.
When the event is \(\bgM\)-regular but the AD-model \(\M\) isn't conditionable on it, \(\reviseof{\M\given\eventopt}=\Mclsof{\bgM\union\eventM}=\Adelim{\bgM_\acc+\eventindifset}{-\group{\bgM_\des+\eventindifset}}\) is an AD-model that respects~\(\bgM\).\footnote{This is discussed in more detail in \cref{sec::belief:expansion}.}
Lastly, when the event is \(\bgM\)-regular and the AD-model \(\M\) is conditionable on it, we gather from the discussion in \cref{sec::conditioning:admodels}, and in particular from \cref{prop::conditional:model}, that the conditional~\(\M\condon\eventopt\) is an AD-model that respects~\(\bgM\).

\labelcref{axiom:revision:2} is obeyed because \(\eventindifset\subseteq\M_\acc\altcondon\eventopt\) and \(\eventM\subseteq\Mclsof{\bgM\union\eventM}\).

For \labelcref{axiom:revision:3}, taking into account the respective expressions~\labelcref{eq::AD:expansion:operator,eq::AD:revision:operator:rewritten} for the expansion and revision operators, we see that we can assume without loss of generality that \(\eventopt\) is \(\bgM\)-regular and that \(\M\) is conditionable on it.
We then see that
\begin{equation*}
\opt=\calledoff{\opt}+(\opt-\calledoff{\opt})\in\M_\acc+\eventindifset
\text{ for all~\(\opt\in\M_\acc\altcondon\eventopt\)},
\end{equation*}
so \(\M_\acc\altcondon\eventopt\subseteq\M_\acc+\eventindifset\),
and similarly \( \M_\des\altcondon\eventopt\subseteq\M_\des+\eventindifset.\)
Also, consider any~\(\opt\in\bgM_\des\) and any~\(\altopt\in\M_\acc\altcondon\eventopt\), then
\begin{equation*}
\opt+\altopt
=\opt+\calledoff{\altopt}+(\altopt-\calledoff{\altopt})
\in\bgM_\des+\M_\acc+\eventindifset
\subseteq\M_\des+\M_\acc+\eventindifset
=\M_\des+\eventindifset,
\end{equation*}
where we used \labelcref{axiom:AD:background} for the inclusion and \labelcref{axiom:AD:no:limbo} for the equality.
This tells us that also \(\bgM_\des+\M_\acc\altcondon\eventopt\subseteq\M_\des+\eventindifset\), and therefore
\begin{equation*}
\reviseof{\M\given\eventopt}
=\Adelim{\M_\acc\altcondon\eventopt}{-\group{\M_\des\altcondon\eventopt\cup\group{\bgM_\des+\M_\acc\altcondon\eventopt}}}
\subseteq\Adelim{\M_\acc+\eventindifset}{-\group{\M_\des+\eventindifset}}
=\expandof{\M\given\eventopt},
\end{equation*}
so \labelcref{axiom:revision:3} is indeed satisfied.

For \labelcref{axiom:revision:5}, observe from the expression~\labelcref{eq::AD:revision:operator:rewritten} that \(\reviseof{\M\given\eventopt}\) is only \(\bgM\)-inconsistent when \(0\in\bgM_\des+\eventindifset\), or in other words when \(\eventM\) is \(\bgM\)-inconsistent, proving the axiom.

And, to conclude, for \labelcref{axiom:revision:7} we must prove that \(\reviseof{\M\given\eventopt[1,2]}\subseteq\expandof{\reviseof{\M\given\eventopt[1]}\given\eventopt[2]}\), where we let, for ease of notation, \(\eventopt[1,2]\coloneqq\eventopt[1]\sqcap\eventopt[2]\).
Given the complicated nature of these operators, we'll start off with a case study.

For the revision operator on the left-hand side, we simply infer from \cref{eq::AD:revision:operator:rewritten} with \(\eventopt\instantiateas\eventopt[1,2]\) that
\begin{multline}\label{eq:BR7:revision}
\reviseof{\M\given\eventopt[1,2]}\\
\begin{aligned}
&=
\begin{cases}
\Adelim{\M_\acc\altcondon{\eventopt[1,2]}}{-\group{\M_\des\altcondon{\eventopt[1,2]}\cup(\M_\acc\altcondon{\eventopt[1,2]}+\bgM_\des)}}
&\text{if \(0\notin\M_\acc\altcondon{\eventopt[1,2]}+\bgM_\des\)}\\
\Adelim{\bgM_{\acc}+\indifset[{\eventopt[1,2]}]}{-\group{\bgM_\des+\indifset[{\eventopt[1,2]}]}}
&\text{if }0\in\M_\acc\altcondon{\eventopt[1,2]}+\bgM_\des\\
&\quad\text{and }0\notin\bgM_\des+\indifset[{\eventopt[1,2]}]\\
\Adelim{\opts}{\opts}
&\text{otherwise.}
\end{cases}
\end{aligned}
\end{multline}

The expansion~\(\expandof{\reviseof{\M\given\eventopt[1]}\given\eventopt[2]}\) on the right-hand side is more complicated.
There are three cases to consider.
The first is that \(\eventopt[1]\) isn't \(\bgM\)-regular, and then the revision yields \(\Adelim{\opts}{\opts}\), and expanding it trivially yields \(\Adelim{\opts}{\opts}\) as well.
The second possibility is that \(\eventopt[1]\) is \(\bgM\)-regular, but \(\M\) is \emph{not} conditionable on it, so \(0\in\M_\acc\altcondon{\eventopt[1]}+\bgM_\des\) and \(0\notin\bgM_\des+\indifset[{\eventopt[1]}]\).
The revision then yields \(\reviseof{\M\given\eventopt[1]}=\Adelim{\bgM_{\acc}+\indifset[{\eventopt[1]}]}{-\group{\bgM_\des+\indifset[{\eventopt[1]}]}}\).
For the subsequent expansion, we then have
\begin{equation*}
\expandof{\reviseof{\M\given\eventopt[1]}\given\eventopt[2]}
=
\begin{cases}
\Adelim{\bgM_{\acc}+\indifset[{\eventopt[1]}]+\indifset[{\eventopt[2]}]}{-\group{\bgM_\des+\indifset[{\eventopt[1]}]+\indifset[{\eventopt[2]}]}}
&\text{if }0\notin\bgM_\des+\indifset[{\eventopt[1]}]+\indifset[{\eventopt[2]}]\\
\Adelim{\opts}{\opts}
&\text{otherwise}.
\end{cases}
\end{equation*}
The third and last possibility is that \(\M\) is conditionable on~\(\eventopt[1]\), so \(0\notin\M_\acc\altcondon{\eventopt[1]}+\bgM_\des\), and then we have that \(\reviseof{\M\given\eventopt[1]}=\Adelim{\M_\acc\altcondon{\eventopt[1]}}{-\group{\M_\des\altcondon{\eventopt[1]}\cup(\M_\acc\altcondon{\eventopt[1]}+\bgM_\des)}}\).
The subsequent expansion will only give an AD-model when \(0\notin\M_\des\altcondon{\eventopt[1]}+\indifset[{\eventopt[2]}]\) and \(0\notin\M_\acc\altcondon{\eventopt[1]}+\bgM_\des+\indifset[{\eventopt[2]}]\), so
\begin{equation*}
\begin{aligned}
&\expandof{\reviseof{\M\given\eventopt[1]}\given\eventopt[2]}\\
&=
\begin{cases}
\Adelim{\M_\acc\altcondon{\eventopt[1]}+\indifset[{\eventopt[2]}]}{-\group{(\M_\des\altcondon{\eventopt[1]}+\indifset[{\eventopt[2]}])\cup(\M_\acc\altcondon{\eventopt[1]}+\bgM_\des+\indifset[{\eventopt[2]}])}}
&\text{if }0\notin\M_\acc\altcondon{\eventopt[1]}+\bgM_\des+\indifset[{\eventopt[2]}]\\
&\quad\text{and }0\notin\M_\des\altcondon{\eventopt[1]}+\indifset[{\eventopt[2]}]\\
\Adelim{\opts}{\opts}&\text{otherwise}.
\end{cases}
\end{aligned}
\end{equation*}%
Summarising, we find that
\begin{equation}\label{eq:BR7:expansion}
\begin{aligned}
&\expandof{\reviseof{\M\given\eventopt[1]}\given\eventopt[2]}\\
&=
\begin{cases}
\Adelim{\M_\acc\altcondon{\eventopt[1]}+\indifset[{\eventopt[2]}]}{-\group{(\M_\des\altcondon{\eventopt[1]}+\indifset[{\eventopt[2]}])\cup(\M_\acc\altcondon{\eventopt[1]}+\bgM_\des+\indifset[{\eventopt[2]}])}}
&\text{if }0\notin\M_\acc\altcondon{\eventopt[1]}+\bgM_\des+\indifset[{\eventopt[2]}]\\
&\quad\text{and }0\notin\M_\des\altcondon{\eventopt[1]}+\indifset[{\eventopt[2]}]\\
&\quad\text{and }0\notin\M_\acc\altcondon{\eventopt[1]}+\bgM_\des\\
\Adelim{\bgM_{\acc}+\indifset[{\eventopt[1]}]+\indifset[{\eventopt[2]}]}{-[\bgM_\des+\indifset[{\eventopt[1]}]+\indifset[{\eventopt[2]}]]}
&\text{if }0\notin\bgM_\des+\indifset[{\eventopt[1]}]+\indifset[{\eventopt[2]}]\\
&\quad\text{and }0\in\M_\acc\altcondon{\eventopt[1]}+\bgM_\des\\
&\quad\text{and }0\notin\bgM_\des+\indifset[{\eventopt[1]}]\\
\Adelim{\opts}{\opts}
&\text{otherwise}.
\end{cases}
\end{aligned}
\end{equation}

We'll now relate \cref{eq:BR7:revision,eq:BR7:expansion} to each other, where we go through each of the possible cases in \cref{eq:BR7:revision}, and match each of them with the possible cases in \cref{eq:BR7:expansion}.

We begin with the first case in \cref{eq:BR7:revision}, where \(\M\) is conditionable on~\(\eventopt[1,2]\).
For the first possibility of the expansion, we need to prove for the accept part that
\begin{align*}
\M_\acc\altcondon{\eventopt[1,2]}
&\subseteq\M_\acc\altcondon{\eventopt[1]}+\indifset[{\eventopt[2]}],\\
\shortintertext{and for the desirability part that}
\M_\des\altcondon{\eventopt[1,2]}\cup(\M_\acc\altcondon{\eventopt[1,2]}+\bgM_\des)
&\subseteq(\M_\des\altcondon{\eventopt[1]}+\indifset[{\eventopt[2]}])\cup(\M_\acc\altcondon{\eventopt[1]}+\bgM_\des+\indifset[{\eventopt[2]}]).
\end{align*}
We start with the accept part.
Consider any~\(\opt\in\M_\acc\altcondon{\eventopt[1,2]}\), so \(\eventopt[1,2]\ast\opt\in\M_\acc\).
Since \(\eventindifset[1]\subseteq\eventindifset[1]+\eventindifset[2]=\eventindifset[1,2]\), we infer from \cref{prop::ordering:equivalent:statements} that \(\eventopt[1,2]\posetleq\eventopt[1]\), and therefore \(\eventopt[1]\ast\group{\eventopt[1,2]\ast\opt}=\eventopt[1,2]\ast\opt\in\M_\acc\), so \(\eventopt[1,2]\ast\opt\in\M_\acc\altcondon{\eventopt[1]}\).
Hence, \(\opt=\eventopt[1,2]\ast\opt+\group{\opt-\eventopt[1,2]\ast\opt}\in\M_\acc\altcondon{\eventopt[1]}+\eventindifset[1,2]\), and \(\M_\acc\altcondon{\eventopt[1]}+\eventindifset[1,2]=\M_\acc\altcondon{\eventopt[1]}+\eventindifset[1]+\eventindifset[2]\subseteq\M_\acc\altcondon{\eventopt[1]}+\eventindifset[2]\),
where we used \cref{lem:indifset:addition} in the last step.
We therefore find that, indeed, \(\M_\acc\altcondon{\eventopt[1,2]}\subseteq\M_\acc\altcondon{\eventopt[1]}+\eventindifset[2]\).
This inclusion also implies that the inclusion~\(\M_\acc\altcondon{\eventopt[1,2]}+\bgM_\des\subseteq\M_\acc\altcondon{\eventopt[1]}+\eventindifset[2]+\bgM_\des\) holds, and a completely analogous argument shows that \(\M_\des\altcondon{\eventopt[1,2]}\subseteq\M_\des\altcondon{\eventopt[1]}+\eventindifset[2]\), which takes care of the desirability part.

For the second possibility of the expansion, we have that \(\M\) is conditionable on~\(\eventopt[1,2]\) but not on~\(\eventopt[1]\); we'll now prove that this is impossible.
That \(\M\) isn't conditionable on~\(\eventopt[1]\) means that there's some~\(\opt\in\bgM_\des\) such that \(\eventopt[1]\ast\group{-\opt}\in\M_\acc\).
We then infer from \labelcref{axiom:background:order:unit:connection} that there's some~\(\alpha\in\posreals\) such that \(\opt\weakgeq\alpha\optunit\), and therefore also \(\eventopt[1]\ast\group{-\opt}\weakleq-\alpha\eventopt[1]\ast\optunit\), by \labelcref{axiom:event:linear,axiom:event:monotone}.
Since also \(\eventopt[1]\ast\group{-\opt}\in\M_\acc\), we infer from \labelcref{axiom:AD:background,axiom:AD:deductive:closedness} that \(-\alpha\calledoff[1]{\optunit}\in\M_\acc\) as well.
\labelcref{axiom:event:unit:option} now sets off the following sequence of implications:
\begin{align*}
\optunit\weakgeq\eventopt[1,2]\ast\optunit
&\then\calledoff[1]{\optunit}\weakgeq\eventopt[1]\ast\group{\eventopt[1,2]\ast\optunit}
&&\reason{\labelcref{axiom:event:monotone}}\\
&\then\calledoff[1]{\optunit}\weakgeq\eventopt[1,2]\ast\optunit
&&\reason{\(\eventopt[1,2]\posetleq\eventopt[1]\)}\\
&\then-\alpha\calledoff[1]{\optunit}\weakleq-\alpha\eventopt[1,2]\ast\optunit\\
&\then-\alpha\eventopt[1]\ast\optunit\weakleq\eventopt[1,2]\ast\group{-\alpha\optunit}
&&\reason{\labelcref{axiom:event:linear}}\\
&\then\eventopt[1,2]\ast\group{-\alpha\optunit}\in\M_\acc,
&&\reason{\(-\alpha\calledoff[1]{\optunit}\in\M_\acc\)}
\end{align*}
so \(-\alpha\optunit\in\M_\acc\altcondon{\eventopt[1,2]}\cap-\bgM_\des\), contradicting the assumption that \(\M_\acc\) is conditionable on~\(\eventopt[1,2]\).

For the third possibility of the expansion, \(\expandof{\reviseof{\M\given\eventopt[1]}\given\eventopt[2]}=\Adelim{\opts}{\opts}\), so the desired inclusion holds trivially,

Now we turn to the second case in \cref{eq:BR7:revision}, where \(\M\) isn't conditionable on~\(\eventopt[1,2]\) but \(\eventopt[1,2]\) is still regular.
We argued above that \(\M_\acc\altcondon{\eventopt[1,2]}\subseteq\M_\acc\altcondon{\eventopt[1]}+\eventindifset[2]\), so \(0\in\M_\acc\altcondon{\eventopt[1,2]}+\bgM_\des\) implies that also \(0\in\M_\acc\altcondon{\eventopt[1]}+\bgM_\des+\eventindifset[2]\).
In addition, \(0\notin\bgM_\des+\eventindifset[1,2]\) implies that \(0\notin\bgM_\des+\eventindifset[1]+\eventindifset[2]\), and therefore also \(0\notin\bgM_\des+\eventindifset[1]\), because \(\eventindifset[1]\subseteq\eventindifset[1]+\eventindifset[2]=\eventindifset[1,2]\).
This implies that we've landed in the second possibility in \cref{eq:BR7:expansion}, so
\begin{equation*}
\reviseof{\M\given\eventopt[1,2]}
=\Adelim{\bgM_{\acc}+\eventindifset[1]+\eventindifset[2]}{-\group{\bgM_\des+\eventindifset[1]+\eventindifset[2]}}
=\expandof{\reviseof{\M\given\eventopt[1]}\given\eventopt[2]}.
\end{equation*}

Finally, we turn to the third case in \cref{eq:BR7:revision}, where \(\eventopt[1,2]\) isn't regular, meaning that \(0\in\bgM_\des+\eventindifset[1,2]\).
Since \(\eventindifset[1]\subseteq\M_\acc\altcondon{\eventopt[1]}\) and therefore \(\bgM_\des+\eventindifset[1,2]\subseteq\bgM_\des+\M_\acc\altcondon{\eventopt[1]}+\eventindifset[2]\), this implies that \(0\in\bgM_\des+\M_\acc\altcondon{\eventopt[1]}+\eventindifset[2]\) as well, so we find the trivial equality \(\reviseof{\M\given\eventopt[1,2]}=\Adelim{\opts}{\opts}=\expandof{\reviseof{\M\given\eventopt[1]}\given\eventopt[2]}\).
\end{proof}

\begin{lemma}\label{lem:indifset:addition}
For any~\(\someopts\subseteq\opts\), we have that \(\someopts\altcondon\eventopt+\eventindifset=\someopts\altcondon\eventopt\).
\end{lemma}

\begin{proof}[Proof of \cref{lem:indifset:addition}]
Consider any~\(\opt\in\someopts\altcondon\eventopt\) and~\(\altopt\in\eventindifset\), then
\begin{equation*}
\calledoff{(\opt+\altopt)}
=\calledoff{\opt}+\calledoff{\altopt}
=\calledoff{\opt}
\in\someopts,
\end{equation*}
so \(\someopts\altcondon\eventopt+\eventindifset\subseteq\someopts\altcondon\eventopt\).
Conversely, consider any~\(\opt\in\someopts\altcondon\eventopt\), then \(\opt=\calledoff{\opt}+(\opt- \calledoff{\opt})\), where \(\calledoff{\opt}\in\someopts\altcondon\eventopt\) and \(\opt-\calledoff{\opt}\in\eventindifset\), both by \labelcref{axiom:event:idempotent}, so \(\someopts\altcondon\eventopt\subseteq\someopts\altcondon\eventopt+\eventindifset\).
\end{proof}

\subsection{For \cref{sec::propositional:models}}

\begin{proof}[Proof of \cref{prop::form:propositional:model}]
First, observe that
\begin{align}
&\posiof{\indifset[\props]\cup\gambles_{\weakgeq0}}
\notag\\
&\quad=\posiof{\indifset[\props]}\cup\posiof{\gambles_{\weakgeq0}}\cup\group{\posiof{\indifset[\props]}+\posiof{\gambles_{\weakgeq0}}}
\notag\\
&\quad=\posiof{\indifset[\props]}\cup\gambles_{\weakgeq0}\cup\group{\posiof{\indifset[\props]}+\gambles_{\weakgeq0}}
\notag\\
&\quad=\gambles_{\weakgeq0}\cup\group{\posiof{\indifset[\props]}+\gambles_{\weakgeq0}}
&&\reason{\(0\in\gambles_{\weakgeq0}\)}
\label{eq::form:propositional:model:third}
\end{align}
We then find the following equivalences for the \(\bgM\)-consistency of the assessment~\(\A_\props\):
\begin{multline*}
\posiof{\indifset[\props]\cup\gambles_{\weakgeq0}}\cap-(\emptyset\cup\gambles_{\stronggt0}) =\emptyset\\
\begin{aligned}
&\ifandonlyif0\notin\posiof{\indifset[\props]\cup\gambles_{\weakgeq0}}+\gambles_{\stronggt0}\\
&\ifandonlyif0\notin\group{\gambles_{\weakgeq0}\cup\group{\posiof{\indifset[\props]}+\gambles_{\weakgeq0}}}+\gambles_{\stronggt0}
&&\reason{Eq.~\labelcref{eq::form:propositional:model:third}}\\
&\ifandonlyif0\notin\group{\gambles_{\weakgeq0}+\gambles_{\stronggt0}}\cup\group{\posiof{\indifset[\props]}+\gambles_{\weakgeq0}+\gambles_{\stronggt0}}\\
&\ifandonlyif0\notin\gambles_{\stronggt0}\cup\group{\posiof{\indifset[\props]}+\gambles_{\stronggt0}}
&&\reason{\(\gambles_{\weakgeq0}+\gambles_{\stronggt0}=\gambles_{\stronggt0}\)}\\
&\ifandonlyif0\notin\posiof{\indifset[\props]}+\gambles_{\stronggt0}.
&&\reason{\(0\notin\gambles_{\stronggt0}\)}
\end{aligned}
\end{multline*}
We now show that the \(\bgM\)-consistency condition~\(0\notin\posiof{\indifset[\props]}+\gambles_{\stronggt0}\) is equivalent to the logical consistency of the set of propositions~\(\props\), or in other words, to the condition~\(\emptyset\notin\filterbase[\props]\), by looking at the following chain of equivalences:
\begin{multline*}
0\notin\posiof{\indifset[\props]}+\gambles_{\stronggt0}\\
\begin{aligned}
&\ifandonlyif\group{\forall\gbl\in\posiof{\indifset[\props]}}\,-\gbl\notin\gambles_{\stronggt0}\\
&\ifandonlyif\group{\forall\gbl\in\posiof{\indifset[\props]}}\,\sup\gbl\geq0\\
&\ifandonlyif\group{\forall n\in\naturals}\group{\forall B_1,\dots,B_n\in\props}\group{\forall\lambda_1,\dots,\lambda_n\in\posreals}\,\sup\group[\bigg]{\sum_{k=1}^n\lambda_k\group{\indof{B_k}-1}}\geq0\\
&\ifandonlyif\group{\forall n\in\naturals}\group{\forall B_1,\dots,B_n\in\props}\,\bigcap_{k=1}^nB_k\neq\emptyset\\
&\ifandonlyif\emptyset\notin\filterbase[\props].
\end{aligned}
\end{multline*}

Assume now that this finite intersection property is satisfied.
Taking into account \cref{eq::model:closure}, the \(\bgM\)-natural extension~\(\closureof{\Msabove{\bgM}}{\A_{\props}}=\closureof{\Ms}{\A_{\props}\union\bgM}\) of the \(\bgM\)-consistent AD-assessment~\(\A_{\props}\) is then given by its set of acceptable gambles
\begin{align*}
\closureof{\adfMsabove{\bgM}}{\A_{\props}}_\acc
&=\posiof{\indifset[\props]\cup\gambles_{\weakgeq0}}\\
&=\gambles_{\weakgeq0}\cup\group{\posiof{\indifset[\props]}+\gambles_{\weakgeq0}}
&&\reason{Eq.~\labelcref{eq::form:propositional:model:third}}
\end{align*}
and its set of desirable gambles
\begin{align*}
\closureof{\adfMsabove{\bgM}}{\A_{\props}}_\des
&=\shullof{\emptyset\cup\gambles_{\stronggt0}}\cup\group[\big]{\shullof{\emptyset\cup\gambles_{\stronggt0}}+\posiof{\indifset[\props]\cup\gambles_{\weakgeq0}}}\\
&=\gambles_{\stronggt0}\cup\group[\big]{\gambles_{\stronggt0}+\group{\gambles_{\weakgeq0}\cup\group{\posiof{\indifset[\props]}+\gambles_{\weakgeq0}}}}
&&\reason{Eq.~\labelcref{eq::form:propositional:model:third}}\\
&=\gambles_{\stronggt0}\cup\group[\big]{\group{\gambles_{\stronggt0}+\gambles_{\weakgeq0}}\cup\group{\gambles_{\stronggt0}+\posiof{\indifset[\props]}+\gambles_{\weakgeq0}}}\\
&=\gambles_{\stronggt0}\cup\gambles_{\stronggt0}\cup\group{\posiof{\indifset[\props]}+\gambles_{\stronggt0}}
&&\reason{\(\gambles_{\weakgeq0}+\gambles_{\stronggt0}=\gambles_{\stronggt0}\)}\\
&=\gambles_{\stronggt0}\cup\group{\posiof{\indifset[\props]}+\gambles_{\stronggt0}}.
\end{align*}

The proof is therefore complete if we can show that
\begin{equation*}
\gambles_{\weakgeq0}\cup\group{\posiof{\indifset[\props]}+\gambles_{\weakgeq0}}=\gambles_{\weakgeq0}^{\filter[\!\!\props]}
\text{ and }
\gambles_{\stronggt0}\cup\group{\posiof{\indifset[\props]}+\gambles_{\stronggt0}}=\gambles_{\stronggt0}^{\filter[\!\!\props]}.
\end{equation*}
We'll concentrate on the proof of the second equality, as the proof of the first is very similar.
For the direct inclusion, it suffices to show that \(\posiof{\indifset[\props]}+\gambles_{\stronggt0}\subseteq\gambles_{\stronggt0}^{\filter[\!\!\props]}\).
So, consider any~\(\gbl\in\posiof{\indifset[\props]}+\gambles_{\stronggt0}\), then there are~\(n\in\naturals\), \(\lambda_1,\dots,\lambda_n\in\posreals\) and \( B_1,\dots,B_n\in\props\) such that \(\inf\group{\gbl-\sum_{k=1}^n\lambda_k\group{\indof{B_k}-1}}>0\), so clearly also~\(\inf\group{\gbl\vert C}>0\) with~\(C\coloneqq\bigcap_{k=1}^nB_k\in\filterbase[\props]\).
For the converse inclusion, consider any~\(\gbl\in\gambles_{\stronggt0}^{\filter[\!\!\props]}\), so there's some~\(C\in\filterbase[\props]\) such that \(\delta\coloneqq\inf\group{\gbl\vert C}>0\).
Then there are~\(n\in\naturals\) and \(B_1,\dots,B_n\in\props\) for which \(C=\bigcap_{k=1}^nB_k\).
Choose all~\(\lambda_k\coloneqq\mu\), where \(\mu\) is any real number such that \(\mu>\inf\group{\gbl\vert C}-\inf\gbl\geq0\).
If we consider the gamble~\(\altgbl\coloneqq\gbl-\sum_{k=1}^n\lambda_k\group{\indof{B_k}-1}\), then it's enough to prove that \(\inf\altgbl>0\).
Fix any~\(\state\in\states\) and let \(m_{\state}\coloneqq\abs{\set{k\in\{1,\dots,n\}\given\state\notin B_k}}\) be the number of~\(B_k\) that don't contain the state~\(\state\).
Then
\begin{equation*}
\altgbl\group\state
=\gbl\group\state-\sum_{k=1}^n\lambda_k\group{\indof{B_k}\group{\state}-1}
=\gbl\group\state+\sum_{k=1}^n\lambda_k\group{1-\indof{B_k}\group{\state}}
\geq\gbl\group\state+m_{\state}\mu.
\end{equation*}
There are now two cases to consider.
The first is that \(m_{\state}=0\), so \(\state\in\bigcap_{k=1}^nB_k=C\), in which case \(\altgbl\group\state=\gbl\group\state\geq\delta\).
The second is that \(m_{\state}\geq1\), and then
\begin{align*}
\altgbl\group\state
&\geq\gbl\group\state+m_{\state}\mu
=\gbl\group\state+\group{m_{\state}-1}\mu+\mu\\
&\geq\gbl\group\state+\group{m_{\state}-1}\mu+\group{\delta-\inf\gbl}\\
&=\delta+\group{\gbl\group\state-\inf\gbl}+\group{m_{\state}-1}\mu\\
&\geq\delta.
\end{align*}
So in all cases, we have that \(\altgbl\group\state\geq\delta\), and therefore, indeed, \(\inf\altgbl\geq\delta>0\).
\end{proof}

\begin{proof}[\protect{Proof that \(\filter[\M]\) is always proper filter}]
First, \(\filter[\M]\) is non-empty because \(\states\in\filter[\M]\), since \(\indof\states-1=0\in\M_\acc\), by \labelcref{axiom:AD:background}.

Second, we show \(\filter[\M]\) is closed under finite intersections.
Assume that \(B_1,B_2\in\filter[\M]\), so \(\indof{B_1}-1\in\M_\acc\) and \(\indof{B_2}-1\in\M_\acc\).
Since \(\indof{B_1\cap B_2}=\indof{B_1}+\indof{B_2}-\indof{B_1\cup B_2}\geq\indof{B_1}+\indof{B_2}-1\), we see that
\begin{equation*}
\indof{B_1\cap B_2}-1
\geq\indof{B_1}+\indof{B_2}-2
=\group{\indof{B_1}-1}+\group{\indof{B_2}-1},
\end{equation*}
and therefore \(\indof{B_1\cap B_2}-1\in\M_\acc\), by \labelcref{axiom:AD:deductive:closedness,axiom:AD:background}.
Hence, indeed, \(B_1\cap B_2\in\filter[\M]\).

Third, we show that \(\filter[\M]\) is increasing.
Assume that \(B_1\subseteq B_2\) and \(B_1\in\filter[\M]\), implying that \(\indof{B_1}-1\in\M_\acc\).
But then, since \(\indof{B_2}\geq\indof{B_1}\), we find that \(\indof{B_2}-1\in\M_\acc\), by \labelcref{axiom:AD:deductive:closedness,axiom:AD:background}, so indeed also \(B_2\in\filter[\M]\).

To conclude, we show that \(\filter[\M]\) is proper.
Assume towards contradiction that \(\emptyset\in\filter[\M]\), implying that \(-1=\indof{\emptyset}-1\in\M_\acc\), contradicting \labelcref{axiom:AD:nonpositive:not:acceptable} since \(1\in\gambles_{\stronggt0}\).
\end{proof}

\begin{proof}[Proof of the statement in \cref{eq::prop:filter:correspondence}]
Since we know that \(\filter[\M]\) is a proper filter, the converse implication is immediate.
For the direct implication, assume that \(\M\in\propadfMs\), so \(\M=\M_{\filter}=\Adelim{\gambles_{\weakgeq0}^{\filter}}{-\gambles_{\stronggt0}^{\filter}}\) for some proper filter~\(\filter\).
It's enough to show that this implies that \(\filter=\filter[\M]\).
Consider, to this end, the following chain of equivalences, valid for any~\(B\subseteq\states\):
\begin{align*}
B\in\filter[\M]
&\ifandonlyif\indof{B}-1\in\M_\acc\\
&\ifandonlyif\group{\exists D\in\filter}\,\inf\group{\indof{B}-1\vert D}\geq0
\ifandonlyif\group{\exists D\in\filter}\,\inf\group{\indof{B}\vert D}\geq1\\
&\ifandonlyif\group{\exists D\in\filter}\,D\subseteq B\\
&\ifandonlyif B\in\filter.
\qedhere
\end{align*}
\end{proof}

\begin{proof}[Proof of \cref{prop::propositional:expansion:consistency}]
For sufficiency, assume that \(\filter\cup\set\event\) satisfies the finite intersection property, then we have to show that \(\group{\M_{\filter}}_\des\cap\indifset[\event]=\emptyset\), or equivalently, \(\gambles_{\stronggt0}^{\filter}\cap\indifset[\event]=\emptyset\).
So, consider any~\(\gbl\in\gambles_{\stronggt0}^{\filter}\), then there's some~\( B\in\filter\) such that \(\inf\group{\gbl\vert B}>0\) and therefore also \(\gbl\group\state>0\) for all~\(\state\in B\).
Assume towards contradiction that \(\gbl\in\indifset[\event]\), so \(\gbl\group\state=0\) for all~\(\state\in\event\).
Since it follows from the finite intersection property for~\(\filter\cup\set\event\) that \(\event\cap B\neq\emptyset\) [see \cref{eq::augmented:finite:intersection:property}], there's some~\(\altstate\in\event\cap B\) for which both \(\gbl\group\altstate>0\) and \(\gbl\group\altstate=0\), which is the desired contradiction.

For necessity, assume that \(\filter\cup\set\event\) doesn't satisfy the finite intersection property, so \(B\cap\event=\emptyset\) for some~\(B\in\filter\).
But then on the one hand \(\indof{B}\in\indifset[\event]\), and on the other hand \(\inf\group{\indof{B}\vert B}=1>0\), so \(\indof{B}\in\gambles_{\stronggt0}^{\filter}\).
Hence, indeed, \(\gambles_{\stronggt0}^{\filter}\cap\indifset[\event]\neq\emptyset\).
\end{proof}

\begin{proof}[Proof of \cref{prop::propositional:expansion:operator}]
It follows from the definition of the expansion operator that we only have to look at the case of consistency, which, according to \cref{prop::propositional:expansion:consistency}, is equivalent to the finite intersection property for~\(\filter\cup\set\event\), so \(\emptyset\notin\filterbase[\!\!\filter\cup\set\event]\).
In that case \(\filter^\prime\coloneqq\filter[\!\!\filter\cup\set\event]\) is a proper filter, and we have to show that
\begin{equation*}
\Adelim{\M_\acc+\indifset[\event]}{-\group{\M_\des+\indifset[\event]}}
=\Adelim{\gambles_{\weakgeq0}^{\filter}+\indifset[\event]}{-\group{\gambles_{\stronggt0}^{\filter}+\indifset[\event]}}
=\Adelim{\gambles_{\weakgeq0}^{\filter^\prime}}{-\gambles_{\stronggt0}^{\filter^\prime}}
\end{equation*}

We start by proving that \(\M_\acc+\indifset[\event]\subseteq\gambles_{\weakgeq0}^{\filter^\prime}\).
Consider any~\(\gbl\in\M_\acc\) and any~\(\altgbl\in\indifset[\event]\), which tells us that \(\indof\event\altgbl=0\) and that there's some~\(B\in\filter\) such that \(\inf\group{\gbl\vert B}\geq0\).
Clearly, then,
\begin{equation*}
\inf\group{\gbl+\altgbl\vert B\cap\event}
=\inf_{\state\in B\cap\event}\group{\gbl\group\state+\altgbl\group\state}
=\inf_{\state\in B\cap\event}\gbl\group\state
\geq\inf_{\state\in B}\gbl\group\state
=\inf\group{\gbl\vert B}\geq0,
\end{equation*}
and therefore indeed \(\gbl+\altgbl\in\gambles_{\weakgeq0}^{\filter^\prime}\), taking into account \cref{eq::augmented:filter}.
A similar argument can be used to show that \(\M_\des+\indifset[\event]\subseteq\gambles_{\stronggt0}^{\filter^\prime}\).

Conversely, consider any~\(\gbl\in\gambles_{\weakgeq0}^{\filter^\prime}\).
Looking at \cref{eq::augmented:filter}, we see that there are two cases to consider.
If \(\inf\group{\gbl\vert B}\geq0\) for some~\(B\in\filter\), then we're done.
If \(\inf\group{f\vert B\cap\event}\geq0\) for some~\(B\in\filter\), then definitely \(\event\cap B\neq\emptyset\) because we assumed that \(\emptyset\notin\filterbase[\!\!\filter\cup\set\event]\).
We may then also assume without loss of generality that \(B\setminus\event\neq\emptyset\), because otherwise \(B\cap\event=B\in\filter\), which takes us back to the previous case.
We can therefore write the gamble~\(\gbl\) as \(\gbl=\altgbl+\altgbltoo\) with
\begin{equation*}
\altgbl\group\state
\coloneqq
\begin{cases}
\sup\gbl
&\text{if \(\state\in B\setminus\event\)}\\
\gbl\group\state
&\text{otherwise}
\end{cases}
\text{ and }
\altgbltoo\group\state
\coloneqq
\begin{cases}
\gbl\group\state-\sup\gbl
&\text{ if \(\state\in B\setminus\event\)}\\
0 &\text{ otherwise}
\end{cases}
\text{ for all \(\state\in\states\).}
\end{equation*}
Observe that \(\indof\event\altgbltoo=0\), so \(\altgbltoo\in\indifset[\event]\), and that
\begin{align*}
\inf\group{\altgbl\vert B}
&=\inf_{\state\in B}\altgbl\group\state
=\min\set[\Big]{\inf_{\state\in B\cap E}\altgbl\group\state,\inf_{\state\in B\setminus E}\altgbl\group\state}
=\min\set[\Big]{\inf_{\state\in B\cap E}\gbl\group\state,\sup\gbl}\\
&=\inf_{\state\in B\cap E}\gbl\group\state
=\inf\group{\gbl\vert B\cap\event}\geq0,
\end{align*}
so \(\altgbl\in\M_\acc\).
Hence, indeed, \(\gbl=\altgbl+\altgbltoo\in\M_\acc+\indifset[\event]\).
A similar argument can be used to show that also \(\gambles_{\stronggt0}^{\filter^\prime}\subseteq\M_\des+\indifset[\event]\).
\end{proof}

\begin{proof}[Proof of \cref{prop::propositional:revision:conditionability}]
Simply consider the following chain of equivalences:
\begin{align*}
0\notin\gambles_{\weakgeq0}^{\filter}\altcondon\event+\gambles_{\stronggt0}
&\ifandonlyif\group{\forall\gbl\in\gambles}\group[\big]{\indof\event\gbl\in\gambles_{\weakgeq0}^{\filter}\then\sup\gbl\geq0}\\
&\ifandonlyif\group{\forall\gbl\in\gambles}\group[\big]{\group{\exists B\in\filter}\indof{B}\indof\event\gbl\geq0\then\sup\gbl\geq0}\\
&\ifandonlyif\group{\forall\gbl\in\gambles}\group[\big]{\sup\gbl<0\then\group{\forall B\in\filter}\indof{B}\indof\event\gbl\not\geq0}\\
&\ifandonlyif\group{\forall B\in\filter}\,B\cap\event\neq\emptyset,
\end{align*}
and use \cref{eq::augmented:finite:intersection:property} to conclude.
\end{proof}

\begin{proof}[Proof of \cref{prop::propositional:revision:operator}]
It's clearly enough to consider the case that \(\emptyset\notin\filterbase[\!\!\filter\cup\set\event]\), so we infer from \cref{eq::AD:revision:operator:rewritten} that
\begin{equation}\label{eq::propositional:revision:operator}
\reviseof{\M_{\filter}\given\event}
=\Adelim[\big]{\gambles_{\weakgeq0}^{\filter}\altcondon\event}{-\group[\big]{\gambles_{\stronggt0}^{\filter}\altcondon\event\cup\group{\gambles_{\stronggt0}+\gambles_{\weakgeq0}^{\filter}\altcondon\event}}}
\end{equation}
We start with the derivation of the accept part, by looking at the following chain of equivalences, valid for any gamble~\(\gbl\in\gambles\):
\begin{equation}\label{eq::propositional:revision:accept:equivalence}
\gbl\in\gambles_{\weakgeq0}^{\filter}\altcondon\event
\ifandonlyif\indof\event\gbl\in\gambles_{\weakgeq0}^{\filter}
\ifandonlyif\group{\exists B\in\filter}\indof{B}\indof\event\gbl\geq0
\ifandonlyif\group{\exists B\in\filter}\indof{B\cap\event}\gbl\geq0.
\end{equation}
Hence, if \(\gbl\in\gambles_{\weakgeq0}^{\filter}\altcondon\event\), we see that there's some~\(D\in\filter\) such that \(\indof{D\cap\event}\gbl\geq0\), so we infer from \cref{eq::augmented:filter} that, indeed, there's some~\(B\in\filter[\!\!\filter\cup\set\event]\) such that \(\indof{B}\gbl\geq0\).

Conversely, if \(\gbl\in\gambles_{\weakgeq0}^{\filter[\!\!\filter\cup\set\event]}\), then there are, by \cref{eq::augmented:filter}, two possibilities.
The first is that there's some~\(D\in\filter\) such that \(\indof{D\cap\event}\gbl\geq0\), which implies that \(\gbl\in\gambles_{\weakgeq0}^{\filter}\altcondon\event\), by \cref{eq::propositional:revision:accept:equivalence}.
The second is that \(\indof{B}\gbl\geq0\) for some~\(B\in\filter\), and since \(B\cap\event\neq\emptyset\) because of the consistency, this implies that also \(\indof{B\cap\event}\gbl\geq0\), and therefore that, again, \(\gbl\in\gambles_{\weakgeq0}^{\filter}\altcondon\event\), by \cref{eq::propositional:revision:accept:equivalence}.
This proves the desired equality on the accept side.

For the remainder of the proof, we start by looking at the following chain of equivalences, valid for any gamble~\(\gbl\in\gambles\):
\begin{multline}\label{eq::propositional:revision:desirability:equivalence}
\gbl\in\gambles_{\stronggt0}^{\filter}\altcondon\event
\ifandonlyif\indof\event\gbl\in\gambles_{\stronggt0}^{\filter}
\ifandonlyif\group{\exists B\in\filter}\inf\group{\indof\event\gbl\vert B}>0\\
\ifandonlyif\group{\exists B\in\filter}\group{B\subseteq\event\text{ and }\inf\group{\gbl\vert B}>0}.
\end{multline}
So we see that there are two possibilities.
The first is that there's some~\(B\in\filter\) such that \(B\subseteq\event\), or in other words, that \(\event\in\filter\), and therefore also \(\filter\cup\set\event=\filter\).
But in this case it follows readily from \cref{eq::propositional:revision:accept:equivalence,eq::propositional:revision:desirability:equivalence} that
\begin{equation*}
\gambles_{\weakgeq0}^{\filter}\altcondon\event
=\gambles_{\weakgeq0}^{\filter}
\text{ and }
\gambles_{\stronggt0}^{\filter}\altcondon\event
=\gambles_{\stronggt0}^{\filter},
\end{equation*}
so we infer from \cref{lem:interior,eq::propositional:revision:operator} that \(\reviseof{\M_{\filter}\given\event}=\Adelim{\gambles_{\weakgeq0}^{\filter}}{-\gambles_{\stronggt0}^{\filter}}=\M_{\filter}\), as required.
The second possibility is that \(\event\notin\filter\), and then we infer from \cref{eq::propositional:revision:desirability:equivalence} that \(\gambles_{\stronggt0}^{\filter}\altcondon\event=\emptyset\), so \cref{eq::propositional:revision:operator} and the already established equality on the accept side then tell us that, indeed,
\begin{equation*}
\reviseof{\M_{\filter}\given\event}
=\Adelim[\big]{\gambles_{\weakgeq0}^{\filter[\!\!\filter\cup\set\event]}}{-\group[\big]{\gambles_{\stronggt0}+\gambles_{\weakgeq0}^{\filter[\!\!\filter\cup\set\event]}}}
=\Adelim[\big]{\gambles_{\weakgeq0}^{\filter[\!\!\filter\cup\set\event]}}{-\gambles_{\stronggt0}^{\filter[\!\!\filter\cup\set\event]}}=\M_{\filter[\!\!\filter\cup\set\event]},
\end{equation*}
where the second equality follows from \cref{lem:interior}.
\end{proof}

\begin{lemma}\label{lem:interior}
Let \(\filter\) be any proper filter, then \(\gambles_{\weakgeq0}^{\filter}+\gambles_{\stronggt0}=\gambles_{\stronggt0}^{\filter}\).
\end{lemma}

\begin{proof}[Proof of \cref{lem:interior}]
For the direct inclusion, start with any element~\(\gbl\) of~\(\gambles_{\weakgeq0}^{\filter}+\gambles_{\stronggt0}\), so \(\gbl=\altgbl+\altgbltoo\) with~\(\altgbl\in\gambles_{\weakgeq0}^{\filter}\) and \(\altgbltoo\in\gambles_{\stronggt0}\).
Then there's some~\(B\in\filter\) such that \(\indof{B}\altgbl\geq0\), and since \(\altgbltoo\stronggt0\), we have that \(\inf\group{\altgbl+\altgbltoo\vert B}\geq\inf\group{\altgbltoo\vert B}\geq\inf\altgbltoo>0\), so \(\gbl=\altgbl+\altgbltoo\in\gambles_{\stronggt0}^{\filter}\).

For the reverse inclusion, start with any element~\(\gbl\) of \(\gambles_{\stronggt0}^{\filter}\), so there's some~\(B\in\filter\) such that \(\inf\group{\gbl\vert B}>0\).
Then \(\gbl=(\gbl-\frac{1}{2}\inf\group{\gbl\vert B})+\frac{1}{2}\inf\group{\gbl\vert B}\), with~\(\inf(\gbl-\frac{1}{2}\inf\group{\gbl\vert B}\vert B)=\inf\group{\gbl\vert B}-\frac12\inf\group{\gbl\vert B}=\frac12\inf\group{\gbl\vert B}\geq0\) and \(\frac{1}{2}\inf\group{\gbl\vert B}\stronggt0\), so \(\gbl\in\gambles_{\weakgeq0}^{\filter}+\gambles_{\stronggt0}\).
\end{proof}

\begin{proof}[Proof of \cref{prop::postulates:satisfied:propositional}]
We've already argued in \cref{prop::AD:model:revision:axioms:obeyed} that \labelcref{axiom:revision:1,axiom:revision:2,axiom:revision:3,axiom:revision:5,axiom:revision:7} hold for general AD-models, so it only remains to prove that \labelcref{axiom:revision:4,axiom:revision:8} are obeyed.

For \labelcref{axiom:revision:4}, assume that \(\M_{\filter}\) and \(\event\) are \(\bgM\)-consistent.
As proved in \cref{prop::propositional:expansion:consistency}, this amounts to requiring logical consistency for the set of propositions~\(\filter\cup\set{\event}\), or, equivalently, to requiring that \(\emptyset\notin\filterbase[\filter\cup\set{\event}]\).
Under this condition, we infer from \cref{prop::propositional:expansion:operator,prop::propositional:revision:operator} that, indeed, \(\expandof{\M_{\filter}\given\event}=\M_{\filter[\filter\cup\set{\event}]}=\reviseof{\M_{\filter}\given\event}\).

For \labelcref{axiom:revision:8}, assume that \(\reviseof{\M_{\filter}\given\event[1]}\) and \(\event[2]\) are \(\bgM\)-consistent.
When the original AD-model~\(\M_{\filter}\) is consistent with \(\event[1]\), so when \(\filter\cup\set{\event[1]}\) is logically consistent, \(\reviseof{\M_{\filter}\given\event[1]}=\M_{\filter\cup\set{\event[1]}}\).
For this revised AD-model to be consistent with \(\event[2]\), we must require the set of propositions \(\filter\cup\set{\event[1]}\cup\set{\event[2]}\) to be logically consistent.
On the other hand, when \(\filter\cup\set{\event[1]}\) isn't logically consistent, the revised AD-model becomes \(\M_{\filter[\set{\event[1]}]}\), and we must then demand that \(\emptyset\notin\filterbase[\set{\event[1]}\cup\set{\event[2]}]\), so that \(\event[1]\cap\event[2]\neq\emptyset\).

Let's first look in more detail at the first case, where we assume that \(\filter\cup\set{\event[1]}\cup\set{\event[2]}\) is logically consistent.
Expanding \(\reviseof{\M\given\event[1]}\) with \(\event[2]\) then results in
\begin{equation*}
\expandof{\reviseof{\M\given\event[1]}\given\event[2]}
=\M_{\filter[\filter\cup\set{\event[1]}\cup\set{\event[2]}]}.
\end{equation*}
For the filter base for this model, we find that
\begin{align*}
\filterbase[\filter\cup\set{\event[1]}\cup\set{\event[2]}]
&=\set[\bigg]{\bigcap_{k=1}^{n}A_k\colon n\in\naturals,A_k\in\filter\cup\set{\event[1]}\cup\set{\event[2]}} \\
&=\filter\cup\set{\event[1]\cap B\colon B\in\filter}\cup\set{\event[2]\cap B\colon B\in\filter}\cup\set{\event[1]\cap\event[2]\cap B\colon B\in\filter}.
\end{align*}
The smallest proper filter that includes this filter base is then given by
\begin{multline*}
\filter[\filter\cup\set{\event[1]}\cup\set{\event[2]}]\\
\begin{aligned}
&=\set{B\subseteq\states\colon\group{\exists D\in\filterbase[\filter\cup\set{\event[1]}\cup\set{\event[2]}]}D\subseteq B}\\
&=\filter\cup\set{B\subseteq\states\colon\group{\exists D\in\filter}\,\event[1]\cap D\subseteq B}\\
&\quad\qquad\cup\set{B\subseteq\states\colon\group{\exists D\in\filter}\,\event[2]\cap D\subseteq B}\cup\set{B\subseteq\states\colon\group{\exists D\in\filter}\,\event[1]\cap\event[2]\cap D\subseteq B}\\
&=\filter\cup\set{B\subseteq\states\colon\group{\exists D\in\filter}\,\event[1]\cap\event[2]\cap D\subseteq B}.
\end{aligned}
\end{multline*}
Let's now look at the revision \(\reviseof{\M\given\event[1]\cap\event[2]}\), for which \cref{eq::propositional:revision:operator:definition} yields under our assumption of consistency that \(\reviseof{\M\given\event[1]\cap\event[2]}=\M_{\filter[\filter\cup\set{\event[1]\cap\event[2]}]}\).
Once again, we turn our attention to the filter base and the corresponding proper filter:
\begin{equation*}
\filterbase[\filter\cup\set{\event[1]\cap\event[2]}]
=\set[\bigg]{\bigcap_{k=1}^{n}A_k\colon n\in\naturals,A_k\in\filter\cup\set{\event[1]\cap\event[2]}}=\filter\cup\set{\event[1]\cap\event[2]\cap B\colon B\in\filter}
\end{equation*}
and
\begin{align*}
\filter[\filter\cup\set{\event[1]\cap\event[2]}]
&=\set{B\subseteq\states\colon\group{\exists D\in\filterbase[\filter\cup\set{\event[1]\cap\event[2]}]}\,D\subseteq B}\\
&=\filter\cup\set{B\subseteq\states\colon\group{\exists D\in\filter}\,\event[1]\cap\event[2]\cap D\subseteq B}.
\end{align*}
Hence, \(\filter[\filter\cup\set{\event[1]}\cup\set{\event[2]}]=\filter[\filter\cup\set{\event[1]\cap\event[2]}]\), and therefore \(\expandof{\reviseof{\M_{\filter}\given\event[1]}\given\event[2]}=\reviseof{\M\given\event[1]\cap\event[2]}\).

We now turn to a more detailed investigation of the second case, where \(\filter\cup\set{\event[1]}\) isn't logically consistent, so \(\emptyset\in\filterbase[\filter\cup\set{\event[1]}]\).
The requirement for the \(\bgM\)-consistency of \(\reviseof{\M\given\event[1]}\) with \(\event[2]\) becomes \(\emptyset\notin\filterbase[\set{\event[1]}\cup\set{\event[2]}]\), and the corresponding expansion \(\expandof{\reviseof{\M\given\event[1]}\given\event[2]}=\M_{\filter[\set{\event[1]}\cup\set{\event[2]}]}\).
But \(\emptyset\in\filterbase[\filter\cup\set{\event[1]}]\) also implies that \(\emptyset\in\filterbase[\set{\event[1]}\cup\set{\event[2]}]\), so we find for the revision that \(\reviseof{\M\given\event[1]\cap\event[2]}=\M_{\filter[\set{\event[1]\cap\event[2]}]}\).
Here too, we look at the filter bases and filters and compare them:
\begin{align*}
\filterbase[\set{\event[1]}\cup\set{\event[2]}]
&=\set{\event[1]}\cup\set{\event[2]}\cup\set{\event[1]\cap\event[2]}\\
\filter[\set{\event[1]}\cup\set{\event[2]}]
&=\set{B\subseteq\states\colon\group{\exists D\in\filterbase[\set{\event[1]}\cup\set{\event[2]}]}\, D\subseteq B}\\
&=\set{B\subseteq\states\colon\event[1]\subseteq B}\cup\set{B\subseteq\states\colon\event[2]\subseteq B}\cup\set{B\subseteq\states\colon\event[1]\cap\event[2]\subseteq B}\\
&=\set{B\subseteq\states\colon\event[1]\cap\event[2]\subseteq B}\\
\filterbase[\set{\event[1]\cap\event[2]}]
&=\event[1]\cap\event[2]\\
\filter[\set{\event[1]\cap\event[2]}]
&=\set{B\subseteq\states\colon\event[1]\cap\event[2]\subseteq B}.
\end{align*}
Since, we see that \(\filter[\set{\event[1]}\cup\set{\event[2]}]=\filter[\set{\event[1]\cap\event[2]}]\), we're now done.
\end{proof}

\subsection{For \cref{sec::full:conditional:previsions}}

\begin{proof}[Proof of \cref{prop::previsional:consistency}]
If we recall \cref{eq::consistency}, we see that the condition for this consistency is given by~\(\posiof{\mathscr{L}_{\prevof{\smallbolleke\given\smallbolleke}}\cup\gambles_{\weakgeq0}}\cap-(\emptyset\cup\gambles_{\stronggt0})=\emptyset\).
We now have the following chain of equivalences:
\begin{multline*}
\posiof{\mathscr{L}_{\prevof{\smallbolleke\given\smallbolleke}}\cup\gambles_{\weakgeq0}}\cap-(\emptyset\cup\gambles_{\stronggt0})=\emptyset\\
\begin{aligned}
&\ifandonlyif0\notin\posiof{\mathscr{L}_{\prevof{\smallbolleke\given\smallbolleke}}\cup\gambles_{\weakgeq0}}+(\emptyset\cup\gambles_{\stronggt0})\\
&\ifandonlyif0\notin\group[\big]{\posiof{\mathscr{L}_{\prevof{\smallbolleke\given\smallbolleke}}}\cup\group{\posiof{\mathscr{L}_{\prevof{\smallbolleke\given\smallbolleke}}}+\gambles_{\weakgeq0}}\cup\gambles_{\weakgeq0}}+\gambles_{\stronggt0}\\
&\ifandonlyif0\notin\group[\big]{\group{\posiof{\mathscr{L}_{\prevof{\smallbolleke\given\smallbolleke}}}+\gambles_{\weakgeq0}}\cup\gambles_{\weakgeq0}}+\gambles_{\stronggt0}
&&\reason{\(0\in\gambles_{\weakgeq0}\)}\\
&\ifandonlyif0\notin(\posiof{\mathscr{L}_{\prevof{\smallbolleke\given\smallbolleke}}}+\gambles_{\stronggt0})\cup\gambles_{\stronggt0}
&&\reason{\(\gambles_{\weakgeq0}+\gambles_{\stronggt0}=\gambles_{\stronggt0}\)}\\
&\ifandonlyif0\notin\posiof{\mathscr{L}_{\prevof{\smallbolleke\given\smallbolleke}}}+\gambles_{\stronggt0},
&&\reason{\(0\notin\gambles_{\stronggt0}\)}
\end{aligned}
\end{multline*}
so the \(\bgM\)-consistency of the AD-assessment~\(\A_{\prevof{\smallbolleke\given\smallbolleke}}\) is equivalent to
\begin{multline*}
\sup\group[\bigg]{\sum_{k=1}^n\lambda_k\indof{B_k}\group{\gbl_k-\prevof{\gbl_k\given B_k}+\epsilon_k}}\geq0\\
\text{ for all \(n\in\naturals\), \(\lambda_k,\epsilon_k\in\posreals\), \(\gbl_k\in\gambles\) and \(B_k\in\powersetof{\states}\setminus\set{\emptyset}\).}
\end{multline*}
Since we assumed that the full conditional prevision~\(\prevof{\bolleke\given\bolleke}\) is coherent, we know that it satisfies the avoiding sure loss condition~\labelcref{eq::avoiding:sure:loss}.
This implies that the condition for \(\bgM\)-consistency will always be obeyed, since
\begin{multline*}
\sup\group[\bigg]{\sum_{k=1}^n\lambda_k\indof{B_k}\group{\gbl_k-\prevof{\gbl_k\given B_k}+\epsilon_k}}\\
\geq\sup\group[\bigg]{\sum_{k=1}^n\lambda_k\indof{B_k}\group{\gbl_k-\prevof{\gbl_k\given B_k}+\epsilon_k}\vert\bigcup_{k=1}^{n}B_k}
\geq0.
\qedhere
\end{multline*}
\end{proof}

\begin{proof}
According to \cref{eq::model:closure}, we find the following AD-model in~\(\adfMsabove{\bgM}\):
\begin{align*}
\M_{\prevof{\smallbolleke\given\smallbolleke}}
\coloneqq&\closureof{\Msabove{\bgM}}{\A_{\prevof{\smallbolleke\given\smallbolleke}}}\\
=&\Adelim[\big]{\posiof{\mathscr{L}_{\prevof{\smallbolleke\given\smallbolleke}}\cup\gambles_{\weakgeq0}}}{-\group[\big]{\shullof{\emptyset\cup\gambles_{\stronggt0}}\cup(\shullof{\emptyset\cup\gambles_{\stronggt0}}+\posiof{\mathscr{L}_{\prevof{\smallbolleke\given\smallbolleke}}\cup\gambles_{\weakgeq0}})}}.
\end{align*}
Recalling that \(\shullof{\emptyset\cup\gambles_{\stronggt0}}=\gambles_{\stronggt0}\) and
\begin{align*}
\posiof{\mathscr{L}_{\prevof{\smallbolleke\given\smallbolleke}}\cup\gambles_{\weakgeq0}}
&=\posiof{\mathscr{L}_{\prevof{\smallbolleke\given\smallbolleke}}}\cup\posiof{\gambles_{\weakgeq0}}\cup\group{\posiof{\mathscr{L}_{\prevof{\smallbolleke\given\smallbolleke}}}+\gambles_{\weakgeq0}}\\
&=\gambles_{\weakgeq0}\cup\group{\posiof{\mathscr{L}_{\prevof{\smallbolleke\given\smallbolleke}}}+\gambles_{\weakgeq0}},
&&\reason{\(0\in\gambles_{\weakgeq0}\)}
\end{align*}
we can simplify this to:
\begin{align*}
M_{\prevof{\smallbolleke\given\smallbolleke}}
&=\Adelim[\big]{\gambles_{\weakgeq0}\cup\group{\posiof{\mathscr{L}_{\prevof{\smallbolleke\given\smallbolleke}}}+\gambles_{\weakgeq0}}}{-\group[\big]{\gambles_{\stronggt0}+\group{\gambles_{\weakgeq0}\cup\group{\posiof{\mathscr{L}_{\prevof{\smallbolleke\given\smallbolleke}}}+\gambles_{\weakgeq0}}}}}\\
&=\Adelim{\gambles_{\weakgeq0}\cup\group{\posiof{\mathscr{L}_{\prevof{\smallbolleke\given\smallbolleke}}}+\gambles_{\weakgeq0}}}{-\group{\gambles_{\stronggt0}\cup\group{\posiof{\mathscr{L}_{\prevof{\smallbolleke\given\smallbolleke}}}+\gambles_{\stronggt0}}}},
\end{align*}
where we took into account that \(0\in\gambles_{\weakgeq0}\) and \(\gambles_{\weakgeq0}+\gambles_{\stronggt0}=\gambles_{\stronggt0}\).
\end{proof}

\begin{proof}[Proof of the statement in~\cref{eq::prevision:model:correspondence}]
Fix any regular event~\(B\in\powersetof{\states}\setminus\set{\emptyset}\).
First, it's easy to check that the constant gamble~\(1\) is an interior point, and therefore an order unit, for the convex cone \(\M_\acc\altcondon B\), which guarantees that for any gamble~\(\gbl\in\gambles\), there are real~\(\alpha,\beta\) for which \(\gbl-\alpha\in\M_\acc\altcondon B\) and \(\beta-\gbl\in\M_\acc\altcondon B\), and therefore the maps~\(\lprev_{\M}\group{\bolleke\vert B}\) and \(\uprev_{\M}\group{\bolleke\vert B}\) are real-valued and \(\lprev_{\M}\group{\bolleke\vert B}\leq\uprev_{\M}\group{\bolleke\vert B}\).
Fix any gamble~\(\gbl\in\gambles\), and observe that, by \cref{eq::model:conditional:full:prevision},
\begin{align*}
\lprev_{\M_{\prevof{\smallbolleke\given\smallbolleke}}}\group{\gbl\condon B}
&=\sup\set{\alpha\in\reals\given\indof{B}\group{\gbl-\alpha}\in\group{\M_{\prevof{\smallbolleke\given\smallbolleke}}}_\acc}\\
&=\sup\set{\alpha\in\reals\given\indof{B}\group{\gbl-\alpha}\in\gambles_{\weakgeq0}\cup\group{\posiof{\mathscr{L}_{\prevof{\smallbolleke\given\smallbolleke}}}+\gambles_{\weakgeq0}}}.
\end{align*}
Fix any real~\(\epsilon>0\), then \(\indof{B}\group{\gbl-\prevof{\gbl\given B}+\epsilon}\in\mathscr{L}_{\prevof{\smallbolleke\given\smallbolleke}}\subseteq\posiof{\mathscr{L}_{\prevof{\smallbolleke\given\smallbolleke}}}+\gambles_{\weakgeq0}\), so it follows that
\begin{equation*}
\prevof{\gbl\condon B}-\epsilon
\leq\lprev_{\M_{\prevof{\smallbolleke\given\smallbolleke}}}\group{\gbl\condon B},
\text{ and similarly, }
\prevof{\gbl\condon B}+\epsilon\geq\uprev_{\M_{\prevof{\smallbolleke\given\smallbolleke}}}\group{\gbl\condon B}.
\end{equation*}
Since \(\lprev_{\M_{\prevof{\smallbolleke\given\smallbolleke}}}\group{\gbl\condon B}\leq\uprev_{\M_{\prevof{\smallbolleke\given\smallbolleke}}}\group{\gbl\condon B}\), this implies that
\begin{equation*}
\prevof{\gbl\condon B}-\epsilon
\leq\lprev_{\M_{\prevof{\smallbolleke\given\smallbolleke}}}\group{\gbl\condon B}
\leq\uprev_{\M_{\prevof{\smallbolleke\given\smallbolleke}}}\group{\gbl\condon B}\leq\prevof{\gbl\condon B}+\epsilon.
\end{equation*}
Since this holds for all~\(\epsilon>0\), we conclude that \(\lprev_{\M_{\prevof{\smallbolleke\given\smallbolleke}}}\group{\gbl\condon B}=\uprev_{\M_{\prevof{\smallbolleke\given\smallbolleke}}}\group{\gbl\condon B}=\prevof{\gbl\condon B}\), as required.
\end{proof}

\begin{proof}[Proof of the statement in~\cref{eq::model:prevision:correspondence}]
If \(\lprev_{\M}\group{\bolleke\condon\bolleke}\) is a coherent full conditional prevision, then clearly \(\M_{\lprev_{\M}\group{\smallbolleke\condon\smallbolleke}}\) is a full previsional AD-model, so the converse implication is immediate.
For the direct implication, assume that \(\M\in\prevadfMs\), so \(\M=\M_{\prevof{\smallbolleke\condon\smallbolleke}}\) for some coherent full conditional prevision~\(\prevof{\bolleke\condon\bolleke}\).
It's now enough to show that this implies that \(\prevof{\bolleke\condon\bolleke}=\lprev_{\M}\group{\bolleke\condon\bolleke}\), which is an immediate consequence of the result in \cref{eq::prevision:model:correspondence}.
\end{proof}

\begin{proof}[Proof of \cref{prop::previsional:revision:conditionability}]
Fix any regular event~\(\event\in\powersetof{\states}\setminus\set{\emptyset}\), then we have to prove that \(\M_\acc\altcondon{\event}\cap-\gambles_{\stronggt0}=\emptyset\), with \(\M_\acc=\group{\posiof{\mathscr{L}_{\prevof{\smallbolleke\given\smallbolleke}}}+\gambles_{\weakgeq0}}\cup\gambles_{\weakgeq0}\).
As proved in \cref{lem::previsional:auxiliary:result},
\begin{equation*}
\M_\acc\altcondon{\event}
=\group{\posiof{\mathscr{L}_{\prev_{\event}\group{\smallbolleke\condon\smallbolleke}}}+\gambles^{\set\event}_{\weakgeq0}}\cup\gambles^{\set\event}_{\weakgeq0},
\end{equation*}
and therefore the conditionability condition reduces to
\begin{equation*}
\gambles^{\set\event}_{\weakgeq0}\cap-\gambles_{\stronggt0}=\emptyset
\text{ and }
\group{\posiof{\mathscr{L}_{\prev_{\event}\group{\smallbolleke\condon\smallbolleke}}}+\gambles^{\set\event}_{\weakgeq0}}\cap-\gambles_{\stronggt0}=\emptyset.
\end{equation*}
The first condition can be rewritten as \(0\notin\gambles^{\set\event}_{\weakgeq0}+\gambles_{\stronggt0}\), and we infer from \cref{lem:conditional:desirable:background} that
\begin{equation*}
\gambles^{\set\event}_{\weakgeq0}+\gambles_{\stronggt0}
=\gambles_{\weakgeq0}+\indifset[\event]+\gambles_{\stronggt0}
=\gambles_{\stronggt0}+\indifset[\event],
\end{equation*}
so we're led to the requirement that \(0\notin\gambles_{\stronggt0}+\indifset[\event]\), which is the condition for the regularity of the event~\(\event\), and therefore satisfied by assumption.

The second condition can be rewritten as \(0\notin\posiof{\mathscr{L}_{\prev_{\event}\group{\smallbolleke\condon\smallbolleke}}}+\gambles^{\set\event}_{\weakgeq0}+\gambles_{\stronggt0}\), which, again invoking \cref{lem:conditional:desirable:background}, reduces to \(0\notin\posiof{\mathscr{L}_{\prev_{\event}\group{\smallbolleke\condon\smallbolleke}}}+\gambles_{\stronggt0}+\indifset[\event]\).
Again taking into account \cref{lem:conditional:desirable:background}, this means that for all~\(m\in\naturals\), \(\mu_\ell,\delta_\ell\in\posreals\), \(\altgbl_\ell\in\gambles\) and \(C_\ell\in\powersetof{\event}\setminus\emptyset\), it must be that
\begin{equation*}
\sup\group[\bigg]{\sum_{\ell=1}^m\mu_\ell\indof{C_\ell}\group{\altgbl_\ell-\prevof{\altgbl_\ell\given C_\ell}+\delta_\ell}\Big\vert\event}\geq0,
\end{equation*}
which is guaranteed to hold by the condition \labelcref{eq::avoiding:sure:loss}.
\end{proof}

\begin{lemma}\label{lem::previsional:auxiliary:result}
Consider any regular event \(\event\in\powersetof{\states}\backslash\set{\emptyset}\) and any full previsional AD-model \(\M_{\prevof{\smallbolleke\given\smallbolleke}}\in\prevadfMs\).
Then
\begin{align}
\group{\posiof{\mathscr{L}_{\prevof{\smallbolleke\condon\smallbolleke}}}+\gambles_{\weakgeq0}}\altcondon\event
&=\posiof{\mathscr{L}_{\prev_{\event}\group{\smallbolleke\condon\smallbolleke}}}+\gambles^{\set\event}_{\weakgeq0}
\label{eq::previsional:auxiliary:result:accepts}\\
\group{\posiof{\mathscr{L}_{\prevof{\smallbolleke\condon\smallbolleke}}}+\gambles_{\stronggt0}}\altcondon\event
&\subseteq\group{\posiof{\mathscr{L}_{\prevof{\smallbolleke\condon\smallbolleke}}}+\gambles_{\weakgeq0}}\altcondon\event+\gambles_{\stronggt0}
\label{eq::previsional:auxiliary:result:desirs}\\
\gambles_{\weakgeq0}\altcondon\event
&=\gambles^{\set\event}_{\weakgeq0}
\label{eq::previsional:auxiliary:result:vacuous:accepts}\\
\gambles_{\stronggt0}\altcondon\event
&=\emptyset\text{ if \(\event\neq\states\)}.
\label{eq::previsional:auxiliary:result:vacuous:desirs}
\end{align}
\end{lemma}

\begin{proof}[Proof of \cref{lem::previsional:auxiliary:result}]
We begin with the proof of \cref{eq::previsional:auxiliary:result:accepts}.
For the converse inclusion, consider any~\(\gbl\in\posiof{\mathscr{L}_{\prev_{\event}\group{\smallbolleke\condon\smallbolleke}}}+\gambles^{\set\event}_{\weakgeq0}\), meaning that there are~\(m\in\naturals\), \(\mu_\ell,\delta_\ell\in\posreals\), \(\altgbl_\ell\in\gambles\), \(C_\ell\in\powersetof{\event}\setminus\emptyset\) and \(\altgbltoo\in\gambles\) with \(\indof{\event}\altgbltoo\weakgeq0\), such that \(\gbl=\sum_{\ell=1}^m\mu_\ell\indof{C_\ell}\group{\altgbl_\ell-\prevof{\altgbl_\ell\given C_\ell}+\delta_\ell}+\altgbltoo\), and therefore
\begin{equation*}
\indof{\event}\gbl
=\sum_{\ell=1}^m\mu_\ell\indof{C_\ell\cap\event}\group{\altgbl_\ell-\prevof{\altgbl_\ell\given C_\ell}+\delta_\ell}+\indof{\event}\altgbltoo
\weakgeq\sum_{\ell=1}^m\mu_\ell\indof{C_\ell}\group{\altgbl_\ell-\prevof{\altgbl_\ell\given C_\ell}+\delta_\ell},
\end{equation*}
so we find that, indeed, \(\gbl\in\group{\posiof{\mathscr{L}_{\prevof{\smallbolleke\condon\smallbolleke}}}+\gambles_{\weakgeq0}}\altcondon\event\).
For the direct inclusion, consider any~\(\gbl\in\group{\posiof{\mathscr{L}_{\prevof{\smallbolleke\condon\smallbolleke}}}+\gambles_{\weakgeq0}}\altcondon\event\), meaning that there are \(n\in\naturals\), \(\lambda_k,\epsilon_k\in\posreals\), \(\gbl_k\in\gambles\) and \(B_k\in\powersetof{\states}\setminus\set{\emptyset}\) such that \(\indof{\event}\gbl\weakgeq\sum_{k=1}^n\lambda_k\indof{B_k}\group{\gbl_k-\prevof{\gbl_k\given B_k}+\epsilon_k}\).
Exploiting \cref{lem::boundedness:posi:L}, we then find that there are~\(B\in\powersetof{\states}\setminus\set{\emptyset}\), \(\altgbl\in\gambles\) and \(\epsilon\in\posreals\) such that
\begin{equation*}
\indof{\event}\gbl
\weakgeq\sum_{k=1}^n\indof{B_k}\group{\gbl_k-\prevof{\gbl_k\given B_k}+\epsilon_k}
\weakgeq\indof{B}\group{\altgbl-\prevof{\altgbl\given B}+\epsilon},
\end{equation*}
and therefore also that \(\indof{\event}\gbl\weakgeq\indof{B\cap\event}\group{\altgbl-\prevof{\altgbl\given B}+\epsilon}\).
We'll therefore be done if we can show that \(\prevof{\altgbl\given B\cap\event}\geq\prevof{\altgbl\given B}\), because we'll then have that \(\indof{\event}\gbl\weakgeq\indof{B\cap\event}\group{\altgbl-\prevof{\altgbl\given B\cap\event}+\epsilon}\), and therefore
\begin{equation*}
\gbl
=\indof{\event}\gbl+\group{\gbl-\indof{\event}\gbl}
\weakgeq\underset{\in\posiof{\mathscr{L}_{\prev_{\event}\group{\smallbolleke\condon\smallbolleke}}}}{\underbrace{\indof{B\cap\event}\group{\altgbl-\prevof{\altgbl\given B\cap\event}+\epsilon}}}+\underset{\in\indifset[\event]}{\underbrace{\group{\gbl-\indof{\event}\gbl}}},
\end{equation*}
so \(\gbl\in\posiof{\mathscr{L}_{\prev_{\event}\group{\smallbolleke\condon\smallbolleke}}}+\indifset[\event]+\gambles_{\weakgeq0}\), and \(\gambles_{\weakgeq0}+\indifset[\event]=\gambles^{\set{\event}}_{\weakgeq0}\), by \cref{lem:conditional:desirable:background}.

So we now set out to prove that \(\prevof{\altgbl\given B\cap\event}\geq\prevof{\altgbl\given B}\).
We have the following sequence of implications, using the properties of coherent conditional previsions:
\begin{multline*}
\indof{\event}\gbl\weakgeq\indof{B}\group{\altgbl-\prevof{\altgbl\given B}+\epsilon}\\
\begin{aligned}
&\then0=\indof{\compof\event}\indof{\event}\gbl
\weakgeq\indof{B\cap\compof\event}\group{\altgbl-\prevof{\altgbl\given B}+\epsilon}\\
&\then0=\prevof{0\given B}\geq\prevof{\indof{B\cap\compof\event}\group{\altgbl-\prevof{\altgbl\given B}+\epsilon}\given B}
&&\reason{\labelcref{axiom:coherence:conditional:linearity}\&\labelcref{axiom:coherence:conditional:bounds}}\\
&\then0\geq\prevof{\indof{B\cap\compof\event}\altgbl\given B}-\prevof{\altgbl\given B}\prevof{\indof{B\cap\compof\event}\given B}+\epsilon\prevof{\indof{B\cap\compof\event}\given B}
&&\reason{\labelcref{axiom:coherence:conditional:linearity}}\\
&\then0\geq\prevof{\indof{\compof\event}\altgbl\given B} - \prevof{\altgbl\given B}\prevof{\indof{\compof\event}\given B}+\epsilon\prevof{\indof{\compof\event}\given B}
&&\reason{\labelcref{axiom:coherence:conditional:bayes}\&\labelcref{axiom:coherence:conditional:bounds}}\\
&\then0\geq-\epsilon\prevof{\indof{\compof\event}\given B}\weakgeq\prevof{\indof{\compof\event}\altgbl\given B}-\prevof{\altgbl\given B}\prevof{\indof{\compof\event}\given B}.
&&\reason{\labelcref{axiom:coherence:conditional:bounds}}
\end{aligned}
\end{multline*}
There are now two possible cases.

The first possibility is that \(\prevof{\indof{\compof\event}\given B}>0\).
Let's first show that then also \(\prevof{\indof{\event}\given B}>0 \).
Indeed, assume towards contradiction that \(\prevof{\indof{\event}\given B}=0\), so \(\prevof{\indof{\compof\event}\given B}=1\).
We infer from the argumentation above that then \(0\geq\prevof{\indof{\compof\event}\altgbl\given B}-\prevof{\altgbl\given B}+\epsilon\), so \(\prevof{\indof{\event}\altgbl\given B}\geq\epsilon>0\).
However, by Bayes' Rule [\labelcref{axiom:coherence:conditional:bayes}], \(\prevof{\indof{\event}\altgbl\given B}=\prevof{\altgbl\given B\cap\event}\prevof{\indof{\event}\given B}=0\), which is the desired contradiction.
So, we can proceed in the knowledge that both \(\prevof{\indof{\compof\event}\given B}>0\) and \(\prevof{\indof{\event}\given B}>0\).
But then we can infer the following sequence of implications, starting from the argumentation above:
\begin{multline*}
-\epsilon\prevof{\indof{\compof\event}\given B}\weakgeq\prevof{\indof{\compof\event}\altgbl\given B}-\prevof{\altgbl\given B}\prevof{\indof{\compof\event}\given B}\\
\begin{aligned}
&\then0>\prevof{\indof{\compof\event}\altgbl\given B}-\prevof{\altgbl\given B}\prevof{\indof{\compof\event}\given B}
&&\reason{\(\epsilon>0\) and \(\prevof{\indof{\compof\event}\given B}>0\)}\\
&\then\prevof{\altgbl\given B}\prevof{\indof{\compof\event}\given B}\weakgt\prevof{\indof{\compof\event}\altgbl\given B}\\
&\then\prevof{\altgbl\given B}\group{1-\prevof{\indof{\event}\given B}}>\prevof{\altgbl\given B}-\prevof{\indof{\event}\altgbl\given B}\\
&\then\prevof{\indof{\event}\altgbl\given B}>\prevof{\altgbl\given B}\prevof{\indof{\event}\given B}\\
&\then\prevof{\altgbl\given B\cap\event}\prevof{\indof{\event}\given B}>\prevof{\altgbl\given B}\prevof{\indof{\event}\given B}
&&\reason{Bayes' Rule [\labelcref{axiom:coherence:conditional:bayes}]}\\
&\then\prevof{\altgbl\given B\cap\event}>\prevof{\altgbl\given B}.
&&\reason{\(\prevof{\indof{\event}\given B}>0\)}
\end{aligned}
\end{multline*}

The second possibility is that \(\prevof{\indof{\compof\event}\given B}=0\), so \(\prevof{\indof{\event}\given B}=1\).
We then infer from Bayes' Rule [\labelcref{axiom:coherence:conditional:bayes}] that \(\prevof{\indof{\compof\event}\altgbl\given B}=\prevof{\altgbl\given B\cap\compof\event}\prevof{\indof{\compof\event}\given B}=0\), and therefore, again using Bayes' Rule [\labelcref{axiom:coherence:conditional:bayes}],
\begin{equation*}
\prevof{\altgbl\given B}
=\prevof{\indof{\event}\altgbl\given B}
=\prevof{\altgbl\given B\cap\event}\prevof{\indof{\event}\given B}
=\prevof{\altgbl\given B\cap\event}.
\end{equation*}
In both cases, we find that, as required, \(\prevof{\altgbl\given B}\geq\prevof{\altgbl\given B\cap\event}\).

Next, we turn to \cref{eq::previsional:auxiliary:result:desirs}.
Consider any~\(\gbl\in\group{\posiof{\mathscr{L}_{\prevof{\smallbolleke\condon\smallbolleke}}}+\gambles_{\stronggt0}}\altcondon\event\), which implies that there are \(n\in\naturals\), \(\lambda_k,\epsilon_k\in\posreals\), \(\gbl_k\in\gambles\) and \(B_k\in\powersetof{\states}\setminus\set{\emptyset}\) such that \(\indof{\event}\gbl\stronggt\sum_{k=1}^n\lambda_k\indof{B_k}\group{\gbl_k-\prevof{\gbl_k\given B_k}+\epsilon_k}\).
Exploiting \cref{lem::boundedness:posi:L}, we then find that there are~\(\altgbl\in\gambles\), \(B\in\powersetof{\states}\setminus\set{\emptyset}\) and \(\epsilon\in\posreals\) such that
\begin{equation*}
\indof{\event}\gbl
\stronggt\sum_{k=1}^n\indof{B_k}\group{\gbl_k-\prevof{\gbl_k\given B_k}+\epsilon_k}
\weakgeq\indof{B}\group{\altgbl-\prevof{\altgbl\given B}+\epsilon}.
\end{equation*}
Hence, \(\indof{\event}\gbl\stronggt\indof{B}\group{\altgbl-\prevof{\altgbl\given B}+\epsilon}\), and therefore also
\begin{equation*}
0
=\indof{\compof\event}\indof{\event}\gbl
\weakgeq\indof{B\cap\compof\event}\group{\altgbl-\prevof{\altgbl\given B}+\epsilon}.
\end{equation*}
Following exactly the same arguments as in the proof of \cref{eq::previsional:auxiliary:result:accepts}, we then find that \(\prevof{\altgbl\given B}\geq\prevof{\altgbl\given B\cap\event}\), so \(\indof{\event}\gbl\stronggt\indof{B}\group{\altgbl-\prevof{\altgbl\given B}+\epsilon}\weakgeq\indof{B}\group{\altgbl-\prevof{\altgbl\given B\cap\event}+\epsilon}\), and therefore also \(\gbl-\indof{B}\group{\altgbl-\prevof{\altgbl\given B\cap\event}+\epsilon}\stronggt0\).
Now, decompose \(\gbl\) as follows:
\begin{align*}
\gbl
&=\group{\gbl-\indof{\event}\gbl}+\indof{\event}\gbl\\
&=\underset{\in\indifset[\event]}{\underbrace{\group{\gbl-\indof{\event}\gbl}}}+\underset{\in\group{\posiof{\mathscr{L}_{\prevof{\smallbolleke\condon\smallbolleke}}}+\gambles_{\weakgeq0}}{\altcondon\event}}{\underbrace{\indof{B}\group{\altgbl-\prevof{\altgbl\given B\cap\event}+\epsilon}}}+\underset{\in\gambles_{\stronggt0}}{\underbrace{\group{\indof{\event}\gbl-\indof{B}\group{\altgbl-\prevof{\altgbl\given B\cap\event}+\epsilon}}}},
\end{align*}
so \(\gbl\in\group{\posiof{\mathscr{L}_{\prevof{\smallbolleke\condon\smallbolleke}}}+\gambles_{\weakgeq0}}{\altcondon\event}+\indifset[\event]+\gambles_{\stronggt0}\).
But,
\begin{align*}
\group{\posiof{\mathscr{L}_{\prevof{\smallbolleke\condon\smallbolleke}}}+\gambles_{\weakgeq0}}{\altcondon\event}+\indifset[\event]+\gambles_{\stronggt0}
&=\group{\posiof{\mathscr{L}_{\prevof{\smallbolleke\condon\smallbolleke}}}+\gambles_{\weakgeq0}}{\altcondon\event}+\gambles_{\stronggt0},
&&\reason{Lem.~\labelcref{lem:indifset:addition}}
\end{align*}
proving the desired inclusion.

For the proof of \cref{eq::previsional:auxiliary:result:vacuous:accepts}, observe that
\begin{equation*}
\gbl\in\gambles_{\weakgeq0}\altcondon\event
\ifandonlyif\indof{\event}\gbl\weakgeq0
\ifandonlyif\inf\group{\gbl\vert\event}\geq0
\ifandonlyif\gbl\in\gambles^{\set\event}_{\weakgeq0},
\text{ for any~\(\gbl\in\gambles\),}
\end{equation*}
and for the proof of \cref{eq::previsional:auxiliary:result:vacuous:desirs}, observe that \(\gbl\in\gambles_{\stronggt0}\altcondon\event\ifandonlyif\indof{\event}\gbl\stronggt0\ifandonlyif\inf\group{\indof{\event}\gbl}>0\), and \(\inf\group{\indof{\event}\gbl}>0\) is impossible unless \(\event=\states\).
\end{proof}

\begin{lemma}\label{lem::boundedness:posi:L}
Every element of~\(\posiof{\mathscr{L}_{\prev\group{\smallbolleke\condon\smallbolleke}}}\) point-wise dominates some element of~\(\mathscr{L}_{\prev\group{\smallbolleke\condon\smallbolleke}}\), and therefore \(\posiof{\mathscr{L}_{\prev\group{\smallbolleke\condon\smallbolleke}}}+\gambles_{\weakgeq0}=\mathscr{L}_{\prev\group{\smallbolleke\condon\smallbolleke}}+\gambles_{\weakgeq0}\).
\end{lemma}

\begin{proof}[Proof of \cref{lem::boundedness:posi:L}]
It's clearly enough to prove the first statement of the lemma, because the second statement follows directly from it.
So consider any~\(\gbl\in\posiof{\mathscr{L}_{\prev\group{\smallbolleke\condon\smallbolleke}}}\), meaning that there are~\(n\in\naturals\), \(\lambda_k,\delta_k\in\posreals\), \(\altgbl_k\in\gambles\) and \( B_k\in\powersetof{\states}\setminus\set{\emptyset}\) such that
\begin{equation*}
\gbl
=\sum_{k=1}^n\lambda_k\indof{B_k}\group{\altgbl_k-\prevof{\altgbl_k\given B_k}+\delta_k}
=\sum_{k=1}^n\indof{B_k}\group{\lambda_k\altgbl_k-\prevof{\lambda_k\altgbl_k\given B_k}+\lambda_k\delta_k}.
\end{equation*}
Let \(\gbl_k\coloneqq\lambda_k\altgbl_k-\prevof{\lambda_k\altgbl_k\given B_k}\in\gambles\) and \(\epsilon_k\coloneqq\lambda_k\delta_k\in\posreals\), then \(\prevof{\gbl_k\given B_k}=0\) and
\begin{align*}
\gbl
=\sum_{k=1}^n\indof{B_k}\group{\gbl_k+\epsilon_k}
&\geq\sum_{k=1}^n\indof{B_k}\gbl_k+\indof{\bigcup_{k=1}^nB_k}\min\set{\epsilon_1,\epsilon_2,\dots,\epsilon_n}\\
&=\indof{\bigcup_{k=1}^nB_k}\group[\bigg]{\sum_{k=1}^n\indof{B_k}\gbl_k}+\indof{\bigcup_{k=1}^nB_k}\min\set{\epsilon_1,\epsilon_2,\dots,\epsilon_n}\\
&=\indof{\bigcup_{k=1}^nB_k}\group[\bigg]{\sum_{k=1}^n\indof{B_k}\gbl_k+\min\set{\epsilon_1,\epsilon_2,\dots,\epsilon_n}}
=\indof{B}\group{\altgbl+\epsilon},
\end{align*}
where we let \(\epsilon\coloneqq\min\set{\epsilon_1,\epsilon_2,\dots,\epsilon_n}\in\posreals\), \(B\coloneqq\bigcup_{k=1}^nB_k\) and \(\altgbl\coloneqq\sum_{k=1}^n\indof{B_k}\gbl_k\in\gambles\).
Observe that
\begin{equation*}
\prevof{\altgbl\vert B}
=\sum_{k=1}^n\prevof{\indof{B_k}\gbl_k\vert B}
=\sum_{k=1}^n\prevof{\gbl_k\vert B\cap B_k}\prevof{\indof{B_k}\given B}
=\sum_{k=1}^n\prevof{\gbl_k\vert B_k}\prevof{\indof{B_k}\given B}
=0,
\end{equation*}
where the second equality follows from applying Bayes' Rule [\labelcref{axiom:coherence:conditional:bayes}]; the third equality follows from the fact that \(B_k\subseteq B\) and therefore \(B_k\cap B=B_k\); and the last equality follows from \(\prevof{\gbl_k\vert B_k}=0\).
Hence, indeed, \(\gbl\geq\indof{B}\group{\altgbl+\epsilon}=\indof{B}\group{\altgbl-\prevof{\altgbl\vert B}+\epsilon}\), as required.
\end{proof}

\begin{lemma}\label{lem:conditional:desirable:background}
For any regular event~\(\event\in\powersetof{\states}\setminus\set{\emptyset}\) and any gamble~\(\gbl\in\gambles\), it holds that \(\gbl\in\gambles_{\stronggt0}+\indifset[\event]\ifandonlyif\inf\group{\gbl\condon{\event}}>0\) and also that \(\gbl\in\gambles_{\weakgeq0}+\indifset[\event]\ifandonlyif\inf\group{\gbl\condon{\event}}\geq0\ifandonlyif\indof{\event}\gbl\weakgeq0\ifandonlyif\gbl\in\gambles_{\weakgeq0}\altcondon\event\).
Hence, also \(\gambles_{\weakgeq0}+\indifset[\event]=\gambles^{\set\event}_{\weakgeq0}=\gambles_{\weakgeq0}\altcondon\event\) and \(\gambles_{\stronggt0}+\indifset[\event]=\gambles^{\set\event}_{\stronggt0}\).
\end{lemma}

\begin{proof}[Proof of \cref{lem:conditional:desirable:background}]
Consider any~\(\altgbl\in\gambles_{\stronggt0}\) and \(\altgbltoo\in\indifset[\event]\), then \(\inf\group{\altgbl+\altgbltoo\condon{\event}}=\inf\group{\altgbl\condon{\event}}>0\), since \(\altgbltoo\in\indifset[\event]\) implies that \(\indof{\event}\altgbltoo=0\) and \(\altgbl\in\gambles_{\stronggt0}\) implies that \(\inf\group{\altgbl\condon{\event}}>0\).
Conversely, consider any~\(\gbl\in\gambles\) for which \(\inf\group{\gbl\condon{\event}}>0\) and focus on the decomposition
\begin{equation*}
\gbl\group\state
=\begin{cases}
\gbl\group\state
&\text{ if \(\state\in\event\)}\\
\inf\group{\gbl\condon{\event}}
&\text{otherwise}
\end{cases}
+\begin{cases}
0
&\text{ if \(\state\in\event\)}\\
\gbl\group\state-\inf\group{\gbl\condon{\event}}
&\text{otherwise}
\end{cases}
\text{ for all \(\state\in\states\)},
\end{equation*}
where the first term is in \(\gambles_{\stronggt0}\) and the second term is in \(\indifset[\event]\), so indeed \(\gbl\in\gambles_{\stronggt0}+\indifset[\event]\).
The proof for the first equivalence involving \(\gambles_{\weakgeq0}+\indifset[\event]\) is completely similar.
The remainder of the proof is now straightforward.
\end{proof}

\begin{proof}[Proof of \cref{prop::previsional:revision}]
Since we infer from \cref{prop::previsional:revision:conditionability} that \(\M_{\prevof{\smallbolleke\given\smallbolleke}}\) is conditionable on all regular events, we infer from the expression~\labelcref{eq::AD:revision:operator} for the revision operator associated with conditioning that, for the accept part,
\begin{align*}
\reviseof{\M_{\prevof{\smallbolleke\given\smallbolleke}}\given\event}_\acc
&=\group{\gambles_{\weakgeq0}\cup\group{\posiof{\mathscr{L}_{\prevof{\smallbolleke\given\smallbolleke}}}+\gambles_{\weakgeq0}}}\altcondon{\event}
&&\reason{Eq.~\labelcref{eq::model:conditional:full:prevision}}\\
&=\gambles_{\weakgeq0}\altcondon{\event}\cup\group{\posiof{\mathscr{L}_{\prevof{\smallbolleke\given\smallbolleke}}}+\gambles_{\weakgeq0}}\altcondon{\event}\\
&=\gambles^{\set\event}_{\weakgeq0}\cup\group{\posiof{\mathscr{L}_{\prev_{\event}\group{\smallbolleke\condon\smallbolleke}}}+\gambles^{\set\event}_{\weakgeq0}};
&&\reason{Eqs.~\labelcref{eq::previsional:auxiliary:result:accepts,eq::previsional:auxiliary:result:vacuous:accepts}}
\end{align*}
for the desirability part,  when \(\event=\states\), the proof of the desired equality is trivial; and when \(\event\neq\states\), we get, using \cref{eq::model:conditional:full:prevision,eq::previsional:auxiliary:result:desirs,eq::previsional:auxiliary:result:vacuous:desirs}, that
\begin{multline*}
\reviseof{\M_{\prevof{\smallbolleke\given\smallbolleke}}\given\event}_\des\\
\begin{aligned}
&=\group{\gambles_{\stronggt0}\cup\group{\posiof{\mathscr{L}_{\prev\group{\smallbolleke\condon\smallbolleke}}}+\gambles_{\stronggt0}}}\altcondon{\event}\cup\group{\gambles_{\stronggt0}+\group{\gambles_{\weakgeq0}\cup\group{\posiof{\mathscr{L}_{\prevof{\smallbolleke\given\smallbolleke}}}+\gambles_{\weakgeq0}}}\altcondon{\event}}\\
&=\group{\posiof{\mathscr{L}_{\prev\group{\smallbolleke\condon\smallbolleke}}+\gambles_{\stronggt0}}}\altcondon{\event}\cup\group{\gambles_{\stronggt0}+\gambles_{\weakgeq0}\altcondon\event}\cup\group{\gambles_{\stronggt0}+\group{\posiof{\mathscr{L}_{\prevof{\smallbolleke\given\smallbolleke}}}+\gambles_{\weakgeq0}}\altcondon{\event}}\\
&=\group{\gambles_{\stronggt0}+\gambles_{\weakgeq0}\altcondon\event}\cup\group{\gambles_{\stronggt0}+\group{\posiof{\mathscr{L}_{\prevof{\smallbolleke\given\smallbolleke}}}+\gambles_{\weakgeq0}}\altcondon{\event}}.
\end{aligned}
\end{multline*}
Now, \cref{lem:conditional:desirable:background} guarantees that
\begin{equation*}
\gambles_{\stronggt0}+\gambles_{\weakgeq0}\altcondon\event
=\gambles_{\stronggt0}+\gambles_{\weakgeq0}+\indifset[\event]
=\gambles_{\stronggt0}+\indifset[\event]
=\gambles^{\set{\event}}_{\stronggt0},
\end{equation*}
and
\begin{align*}
\gambles_{\stronggt0}+\group{\posiof{\mathscr{L}_{\prevof{\smallbolleke\given\smallbolleke}}}+\gambles_{\weakgeq0}}\altcondon{\event}
&=\gambles_{\stronggt0}+\posiof{\mathscr{L}_{\prev_{\event}\group{\smallbolleke\condon\smallbolleke}}}+\gambles^{\set\event}_{\weakgeq0}
&&\reason{Eq.~\labelcref{eq::previsional:auxiliary:result:accepts}}\\
&=\posiof{\mathscr{L}_{\prev_{\event}\group{\smallbolleke\condon\smallbolleke}}}+\gambles_{\stronggt0}+\gambles_{\weakgeq0}+\indifset[\event]
&&\reason{Lem.~\labelcref{lem:conditional:desirable:background}}\\
&=\posiof{\mathscr{L}_{\prev_{\event}\group{\smallbolleke\condon\smallbolleke}}}+\gambles_{\stronggt0}+\indifset[\event]\\
&=\posiof{\mathscr{L}_{\prev_{\event}\group{\smallbolleke\condon\smallbolleke}}}+\gambles^{\set\event}_{\stronggt0}.
&&\reason{Lem.~\labelcref{lem:conditional:desirable:background}}
\qedhere
\end{align*}
\end{proof}

\begin{proof}[Proof of \cref{eq::precisional:consistency:condition:for:expansion}]
Simply consider the following chain of equivalences:
\begin{align*}
\group{\posiof{\mathscr{L}_{\prevof{\smallbolleke\given\smallbolleke}}}+\gambles_{\stronggt0}}\cap\indifset[\event]=\emptyset
&\ifandonlyif0\notin\posiof{\mathscr{L}_{\prevof{\smallbolleke\given\smallbolleke}}}+\gambles_{\stronggt0}+\indifset[\event]\\
&\ifandonlyif\posiof{\mathscr{L}_{\prevof{\smallbolleke\given\smallbolleke}}}\cap-\group{\gambles_{\stronggt0}+\indifset[\event]}=\emptyset\\
&\ifandonlyif\group{\forall\altgbl\in\posiof{\mathscr{L}_{\prevof{\smallbolleke\given\smallbolleke}}}}\,-\altgbl\notin\gambles_{\stronggt0}+\indifset[\event]\\
&\ifandonlyif\group{\forall\altgbl\in\posiof{\mathscr{L}_{\prevof{\smallbolleke\given\smallbolleke}}}}\,\inf\group{-\altgbl\condon{\event}}\leq0
&&\reason{\cref{lem:conditional:desirable:background}}\\
&\ifandonlyif\group{\forall\altgbl\in\posiof{\mathscr{L}_{\prevof{\smallbolleke\given\smallbolleke}}}}\,\sup\group{\altgbl\condon{\event}}\geq0,
\end{align*}
which completes the proof.
\end{proof}

\begin{proof}[Proof of \cref{prop::conditional:expansion:consistency}]
We begin by proving necessity.
Consider any regular event~\(B\) such that \(B\cap\event\neq\emptyset\), and any real~\(\epsilon>0\).
Apply the consistency condition \labelcref{eq::precisional:consistency:condition:for:expansion} for \(n\instantiateas1\), \(\lambda_1\instantiateas1\), \(\epsilon_1\instantiateas\epsilon\), \(\gbl_1\instantiateas-\indof{\event}\) and \(B_1\instantiateas B\) to find that \(\sup\group{\indof{B}\group{-\indof{\event}+\prevof{\indof{\event}\vert B}+\epsilon}\condon{\event}}\geq0\).
Since~\(B\cap\event\neq\emptyset\), this implies that \(\prevof{\indof{\event}\vert B}-1+\epsilon\geq0\), and therefore that \(\prevof{\indof{\event}\vert B}=1\).

For sufficiency, assume that \(\prevof{\indof{\event}\vert B}=1\), and therefore also \(\prevof{\indof{\compof\event}\vert B}=0\), for all regular events~\(B\) such that \(B\cap\event\neq\emptyset\).
Bayes' Rule [\labelcref{axiom:coherence:conditional:bayes}] then tells us that for all~\(\gbl\in\gambles\),
\begin{align}
\prevof{\gbl\vert B}
&=\prevof{\indof{\event}\gbl\vert B}+\prevof{\indof{\compof\event}\gbl\vert B}
\notag\\
&=\prevof{\gbl\vert B\cap\event}\prevof{\indof{\event}\vert B}+\prevof{\gbl\vert B\cap\compof\event}\prevof{\indof{\compof\event}\vert B}
\notag\\
&=\prevof{\gbl\vert B\cap\event}.
\label{eq::conditional:expansion:consistency:1}
\end{align}
Now consider any~\(n\in\naturals\), \(\lambda_k,\epsilon_k\in\posreals\), \(f_k\in\gambles\) and \(B_k\in\powersetof{\states}\setminus\set\emptyset\), then we can write
\begin{multline*}
\sup\group[\bigg]{\sum_{k=1}^n\lambda_k\indof{B_k}\group{\gbl_k-\prevof{\gbl_k\given B_k}+\epsilon_k}\Big\vert\event}\\
\begin{aligned}
&=\sup_{\state\in\event}\sum_{k=1}^n\lambda_k\indof{B_k}\group\state\group{\gbl_k\group\state-\prevof{\gbl_k\given B_k}+\epsilon_k}\\
&=\sup_{\state\in\event}\sum_{k=1}^n\lambda_k\indof{B_k}\group\state\indof{\event}\group\state\group{\gbl_k\group\state-\prevof{\gbl_k\given B_k}+\epsilon_k}\\
&=\sup_{\state\in\event}\smashoperator[r]{\sum_{\substack{k=1\\B_k\cap\event\neq\emptyset}}^n}\lambda_k\indof{B_k\cap\event}\group\state\group{\gbl_k\group\state-\prevof{\gbl_k\given B_k}+\epsilon_k}\\
&=\sup_{\state\in\event}\smashoperator[r]{\sum_{\substack{k=1\\B_k\cap\event\neq\emptyset}}^n}\lambda_k\indof{B_k\cap\event}\group\state\group{\gbl_k\group\state-\prevof{\gbl_k\given B_k\cap\event}+\epsilon_k}
&&\reason{Eq.~\labelcref{eq::conditional:expansion:consistency:1}}\\
&=\sup\group[\bigg]{\smashoperator[r]{\sum_{\substack{k=1\\B_k\cap\event\neq\emptyset}}^n}\lambda_k\indof{B_k\cap\event} \group{\gbl_k-\prevof{\gbl_k\given B_k\cap\event}+\epsilon_k}\Big\vert\event}\geq0,
&&\reason{\labelcref{eq::avoiding:sure:loss}}
\end{aligned}
\end{multline*}
which completes the proof.
\end{proof}

\begin{proof}[Proof of \cref{prop::conditional:expansion}]
We start with the accept part, for which it's enough to prove that
\begin{equation*}
\posiof{\mathscr{L}_{\prevof{\smallbolleke\given\smallbolleke}}}+\gambles_{\weakgeq0}+\indifset[\event]
=\posiof{\mathscr{L}_{\prev_{\event}\group{\smallbolleke\condon\smallbolleke}}}+\gambles^{\set\event}_{\weakgeq0},
\end{equation*}
since we infer from \cref{lem:conditional:desirable:background} that \(\gambles_{\weakgeq0}+\indifset[\event]=\gambles^{\set\event}_{\weakgeq0}\).
For the direct inclusion, consider \(\gbl\in\posiof{\mathscr{L}_{\prevof{\smallbolleke\given\smallbolleke}}}+\gambles_{\weakgeq0}+\indifset[\event]\), so there are~\(n\in\mathbb{N}\), \(\lambda_k,\epsilon_k\in\posreals\), \(\gbl_k\in\gambles\), \(B_k\in\powersetof{\states}\setminus\set\emptyset\) and \(\altgbl\in\indifset[\event]\) such that \(\gbl\weakgeq\sum_{k=1}^n\lambda_k\indof{B_k}\group{\gbl_k-\prevof{\gbl_k\vert B_k}+\epsilon_k}+\altgbl\), and therefore also that
\begin{equation*}
\indof{\event}\gbl
\weakgeq\sum_{k=1}^n\lambda_k\indof{B_k\cap\event}\group{\gbl_k-\prevof{\gbl_k\vert B_k}+\epsilon_k}+\indof{\event}\altgbl
=\sum_{k=1}^n\lambda_k\indof{B_k\cap\event}\group{\gbl_k-\prevof{\gbl_k\vert B_k}+\epsilon_k}.
\end{equation*}
An argument similar to the one in the proof of \cref{prop::conditional:expansion:consistency} allows us to derive from the assumed \(\bgM\)-consistency of~\(\M_{\prevof{\smallbolleke\given\smallbolleke}}\union\M_{\event}\) that \(\prevof{\gbl_k\vert B_k}=\prevof{\gbl_k\vert B_k\cap\event}\) whenever \(B_k\cap\event\neq\emptyset\).
Hence,
\begin{equation*}
\gbl=\indof{\event}\gbl+\group{\gbl-\indof{\event}\gbl}
\weakgeq\smashoperator{\sum_{\substack{k=1\\B_k\cap\event\neq\emptyset}}^n}\lambda_k\indof{B_k\cap\event}\group{\gbl_k-\prevof{\gbl_k\vert B_k\cap\event}+\epsilon_k}+\underset{\in\indifset[\event]}{\underbrace{\group{\gbl-\indof{\event}\gbl}}},
\end{equation*}
and therefore indeed that \(\gbl\in\posiof{\mathscr{L}_{\prev_{\event}\group{\smallbolleke\condon\smallbolleke}}}+\gambles_{\weakgeq0}+\indifset[\event]=\posiof{\mathscr{L}_{\prev_{\event}\group{\smallbolleke\condon\smallbolleke}}}+\gambles^{\set\event}_{\weakgeq0}\).
The converse inclusion follows directly from \(\mathscr{L}_{\prev_{\event}\group{\smallbolleke\condon\smallbolleke}}\subseteq\mathscr{L}_{\prevof{\smallbolleke\given\smallbolleke}}\) and \(\gambles^{\set\event}_{\weakgeq0}=\gambles_{\weakgeq0}+\indifset[\event]\) [\cref{lem:conditional:desirable:background}].

For the desirability part, it's enough to prove that
\begin{equation*}
\posiof{\mathscr{L}_{\prevof{\smallbolleke\given\smallbolleke}}}+\gambles^{\set\event}_{\stronggt0}
=\posiof{\mathscr{L}_{\prev_{\event}\group{\smallbolleke\condon\smallbolleke}}}+\gambles^{\set\event}_{\stronggt0},
\end{equation*}
since we infer from \cref{lem:conditional:desirable:background} that \(\gambles_{\stronggt0}+\indifset[\event]=\gambles^{\set\event}_{\stronggt0}\).
The converse inclusion follows directly from \(\mathscr{L}_{\prev_{\event}\group{\smallbolleke\condon\smallbolleke}}\subseteq\mathscr{L}_{\prevof{\smallbolleke\given\smallbolleke}}\).
For the direct inclusion, we proceed a bit differently.
Consider any~\(\gbl\in\posiof{\mathscr{L}_{\prevof{\smallbolleke\given\smallbolleke}}}+\gambles^{\set\event}_{\stronggt0}\), meaning that there are~\(n\in\mathbb{N}\), \(\lambda_k,\epsilon_k\in\posreals\), \(\gbl_k\in\gambles\), \(B_k\in\powersetof{\states}\setminus\set\emptyset\) and \(\altgbl\in\indifset[\event]\) such that
\begin{equation*}
\inf\group[\bigg]{\gbl-\sum_{k=1}^n\lambda_k\indof{B_k}\group{\gbl_k-\prevof{\gbl_k\vert B_k}+\epsilon_k}\Big\vert\event}>0,
\end{equation*}
and therefore also that
\begin{multline*}
\inf\group[\bigg]{\indof{\event}\gbl-\smashoperator{\sum_{\substack{k=1\\B_k\cap\event\neq\emptyset}}^n}\lambda_k\indof{B_k\cap\event}\group{\gbl_k-\prevof{\gbl_k\vert B_k\cap\event}+\epsilon_k}\Big\vert\event}\\
=\inf\group[\bigg]{\indof{\event}\gbl-\sum_{k=1}^n\lambda_k\indof{B_k\cap\event}\group{\gbl_k-\prevof{\gbl_k\vert B_k}+\epsilon_k}\Big\vert\event}>0,
\end{multline*}
where, for the first equality, we used an argument similar to the one in the proof of \cref{prop::conditional:expansion:consistency} to derive from the \(\bgM\)-consistency of~\(\M_{\prevof{\smallbolleke\given\smallbolleke}}\union\M_{\event}\) that \(\prevof{\gbl_k\vert B_k}=\prevof{\gbl_k\vert B_k\cap\event}\) whenever \(B_k\cap\event\neq\emptyset\).
This allows us to conclude that \(\indof{\event}\gbl\in\posiof{\mathscr{L}_{\prevof{\smallbolleke\given\smallbolleke}}}+\gambles^{\set\event}_{\stronggt0}\), and therefore that \(\gbl=\indof{\event}\gbl+\group{\gbl-\indof{\event}\gbl}\in\posiof{\mathscr{L}_{\prev_{\event}\group{\smallbolleke\condon\smallbolleke}}}+\gambles^{\set\event}_{\stronggt0}+\indifset[\event]=\posiof{\mathscr{L}_{\prev_{\event}\group{\smallbolleke\condon\smallbolleke}}}+\gambles^{\set\event}_{\stronggt0}\), where the last equality again follows from \cref{lem:conditional:desirable:background}.
\end{proof}

\begin{proof}[Proof of \cref{prop::postulates:satisfied:previsional}]
We proved in \cref{prop::AD:model:revision:axioms:obeyed} that \labelcref{axiom:revision:1,axiom:revision:2,axiom:revision:3,axiom:revision:4,axiom:revision:5,axiom:revision:7} are satisfied for general AD-models, so they're definitely satisfied for the full previsional AD-models.
So it only remains to consider \labelcref{axiom:revision:4,axiom:revision:8}.

\labelcref{axiom:revision:4} follows at once from \cref{prop::previsional:revision,prop::conditional:expansion}.

For the proof of \labelcref{axiom:revision:8}, we content ourselves with the following brief summary of the proof strategy.
We've already argued in \cref{prop::previsional:revision,prop::conditional:expansion} that expanding and revision using an event~\(\event[1]\) simply results in a new `full previsional' AD-model \(\M_{\prev_{\event[1]}\group{\smallbolleke\condon\smallbolleke}}\) that is, essentially, the old model restricted to the new possibility space~\(\event[1]\) rather than~\(\states\).
Therefore, when we subsequently deal with a new event~\(\event[2]\), the same reasoning as in all the propositions and lemmas applies, but now simply with respect to the possibility space~\(\event[1]\) rather than~\(\states\).
So, to prove \labelcref{axiom:revision:8}, we simply repeat the proofs for \labelcref{axiom:revision:4}, where we replace \(\states\) with \(\event[1]\), and \(\event\) with \(\event[1]\cap\event[2]\); other than that, all steps in the proofs will remain the same, {\itshape mutatis mutandis}.
\end{proof}

\end{document}